\pdfoutput=1
\PassOptionsToPackage{table}{xcolor}
\documentclass{article}
\usepackage{iclr2027_conference,times}

\usepackage{xcolor}

\usepackage[T1]{fontenc}
\usepackage[utf8]{inputenc}
\usepackage{microtype}
\usepackage{amsmath}
\usepackage{amssymb}
\usepackage{amsthm}
\usepackage{thm-restate}
\usepackage{refcount}
\usepackage{graphicx}
\usepackage{booktabs}
\usepackage{multirow}
\usepackage{float}
\usepackage{etoolbox}
\makeatletter
\renewcommand\paragraph{\@startsection{paragraph}{4}{\z@}{0.5ex plus 0.2ex minus 0.1ex}{-0.6em}{\normalsize\bf}}
\makeatother
\AtBeginEnvironment{table}{\setlength{\abovecaptionskip}{0pt}\setlength{\belowcaptionskip}{6pt}}
\AtBeginEnvironment{table*}{\setlength{\abovecaptionskip}{0pt}\setlength{\belowcaptionskip}{6pt}}
\AtBeginEnvironment{figure}{\setlength{\abovecaptionskip}{5pt}}
\usepackage{placeins}
\usepackage{algorithm}
\usepackage{algpseudocode}
\algrenewcommand\algorithmicrequire{\textbf{Input:}}
\algrenewcommand\algorithmicensure{\textbf{Output:}}
\newcommand{\algin}[1]{\Statex\hspace{\algorithmicindent}#1}
\usepackage{wrapfig}
\usepackage{needspace}
\usepackage{tikz}
\usetikzlibrary{arrows.meta,positioning,fit,backgrounds,calc}
\usepackage{hyperref}
\hypersetup{pdftitle={CertMark: Distortion-Free Multi-Bit Watermarking with Certified Decoding},pdfauthor={Paweł Batorski, Przemysław Spurek, Paul Swoboda}}
\renewcommand{\theHequation}{\arabic{equation}}
\usepackage{url}

\def\redc{\cellcolor[HTML]{FF999A}}
\def\orangec{\cellcolor[HTML]{FFCC99}}
\def\yellowc{\cellcolor[HTML]{FFF8AD}}
\newcommand{\ours}{CertMark}
\newcommand{\oursAgn}{\ours{} model-agnostic}
\newcommand{\oursMA}{\ours{} model-aware}
\newcommand{\oursAgnK}[1]{\oursAgn{} #1-bit}
\newcommand{\oursMAK}[1]{\oursMA{} #1-bit}
\newcommand{\oursAgnTwo}{\oursAgn{} 2-bit}
\newcommand{\oursMATwo}{\oursMA{} 2-bit}

\newcommand{\PRF}{\mathrm{PRF}}

\newtheorem{theorem}{Theorem}
\theoremstyle{plain}\newtheorem*{remark}{Remark}
\newcounter{savedtheorem}

\title{\ours{}: Distortion-Free Multi-Bit Watermarking with Certified Decoding}

\author{Paweł Batorski\textsuperscript{1} \quad Przemysław Spurek\textsuperscript{2,3} \quad Paul Swoboda\textsuperscript{1} \\
\normalfont\textsuperscript{1}Heinrich Heine University Düsseldorf \quad
\textsuperscript{2}Jagiellonian University \quad
\textsuperscript{3}IDEAS Research Institute}

\iclrfinalcopy

\begin{document}

\maketitle
\lhead{Preprint}
\vspace*{-8pt}

\begin{abstract}%
Leading multi-bit watermarking methods for language models encode messages by biasing the model's next-token probabilities, creating a trade-off between message recovery and text quality. Their decoders typically return the highest-scoring candidate from accumulated token-level evidence, without a certified abstention rule that bounds the probability of outputting an incorrect message. We introduce CertMark, a distribution-preserving multi-bit watermark with certified decoding. Rather than modifying probabilities, CertMark uses the embedded message to seed an exact Gumbel-max sampler, thereby preserving the model's original sampling distribution. We propose two scalable decoders: a model-agnostic, text-only decoder and a model-aware variant that leverages the original next-token distributions for stronger recovery. Both support certified abstention with mathematical bounds on the probability of returning an incorrect message. Across text completion, summarization, and story generation, CertMark matches the perplexity of unwatermarked text while reliably recovering multi-bit messages. The model-aware decoder further achieves higher bit accuracy than probability-biasing baselines.
Our code is publicly available at \url{https://github.com/Batorskq/CertMark}.
\end{abstract}

\section{Introduction}
\label{sec:intro}

Multi-bit watermarks let the provider of a language model attach a payload to text generation, such as a user identifier or a model version, and recover it later from the text alone \citep{wang2024codable,yoo-etal-2024-advancing,xu-etal-2026-xmark}.
The most prominent schemes \citep{xu-etal-2026-xmark} encode the payload by biasing the model's next-token probabilities toward a message-dependent part of the vocabulary, the multi-bit form of the green-list watermark \citep{kirchenbauer2023watermark,yoo-etal-2024-advancing,xu-etal-2026-xmark}.
The bias introduced into the distribution encodes the watermarking signal.
Therefore, watermarking recovery and text quality compete against each other~\citep{kirchenbauer2023watermark,kirchenbauer2024reliability,hu2024unbiased,giboulot2024watermax}.
Additionally, the watermarking decoder returns a vote with no bound on the probability that the recovered message is wrong \citep{yoo-etal-2024-advancing,wang2024codable,xu-etal-2026-xmark}.
Alternatively, distortion-free sampling methods do not decrease text quality.
They are well established for deciding whether a text is watermarked at all \citep{kuditipudi2024robust,christ2024undetectable,hu2024unbiased,wu2024resilient}, and has been extended to the multi-bit case \citep{jiang-etal-2025-stealthink,jiang2026mirrormark,feng2025bimark} but does not come with a certificate for recovering the message and recovers messages less accurately.

\ours{} is the first method that is distortion-free, multi-bit and comes with certificates on message recovery.
Our method encodes messages into the seeds for the random sampling during autoregressive decoding.
Specifically, the message seeds the random draws of a Gumbel-max sampler \citep{maddison2014astar}, every token is still sampled from the deployed model's own distribution, and a text produced this way is, to anyone without the key, distributed exactly like unwatermarked text, the guarantee that distortion-free zero-bit schemes provide \citep{kuditipudi2024robust,christ2024undetectable}.
The key holder, in turn, can replay the draws for every candidate message.
A wrong candidate scores like pure chance whatever the text, so the decoder can attach to each recovered message an exact bound on the probability that it is wrong, and abstain when the evidence is too thin.
When the detector can also run the model, it weighs each token by how surprising the model found it \citep{lu-etal-2024-entropy,lee-etal-2024-who} and recovers the payload where text alone cannot.

Our results are twofold. The distortion-free, chance-level and certificate statements are theorems under the standard idealisation of the pseudorandom function as a random function \citep{christ2024undetectable,golowich2024edit}. The empirical claims come from the XMark benchmark \citep{xu-etal-2026-xmark}, text completion on C4 \citep{raffel2020exploring}, summarization on CNN/DailyMail \citep{hermann2015teaching} and story generation on WritingPrompts \citep{fan-etal-2018-hierarchical} with Qwen3.5-4B and Llama-3.1-8B, together with machine translation, long messages, a copy-paste attack \citep{kirchenbauer2024reliability,xu-etal-2026-xmark}, a runtime comparison and a direct check of the certificate.

In summary, our contributions are as follows:
\begin{itemize}\setlength{\itemsep}{1pt}
  \item[i)] A distortion-free multi-bit watermark whose text-only decoder comes with an exact certificate on the probability of returning a wrong message, valid for any model, prompt and text.
  \item[ii)] A model-aware decoder that reads the same generations through the deployed sampler and remains certified and robust when part of the text has been edited.
  \item[iii)] An evaluation on two model families in which \ours{} keeps text quality at the unwatermarked level, recovers payloads at or above every baseline once the model is available, has the fastest text-only decoder, and whose certificate holds empirically.
\end{itemize}

\section{Related Work}

\paragraph{Zero-bit watermarking.}
Zero-bit watermarks detect whether a text is marked but carry no payload. Schemes either bias the next-token distribution toward keyed tokens \citep{kirchenbauer2023watermark,fernandez2023three,zhao2024provable,kirchenbauer2024reliability,liu2024unforgeable,liu2024semantic,liu2024adaptive,chang-etal-2024-postmark}, or control sampling while preserving the unwatermarked law \citep{kuditipudi2024robust,hu2024unbiased,wu2024resilient,christ2024undetectable,dathathri2024scalable,zhao2025permute,he2025theoretically,chen-etal-2025-improved}. Token-level signals remain vulnerable to strong rewriting and paraphrasing \citep{krishna2023paraphrasing,zhang2024watermarks,pang2024no,jovanovic2024watermark}.

\paragraph{Multi-bit watermarking.}
Multi-bit schemes carry an identifier or metadata rather than a single presence bit. Most use message-conditioned logit shifts, vocabulary partitions, or blockwise allocation, trading payload recovery against decoding cost, robustness, and text quality \citep{fernandez2023three,wang2024codable,li-etal-2024-identifying,yoo-etal-2024-advancing,qu-etal-2025-provably,xu2025majority,xu2025robust,xu-etal-2026-xmark,kim2026blockwise}. Distribution-preserving alternatives use keyed reweighting or place the message in the sampler's randomness \citep{jiang-etal-2025-stealthink,feng2025bimark,boroujeny2024multibit,jiang2026mirrormark}. All return the highest-scoring message with no bound on the probability that it is wrong. \ours{} instead uses the payload to seed an exact Gumbel-max draw, couples its text-only decoder to an exact wrong-message null, and optionally uses the model to decode the same generations more accurately.

\section{Method}
\label{sec:method}

Table~\ref{tab:notation} collects the notation used in this section.

\begin{table}[t]
\centering\footnotesize\setlength{\tabcolsep}{2pt}\renewcommand{\arraystretch}{0.95}
\caption{Notation used in the method.}
\label{tab:notation}
\resizebox{\ifdim\width>\textwidth\textwidth\else\width\fi}{!}{%
\begin{tabular}{@{}ll@{\hspace{5pt}}ll@{}}
\toprule
\multicolumn{2}{@{}l}{\textit{Model and text}} & \multicolumn{2}{l@{}}{\textit{Encoder}} \\
\midrule
$V$ & vocabulary & $h$ & context length in tokens \\
$x$ & prompt & $c_t=y_{t-h:t-1}$ & context of step $t$ \\
$y_{1:n}$, $y_t$, $y_{<t}$ & text, its $t$-th token, its prefix & $i_t$ & chunk carried at step $t$ \\
$Y_{1:n}$ & emitted text as a random variable & $\mathcal{C}_{t}$ & contexts used so far \\
$p_t(\cdot)=p(\cdot\mid x,y_{<t})$ & deployed sampler's law at step $t$ & $a_t$ & argument tuple of step $t$ \\
\multicolumn{2}{@{}l}{\textit{Key and randomness}} & $u_t(v)$ & uniform attached to token $v$ at step $t$ \\
$\kappa\in\mathcal{S}$ & secret key, $\mathcal{S}=\{0,1\}^{64}$ & $S(v)$ & Gumbel-max score of token $v$ \\
$\PRF_\kappa$, $\langle\cdot\rangle$ & keyed PRF, its argument encoding & $\mathcal{E}$ & encoder as a deterministic map \\
$\mathcal{X}$ & typed PRF arguments & \multicolumn{2}{l@{}}{\textit{Decoder and certificate}} \\
$\rho\colon\mathcal{S}\to(0,1)$ & seed to a grid point in $(0,1)$ & $\mathcal{T}$ & visited positions (first occurrences) \\
$\mathsf{chunk},\mathsf{msg},\mathsf{fresh}$ & domain-separating tags & $\mathcal{T}_c$, $n_c=|\mathcal{T}_c|$ & positions of chunk $c$, their count \\
$\Pr_\kappa$ & probability over the idealised key & $\{\mathcal{T}_c\}_{c=1}^C$ & scored positions of all chunks \\
\multicolumn{2}{@{}l}{\textit{Message}} & $u^{(m')}_t(v)$ & uniform under candidate $m'$ \\
$m\in\{0,1\}^{L}$ & message of $L$ bits & $S_c(m')$ & text-only score of candidate $m'$ \\
$m=(m_1,\dots,m_C)$ & its $C$ chunks, $m_c\in\{0,1\}^{k_c}$ & $\hat\delta_c$, $\delta$ & chunk certificate, chosen level \\
$m'$ & candidate value of a chunk & $D_\delta$, $\bot$ & certified decoder, abstention \\
$\hat m_c$, $\hat m$ & decoded chunk, decoded message & $H_1$, $H_0$ & hypotheses $m_c=m'$ and $m_c\neq m'$ \\
$\sigma$ & error budget in \eqref{eq:desiderata} & & \\
\bottomrule
\end{tabular}}
\end{table}

\paragraph{Preliminaries}
\label{app:prelim}

Write $V$ for the vocabulary, $x$ for the prompt, $y_{<t}$ for the tokens emitted before step $t$ and $p_t(\cdot)=p(\cdot\mid x,y_{<t})$ for the law the deployed sampler actually draws $y_t$ from, after temperature and truncation. A multi-bit watermark embeds a message $m\in\{0,1\}^{L}$, such as a user identifier, into the generated text so that the holder of a secret key $\kappa$ can recover it from the text alone \citep{wang2024codable,yoo-etal-2024-advancing,xu-etal-2026-xmark}.

Three properties are wanted of such a scheme at once, for a message $m$, a key $\kappa$ and the generated text $y$:
\begin{equation}
\underbrace{\mathcal{L}(y\mid m,\kappa)=\textstyle\prod_{t} p_t}_{\text{(i) distortion-free}},
\qquad
\underbrace{\hat m(y,\kappa)=m}_{\text{(ii) recoverable}},
\qquad
\underbrace{\Pr[\hat m\neq m]\le\sigma}_{\text{(iii) certified}}.
\label{eq:desiderata}
\end{equation}
Most existing schemes maximise (ii) alone: green-list and other logit bias methods alters the sampling law at every step, so (i) fails.
Typically, the decoder returns the highest-scoring message with no bound of the form (iii). The schemes that do satisfy (i), by reweighting the law so that its mean over the key is unchanged or by moving the message into the sampling randomness \citep{jiang-etal-2025-stealthink,feng2025bimark,jiang2026mirrormark}, still return a vote with no bound of the form (iii). Appendix~\ref{app:mirrormark} compares them with \ours{}.
In this work we develop \ours{}, a distortion-free multi-bit watermark that attains (i)--(iii) that lets the message select the sampler's randomness rather than its law.

\subsection{Encoding and Model-Agnostic Decoding}

\ours{}  uses the key $\kappa$ and message $m\in\{0,1\}^{L}$ to determine the sampler's randomness and lets $p_t$ unchanged.
A key holder can reconstruct this randomness, while anyone without the key sees an
ordinary sample from $p_t$.  The following identity, the Gumbel-max trick
\citep{gumbel1954statistical,maddison2014astar}, makes this possible.

\begin{restatable}[Gumbel-max sampling]{lemma}{certmarkGumbel}
\label{lem:gumbel}
Fix a step $t$ and recall that $p_t$ is the sampler's law over $V$ at that step, redrawn at every step. Draw $u(v)\sim U(0,1)$ i.i.d.\ over $v\in V$, form the Gumbel variates and the scores
\begin{equation}
S(v)\;=\;\log p_t(v)-\log\bigl(-\log u(v)\bigr),
\label{eq:gumbel}
\end{equation}
and emit $y=\arg\max_{v\in V}S(v)$. Then $\Pr[y=v]=p_t(v)$ for every $v\in V$.
\end{restatable}

The lemma separates the sampling law from the realized draw: $p_t$ fixes the
law, while $u_t$ selects the sample.  We generate $u_t$ pseudorandomly from
$(\kappa,m)$, so the encoder and detector can reproduce it without changing
$p_t$.  Each $u_t$ must remain i.i.d.\ uniform conditional on the prefix and must
be fresh across steps.  The pseudorandom function and freshness rule below ensure this.  The lemma
and chain rule then yield Theorem~\ref{thm:exact}.  The decoder uses the same
randomness to distinguish the correct message from wrong candidates.

\paragraph{Pseudorandom function.}
As is standard in multi-bit watermarking, the sampling randomness comes from a pseudorandom function $\PRF_\kappa:\{0,1\}^*\to\mathcal{S}$, $\mathcal{S}=\{0,1\}^{64}$, keyed by the secret key $\kappa$ \citep{goldreich1986construct}. It is deterministic, so anyone holding $\kappa$ can recompute any of its values, while without the key its values at distinct inputs are computationally indistinguishable from independent uniform elements of $\mathcal{S}$. We evaluate it at tuples of typed arguments from $\mathcal{X}$ (tags, token sequences and integers), mapped to a single input by an injective, prefix-free encoding, $\PRF_\kappa(x_1,\dots,x_r):=\PRF_\kappa(\langle x_1,\dots,x_r\rangle)$, so distinct tuples are distinct inputs, and the tags $\mathsf{chunk}$, $\mathsf{msg}$ and $\mathsf{fresh}$ separate its three uses. An output $z\in\mathcal{S}$, read as an unsigned integer, becomes a number in $(0,1)$ through $\rho(z)=(\lfloor z/2^{11}\rfloor+\frac12)\,2^{-53}$, which is uniform on the $53$-bit midpoint grid of $(0,1)$ when $z$ is uniform. Throughout the analysis we use the standard random-function idealisation: $\PRF_\kappa$ is treated as a uniformly random function, so the values $\rho(\PRF_\kappa(a))$ at distinct inputs $a$ are independent $U(0,1)$ variables. A uniform is therefore fresh exactly when its argument tuple has not been queried before, which the encoder below guarantees. Appendix~\ref{app:prf} gives our instantiation.

\begin{figure}[t]
  \centering
  \resizebox{\textwidth}{!}{%
\pgfdeclarelayer{frames}
\pgfsetlayers{frames,background,main}
\begin{tikzpicture}[
  font=\small,
  >={Latex[length=1.8mm,width=1.5mm]},
  box/.style={draw, rounded corners=2.5pt, align=center, inner sep=3.5pt, minimum height=9mm, line width=0.6pt},
  enc/.style={box, fill=blue!7, draw=blue!55!black},
  dec/.style={box, fill=orange!12, draw=orange!75!black},
  sh/.style={box, fill=black!4, draw=black!45},
  arr/.style={->, line width=0.7pt, draw=black!65},
  link/.style={<->, dashed, line width=0.5pt, draw=black!35},
  lab/.style={font=\footnotesize\itshape, text=black!55, align=center},
  tok/.style={draw=black!45, fill=white, minimum width=7mm, minimum height=5.6mm, inner xsep=2.5pt, inner ysep=0pt, font=\footnotesize, rounded corners=1pt},
  chunk/.style={draw=blue!55!black, fill=white, minimum width=7mm, minimum height=5.6mm, inner xsep=2.5pt, inner ysep=0pt, font=\footnotesize, rounded corners=1pt},
]
\def\xa{0.45}\def\xb{4.8}\def\xc{9.85}\def\xd{13.9}\def\xe{18.5}
\def\ye{0}\def\yd{-5.4}
\node[tok] (t1) at (\xa-0.75,\ye) {$y_{t-h}$};
\node[right=1mm of t1, inner sep=0pt, font=\footnotesize] (t2) {$\cdots$};
\node[tok, right=1mm of t2] (t3) {$y_{t-1}$};
\begin{scope}[on background layer]
\node[sh, fit=(t1)(t3), inner sep=5pt, label={[lab]above:context $c_t=y_{t-h:t-1}$}] (E1) {};
\end{scope}
\node[sh, text width=4.25cm, label={[lab]above:chunk selected by the context}] (E2) at (\xb,\ye) {$i_t=\PRF_\kappa(\mathsf{chunk},c_t)$\\[2pt]$\bmod\ C$};
\node[enc, text width=3.0cm] (E3a) at (\xc,\ye+0.78) {$a_t=(\mathsf{msg},c_t,m_{i_t})$};
\node[enc, text width=3.0cm, label={[lab]below:repeated context}] (E3b) at (\xc,\ye-0.78) {$a_t=(\mathsf{fresh},c_t,t)$};
\node[enc, text width=3.35cm, label={[lab]above:uniforms for every token}] (E4) at (\xd,\ye) {$u_t(v)=\rho\bigl(\PRF_\kappa(a_t,v)\bigr)$\\[2pt]$v\in V$};
\node[enc, text width=3.8cm, label={[lab, xshift=-2mm]above:Gumbel-max: emit $y_t$}] (E5) at (\xe,\ye) {$y_t=\arg\max_{v}\bigl[\log p_t(v)$\\[1pt]$\qquad-\log(-\log u_t(v))\bigr]$};
\node[chunk] (c1) at (\xc-1.35,\ye+2.3) {$m_1$};
\node[chunk, right=1mm of c1, fill=blue!22] (c2) {$m_{i_t}$};
\node[right=1mm of c2, inner sep=0pt, font=\footnotesize] (c3) {$\cdots$};
\node[chunk, right=1mm of c3] (c4) {$m_C$};
\begin{scope}[on background layer]
\node[enc, fit=(c1)(c4), inner sep=4pt, minimum height=0pt, label={[lab]left:message $m\in\{0,1\}^{L}$\\in $C$ chunks of $k_c$ bits}] (E0) {};
\end{scope}
\node[sh, minimum height=6mm, inner sep=3pt] (LM) at (\xe,\ye+2.25) {language model: $p_t(\cdot\mid x,y_{<t})$};
\draw[arr] (E1) -- (E2);
\draw[arr] (E2.east) -- node[lab, above, sloped, pos=0.45, inner sep=0.5pt] {new $c_t$} (E3a.west);
\draw[arr] (E2.east) -- node[lab, below, sloped, pos=0.45, inner sep=0.5pt] {seen $c_t$} (E3b.west);
\draw[arr] (E3a.east) -- (E4.west);
\draw[arr] (E3b.east) -- (E4.west);
\draw[arr] (E4) -- (E5);
\draw[arr] ([xshift=1.55cm]LM.south) -- ([xshift=1.55cm]E5.north);
\draw[arr, blue!55!black] (c2.south) -- node[lab, right, text=blue!55!black, pos=0.5, inner sep=1.5pt] {chunk $i_t$} (c2.south |- E3a.north);
\coordinate (fb) at (0,\ye-2.0);
\draw[arr, dashed, black!50, rounded corners=6pt] (E5.south) -- (E5.south |- fb) -- node[lab, above, pos=0.72, inner sep=1.5pt] {append $y_t$ to the text and move to step $t+1$} (E1.south |- fb) -- (E1.south);
\node[tok] (d1) at (\xa-0.75,\yd) {$y_{t-h}$};
\node[right=1mm of d1, inner sep=0pt, font=\footnotesize] (d2) {$\cdots$};
\node[tok, right=1mm of d2] (d3) {$y_{t-1}$};
\node[tok, right=2.2mm of d3, fill=orange!25, draw=orange!75!black] (d4) {$y_t$};
\begin{scope}[on background layer]
\node[sh, fit=(d1)(d3), inner sep=5pt, label={[lab]below:first occurrence\\of $c_t$ in the text $y$}] (D1) {};
\end{scope}
\node[sh, text width=4.25cm, label={[lab]below:recomputed from $y$ and $\kappa$}] (D2) at (\xb,\yd) {$i_t=\PRF_\kappa(\mathsf{chunk},c_t)$\\[2pt]$\bmod\ C$};
\node[dec, text width=3.0cm, fill=orange!5, draw=orange!55!black, label={[lab, yshift=-1mm]below:message unknown:\\try all $2^{k_c}$ candidates}] (D3s2) at (\xc+0.26,\yd-0.26) {\phantom{every candidate}\\[2pt]\phantom{$(\mathsf{msg},c_t,m')$}};
\node[dec, text width=3.0cm, fill=orange!8, draw=orange!65!black] (D3s1) at (\xc+0.13,\yd-0.13) {\phantom{every candidate}\\[2pt]\phantom{$(\mathsf{msg},c_t,m')$}};
\node[dec, text width=3.0cm] (D3) at (\xc,\yd) {every $m'\in\{0,1\}^{k_c}$:\\[2pt]$(\mathsf{msg},c_t,m')$};
\node[dec, text width=3.35cm, label={[lab]below:uniform of the emitted token}] (D4) at (\xd,\yd) {$u^{(m')}_t(y_t)=$\\[2pt]$\rho\bigl(\PRF_\kappa(\mathsf{msg},c_t,m',y_t)\bigr)$};
\node[dec, text width=4.7cm, label={[lab]below:accumulate and decide after the last token}] (D5) at (\xe,\yd) {$S_{i_t}(m')\mathrel{+}=-\log\bigl(1-u^{(m')}_t(y_t)\bigr)$\\[2pt]$\hat m_c=\arg\max_{m'}S_c(m')$, certificate $\hat\delta_c$};
\begin{scope}[on background layer]\draw[arr] (D1) -- (D2);\end{scope}
\draw[arr] (D2) -- (D3);
\draw[arr] (D3) -- (D4);
\draw[arr] (D4) -- (D5);
\coordinate (ytl) at (0,\yd+1.0);
\draw[arr, orange!75!black, rounded corners=6pt] (d4.north) -- (d4.north |- ytl) -- node[lab, above, pos=0.5, text=orange!75!black, inner sep=1pt] {$y_t$} (D4.north |- ytl) -- (D4.north);
\coordinate (encpad) at (\xc,\ye-2.15);
\coordinate (decpad) at ($(D5.south)+(0,-0.8)$);
\coordinate (decpadtop) at (\xb,\yd+1.3);
\begin{pgfonlayer}{frames}
\node[draw=blue!40, rounded corners=4pt, fill=blue!2, fit=(E1)(E0)(E3b)(LM)(E5)(encpad), inner xsep=8pt, inner ysep=7pt] (encbox) {};
\node[draw=orange!50, rounded corners=4pt, fill=orange!3, fit=(D1)(D3s2)(D5)(decpad)(decpadtop), inner xsep=8pt, inner ysep=7pt] (decbox) {};
\end{pgfonlayer}
\draw[link] (E1.south |- encbox.south) -- node[lab, fill=white, inner sep=1pt] {same} (D1.north |- decbox.north);
\draw[link] (E2.south |- encbox.south) -- node[lab, fill=white, inner sep=1pt] {same} (D2.north |- decbox.north);
\node[rotate=90, anchor=south, font=\small\bfseries, text=blue!55!black, align=center] at (encbox.west) {Encoder\\{\normalfont\footnotesize\itshape holds $\kappa$, $m$, model}};
\node[rotate=90, anchor=south, font=\small\bfseries, text=orange!75!black, align=center] at (decbox.west) {Decoder\\{\normalfont\footnotesize\itshape holds $\kappa$, text only}};
\end{tikzpicture}
  }
  \caption{One step of \ours{}. The encoder (top) hashes the context to pick a chunk, reads the uniforms at an argument tuple carrying that chunk's value, or at a fresh position tuple when the context recurs, and emits the Gumbel-max token. The decoder (bottom) recomputes context and chunk from text and key alone, tries every value of the chunk, and accumulates $-\log(1-u)$ over the emitted tokens.}
  \label{fig:method}
\end{figure}
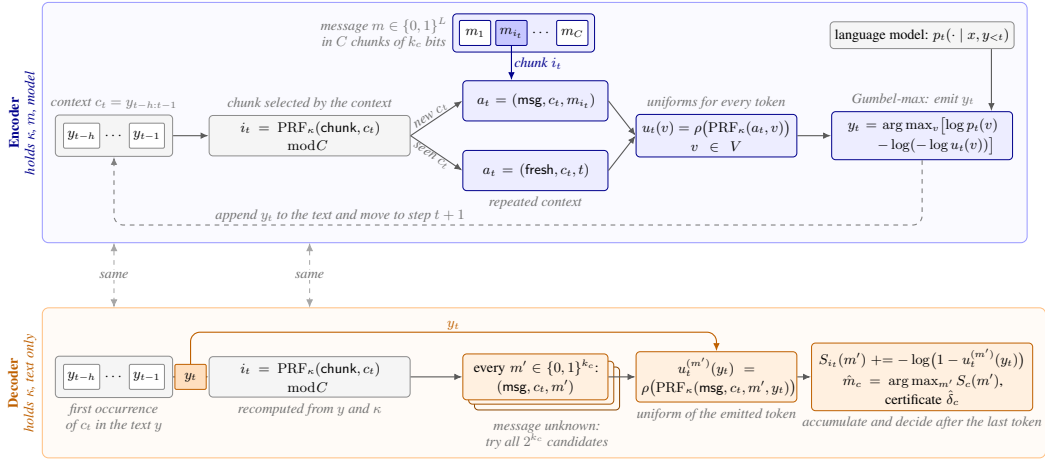

\paragraph{Encoder}
Rather than encoding the full message $m$ at once, the encoder divides it into
chunks: for $m\in\{0,1\}^{L}$, we write $m=(m_1,\dots,m_C)$ with
$m_c\in\{0,1\}^{k_c}$ and $\sum_{c=1}^{C}k_c=L$.  When a previously unseen
context appears, the encoder uses that step to encode one selected chunk.

Ordinary sampling turns a prefix into $p_t(\cdot | x,y_{<t})$ and then draws $y_t\sim p_t$ using fresh randomness.
Below we detail our new encoder
\begin{equation}
\mathcal{E}:\;(y_{<t},\,p_t,\,\kappa,\,m)\;\longmapsto\;y_t \,.
\end{equation}
$\mathcal{E}$ consumes no randomness of its own but replaces it deterministically by the message $m$ and key $\kappa$, yet, by Lemma~\ref{lem:gumbel}, emits $y_t\sim p_t$.

At step $t$ the model defines the next-token distribution $p_t(v)$ of Section~\ref{app:prelim}, and we take as context the last $h$ emitted tokens, $c_t=y_{t-h:t-1}$.
The context alone selects the message chunk,
\begin{equation}
i_t\;=\;\PRF_\kappa(\mathsf{chunk},c_t)\bmod C ,
\label{eq:chunkidx}
\end{equation}
and that chunk's value then enters the arguments at which the uniforms are read.

Lemma~\ref{lem:gumbel} requires the uniforms at step $t$ to be new: were $u_t$ to repeat a vector already used in this generation, it would no longer be independent of the prefix. A context occurring for the first time guarantees this, because the tuple $(\mathsf{msg},c_t,m_{i_t})$ is then one the encoder has never queried. A context that recurs does not, so at those steps the encoder replaces the message by the position $t$, which is new by construction. Writing $\mathcal{C}_{t-1}$ for the contexts already used in this generation, the step reads its uniforms at the argument tuple $a_t$ given by
\begin{equation}
a_t=\begin{cases}
(\mathsf{fresh},\,c_t,\,t), & c_t\in\mathcal{C}_{t-1},\\[2pt]
(\mathsf{msg},\,c_t,\,m_{i_t}), & \text{otherwise, and } \mathcal{C}_t=\mathcal{C}_{t-1}\cup\{c_t\},
\end{cases}
\qquad
u_t(v)=\rho\bigl(\PRF_\kappa(a_t,v)\bigr),
\label{eq:encoder}
\end{equation}
and the encoder emits the token that Lemma~\ref{lem:gumbel} produces from these uniforms,
\begin{equation}
y_t\;=\;\arg\max_{v\in V}\;\Bigl[\log p_t(v)-\log\bigl(-\log u_t(v)\bigr)\Bigr].
\label{eq:emit}
\end{equation}
Thus, a repeated context encodes no message chunk. Figure~\ref{fig:method} summarises the encoder and, aligned beneath it, the decoder of the next paragraph. Appendix~\ref{app:pseudocode} gives both as pseudocode.

The key property of the encoder is that it changes the source of the sampling
randomness without changing the distribution of the generated text.  Because
each uniform vector is independent of the prefix, Lemma~\ref{lem:gumbel} applies
at every step and gives the same sequence distribution as unwatermarked LLM
generation, as formalized next.

\begin{restatable}[Distortion-free generation]{theorem}{certmarkExact}
\label{thm:exact}
Treat $\PRF$ as a random function and draw the key $\kappa$ uniformly from $\mathcal{S}$. Let $Y_{1:n}$ be the sequence emitted by \eqref{eq:encoder}--\eqref{eq:emit} on prompt $x$ carrying message $m\in\{0,1\}^{L}$. Then for every such $m$ and $x$ and every $w\in V^{n}$,
\begin{equation}
\Pr_{\kappa}\bigl[Y_{1:n}=w\bigr]\;=\;\prod_{t=1}^{n}p\bigl(w_t\mid x,w_{<t}\bigr),
\label{eq:exact}
\end{equation}
which is the law of unwatermarked sampling and, in particular, does not depend on $m$.
\end{restatable}

\noindent\textit{Proof.} The proof can be found in Appendix~\ref{app:exact}.

The theorem shows that, over an
unknown random key, \ours{} generates the same sequence distribution as the
unwatermarked LLM, independently of the encoded message.  As a corollary, a
reader without the key cannot distinguish the two distributions from the text
alone.

\begin{remark}[Deployment model]
Theorem~\ref{thm:exact} is a single-generation guarantee. In our experiments,
as is standard for PRF-seeded distortion-free samplers
\citep{kuditipudi2024robust,dathathri2024scalable}, we reuse one key. The encoder is therefore deterministic in the
prompt and message: repeated prompts reproduce the same text, and shared
contexts reuse the same uniforms. Each generation still has the exact
unwatermarked marginal law, but generations need not be independent \citep{fu-etal-2024-gumbelsoft,wu2025distortion,gloaguen2025blackbox}.
Theoretically, Theorem~\ref{thm:nonce} in Appendix~\ref{app:additional} resolves
this by adding a fresh $r$-bit nonce to every encoder argument. We leave its evaluation to future work.
\end{remark}

\paragraph{Model-Agnostic Decoder}
The decoder holds only the text and $\kappa$, the same access every baseline detector has.
It re-tokenises the text and, with $c_t=y_{t-h:t-1}$ as in the encoder, visits the positions whose context occurs for the first time \emph{in that text},
\begin{equation}
\mathcal{T}\;=\;\bigl\{\,t\in\{h+1,\dots,n\}:\ c_t\neq c_s\ \text{for all}\ h+1\le s<t\,\bigr\},
\label{eq:visited}
\end{equation}
recomputing $i_t$ at each $t\in\mathcal{T}$ exactly as the encoder did. Positions $t\le h$ have no full context and a repeated context was read under the tag $\mathsf{fresh}$ by \eqref{eq:encoder}, so neither carries a message and both are skipped.
Writing $\mathcal{T}_c=\{t\in\mathcal{T}: i_t=c\}$ for the positions so assigned to chunk $c$ and $n_c=|\mathcal{T}_c|$, it scores every candidate value $m'\in\{0,1\}^{k_c}$ of that chunk by
\begin{equation}
S_c(m')\;=\;\sum_{t\in \mathcal{T}_c}-\log\bigl(1-u^{(m')}_t(y_t)\bigr),
\qquad
u^{(m')}_t(v)=\rho\bigl(\PRF_\kappa(\mathsf{msg},c_t,m',v)\bigr),
\label{eq:score}
\end{equation}
and outputs $\hat m_c=\arg\max_{m'}S_c(m')$. This is a joint decision over the whole codebook of $2^{k_c}$ words. Chunking is what keeps this at $C\cdot2^{k_c}$ rather than $2^{L}$. After decoding all $C$ chunks, we concatenate $\hat m_1,\dots,\hat m_C$ to obtain the final message $\hat m=(\hat m_1,\dots,\hat m_C)$.

\begin{remark}[Choice of the score]
We use our specific form for scores~\eqref{eq:score}, because it allows us to derive that the score for wrong messages follows a Gamma distribution.
Alternatively, we might use the simpler score $\sum_{t\in \mathcal{T}_c}u^{(m')}_t(y_t)$, which would, however, result in harder to handle Irwin-Hall laws~\citep{irwin1927frequency,hall1927distribution}.
\end{remark}

\subsection{Certification and Theoretical Analysis}
\label{sec:cert}

From this point through the end of Section~\ref{sec:modelaware}, we use the
random-function idealisation of $\PRF$ from Section~\ref{sec:method} (Appendix~\ref{app:prf}). We fix the
deployed model, prompt $x$, sampling settings, embedded message $m$, and a
chunk $c$ with $k_c\ge1$. Unless stated otherwise, probabilities are over the
idealised keyed randomness. For this fixed chunk, the decoder considers every candidate
$m'\in\{0,1\}^{k_c}$, calling $m_c$ the correct candidate and every
$m'\neq m_c$ a wrong candidate.
The decoder returns the candidate with
the largest score.  To analyze this comparison, Theorem~\ref{thm:null} gives the
score distribution of every wrong candidate, while Theorem~\ref{thm:gain}
gives the per-token distributions for both cases. Its full statement in
Appendix~\ref{app:gain} also gives their expected scores.

\begin{restatable}[Score distribution for wrong candidates]{theorem}{certmarkNull}
\label{thm:null}
For a wrong candidate $m'\neq m_c$, conditional on the emitted text and the scored positions
$\mathcal{T}_c$ selected by the chunk assignments, whenever $n_c=|\mathcal{T}_c|\ge1$,
\begin{equation}
S_c(m')\;\sim\;\Gamma(n_c,1)\qquad\text{exactly,}
\label{eq:null}
\end{equation}
where $\Gamma(n_c,1)$ has shape $n_c$ and rate $1$. If $n_c=0$, then
$S_c(m')=0$.
\end{restatable}

\noindent\textit{Proof.} The proof can be found in Appendix~\ref{app:null}.

\begin{theorem}[Candidate distributions]
\label{thm:gain}
For $t\in \mathcal{T}_c$, condition on
$y_{<t}$ and $p=p_t(y_t)\in(0,1]$, the sampler's probability of the emitted
token. For every candidate $m'$, its reconstructed value satisfies
\begin{equation}
u_t^{(m')}(y_t)\sim
\begin{cases}
\mathrm{Beta}(1/p,1), & m'=m_c \quad\text{(correct candidate)},\\[2pt]
U(0,1), & m'\neq m_c \quad\text{(wrong candidate)},
\end{cases}
\label{eq:candidate_laws}
\end{equation}
where $u_t^{(m_c)}(y_t)=u_t(y_t)$.
\end{theorem}
\stepcounter{equation}\xdef\gainEquationNumber{\theequation}

\noindent\textit{Proof.} The proof can be found in
Appendix~\ref{app:gain}.

We now describe the certified decoder, which uses the null distribution from
Theorem~\ref{thm:null} to decide whether to return the predicted chunk or
abstain.  An abstention means that the decoder returns no value for a particular
message chunk because its certificate does not meet the chosen confidence
level.  Theorem~\ref{thm:cert} shows that the probability of returning an
incorrect chunk is at most the chosen level $\delta$.

\begin{restatable}[Certified decoding]{theorem}{certmarkCert}
\label{thm:cert}
For a level $\delta\in(0,1]$, let $\hat m_c(y)=\arg\max_{m'}S_c(m')$ and define
\begin{equation}
\hat\delta_c(y)\;=\;\min\Bigl\{1,\;\bigl(2^{k_c}-1\bigr)\,
\Pr\bigl[\Gamma(n_c,1)\ge S_c(\hat m_c(y))\bigr]\Bigr\},
\label{eq:certdef}
\end{equation}
with $\hat\delta_c(y)=1$ when $n_c=0$.  Let $\bot$ denote abstention and define
the certified decoder $D_\delta\colon V^{n}\to\{0,1\}^{k_c}\cup\{\bot\}$ by
\begin{equation}
D_\delta(y)\;=\;\begin{cases}\hat m_c(y), & \hat\delta_c(y)\le\delta,\\[2pt] \bot, & \hat\delta_c(y)>\delta,\end{cases}
\label{eq:rule}
\end{equation}
where the scores are computed from the text $y$ and the key.  Then, for the text
$Y_{1:n}$ emitted by \eqref{eq:encoder}--\eqref{eq:emit},
\begin{equation}
\Pr_{\kappa}\bigl[D_\delta(Y_{1:n})\notin\{m_c,\bot\}\bigr]\;\le\;\delta,
\label{eq:cert}
\end{equation}
Here the probability is over the idealised PRF randomness induced by the
one-time draw of $\kappa$.
\end{restatable}

\noindent\textit{Proof.} The proof can be found in Appendix~\ref{app:cert}.

The theorem provides certified
decoding because the probability of returning an incorrect chunk is at most
$\delta$, with abstention when the evidence is insufficient.  As a corollary,
decoding each of the $C$ chunks at level $\delta/C$ bounds the probability of
returning any incorrect chunk in the final message by $\delta$.

Appendix~\ref{app:scaling} adds a payload-scaling bound, Theorem~\ref{thm:scaling}, under which the error of a chunk falls exponentially in the accumulated evidence and each additional bit costs about $\ln 2$ of it.

\subsection{Model-Aware Decoder}
\label{sec:modelaware}

A model-aware decoder also has access to the prompt and deployed model, allowing
it to recompute $p_t=p_t(y_t)$ for each scored token under the complete sampling
procedure. By Theorem~\ref{thm:gain}, the candidate distributions in
\eqref{eq:candidate_laws} have densities
$f_{\mathrm{correct}}(u\mid p)=p^{-1}u^{1/p-1}$ and
$f_{\mathrm{wrong}}(u)=1$ on $(0,1)$. A single scored token therefore
contributes the likelihood ratio
\begin{equation}
f_{p}(u)\;=\;\frac{f_{\mathrm{correct}}(u\mid p)}{f_{\mathrm{wrong}}(u)}\;=\;\frac{p^{-1}u^{1/p-1}}{1}\;=\;p^{-1}u^{1/p-1},
\label{eq:lr_token}
\end{equation}
which is identically $1$ when $p=1$.
A token, where the sampler was certain, carries no evidence.
This happens for large for $u$ near $1$ when $p<1$, so a large uniform at a token the sampler was unsure of is strong evidence for the candidate.

For the results in this subsection, we additionally condition on the emitted
text and the scored-position sets $\{\mathcal{T}_c\}_{c=1}^C$, which fixes
$p_t=p_t(y_t)\in(0,1]$ for $t\in \mathcal{T}_c$. For a candidate $m'$ we test the hypothesis
$H_1\colon m_c=m'$ against the alternative $H_0\colon m_c\neq m'$.

\begin{theorem}[Exact likelihood ratio]
\label{thm:llr}
For a fixed candidate $m'$, the log-likelihood ratio of $H_1$ to $H_0$
based on its scored uniforms is
\begin{equation}
\ell_c(m')=\sum_{t\in \mathcal{T}_c}\Bigl[
  \bigl(\tfrac{1}{p_t}-1\bigr)\log u^{(m')}_t(y_t)
  +\log\tfrac{1}{p_t}\Bigr].
\label{eq:llr}
\end{equation}
\end{theorem}

\noindent\textit{Proof.} The proof can be found in Appendix~\ref{app:llr}.

\begin{table*}[t]
\centering\footnotesize\setlength{\tabcolsep}{2.5pt}
\caption{Main benchmark on Qwen3.5-4B: C4 text completion, CNN/DailyMail summarization and WritingPrompts story generation with $L=8$ and token budgets $T\in\{150,200,250,300\}$. Perplexity is measured under Qwen3.5-9B; the distortion columns are defined in Section~\ref{sec:results}. The four \ours{} rows are the model-agnostic and the model-aware decoder, each with one 8-bit chunk and with 2-bit chunks. Colours mark the best, second-best and third-best method in each column.}
\label{tab:xmark_table1}
\label{tab:table1_qwen354b}
\resizebox{\ifdim\width>\textwidth\textwidth\else\width\fi}{!}{%
\begin{tabular}{lcccccccccccccc}
\toprule
\multirow{2}{*}{Method} & \multicolumn{2}{c}{$T=150$} & \multicolumn{2}{c}{$T=200$} & \multicolumn{2}{c}{$T=250$} & \multicolumn{2}{c}{$T=300$} & \multicolumn{2}{c}{Avg.} & \multicolumn{4}{c}{Distortion (avg.\ over $T$)} \\
\cmidrule(lr){2-3} \cmidrule(lr){4-5} \cmidrule(lr){6-7} \cmidrule(lr){8-9} \cmidrule(lr){10-11} \cmidrule(lr){12-15}
 & BA$\uparrow$ & PPL$\downarrow$ & BA$\uparrow$ & PPL$\downarrow$ & BA$\uparrow$ & PPL$\downarrow$ & BA$\uparrow$ & PPL$\downarrow$ & BA$\uparrow$ & PPL$\downarrow$ & Top-1$\uparrow$ & Top-5$\uparrow$ & R-1$\uparrow$ & R-L$\uparrow$ \\
\midrule
\rowcolor[HTML]{EDEDED} \multicolumn{15}{l}{\textbf{Text Completion}} \\
CycleShift & \redc 100.00 & 7.97 & \redc 100.00 & 7.58 & \redc 100.00 & 7.05 & \redc 100.00 & 7.44 & \redc 100.00 & 7.51 & 55.02 & 84.96 & 0.290 & 0.170 \\
DepthW & 95.75 & 7.40 & \redc 100.00 & 6.96 & \redc 100.00 & 6.52 & 97.25 & 6.68 & \yellowc 98.25 & 6.89 & 57.39 & 85.87 & 0.290 & 0.168 \\
StealthInk & 89.75 & 6.15 & 90.50 & 5.79 & 95.25 & 5.73 & 94.25 & 5.61 & 92.44 & 5.82 & 62.81 & 89.16 & 0.308 & \yellowc 0.179 \\
MPAC & \yellowc 97.25 & 7.70 & \orangec 98.00 & 7.52 & \orangec 98.75 & 7.03 & \yellowc 99.00 & 7.09 & \yellowc 98.25 & 7.33 & 55.60 & 85.34 & 0.293 & 0.170 \\
RSBH & 96.00 & 7.11 & \yellowc 97.25 & 6.88 & \yellowc 97.75 & 7.03 & 97.00 & 7.03 & 97.00 & 7.01 & 56.21 & 86.33 & 0.289 & 0.167 \\
XMark & \redc 100.00 & 6.93 & \redc 100.00 & 7.07 & \redc 100.00 & 6.78 & \redc 100.00 & 6.81 & \redc 100.00 & 6.90 & 56.66 & 86.86 & 0.301 & 0.173 \\
\cmidrule(lr){1-15}
\oursAgnK{8} & \redc 100.00 & \yellowc 5.65 & \redc 100.00 & \orangec 5.23 & \redc 100.00 & \redc 5.29 & \redc 100.00 & \orangec 5.11 & \redc 100.00 & \redc 5.32 & \orangec 65.48 & \redc 90.08 & \yellowc 0.311 & \orangec 0.184 \\
\oursAgnTwo{} & \orangec 99.75 & \redc 5.63 & \redc 100.00 & 5.55 & \redc 100.00 & \orangec 5.34 & \orangec 99.75 & \yellowc 5.20 & \orangec 99.88 & \yellowc 5.43 & \redc 65.59 & 89.89 & \redc 0.316 & \orangec 0.184 \\
\oursMAK{8} & \redc 100.00 & 5.69 & \redc 100.00 & \redc 5.20 & \redc 100.00 & \yellowc 5.35 & \redc 100.00 & \redc 5.09 & \redc 100.00 & \orangec 5.33 & 65.42 & \orangec 90.03 & \orangec 0.312 & \redc 0.185 \\
\oursMATwo{} & \redc 100.00 & \orangec 5.64 & \redc 100.00 & \yellowc 5.52 & \redc 100.00 & 5.36 & \redc 100.00 & 5.24 & \redc 100.00 & 5.44 & \yellowc 65.43 & \yellowc 89.94 & \redc 0.316 & \orangec 0.184 \\
\midrule
\rowcolor[HTML]{EDEDED} \multicolumn{15}{l}{\textbf{Text Summarization}} \\
CycleShift & 95.75 & 3.58 & 95.00 & 3.75 & 95.00 & 3.69 & \yellowc 99.50 & 3.71 & 96.31 & 3.68 & 74.56 & 96.87 & 0.541 & 0.348 \\
DepthW & 74.75 & 3.21 & 75.50 & 3.43 & 85.00 & 3.51 & 87.00 & 3.57 & 80.56 & 3.43 & 76.72 & 97.33 & 0.551 & 0.357 \\
StealthInk & 75.25 & \redc 2.95 & 80.75 & \yellowc 3.10 & 82.25 & \yellowc 3.11 & 88.00 & 3.14 & 81.56 & \yellowc 3.08 & 80.59 & 98.14 & 0.562 & 0.372 \\
MPAC & 86.25 & 3.60 & 89.75 & 3.50 & 92.00 & 3.59 & 94.75 & 3.79 & 90.69 & 3.62 & 74.92 & 96.74 & 0.546 & 0.358 \\
RSBH & 78.75 & 3.56 & 88.75 & 3.58 & 96.25 & 3.80 & 93.00 & 3.72 & 89.19 & 3.67 & 74.54 & 97.09 & 0.542 & 0.352 \\
XMark & 91.25 & 3.49 & 93.75 & 3.53 & 97.75 & 3.64 & 99.00 & 3.63 & 95.44 & 3.57 & 74.96 & 97.45 & 0.544 & 0.349 \\
\cmidrule(lr){1-15}
\oursAgnK{8} & \yellowc 97.75 & 3.05 & \yellowc 98.75 & \redc 3.05 & \yellowc 98.50 & \redc 3.06 & \redc 100.00 & \redc 3.08 & \yellowc 98.75 & \redc 3.06 & \redc 81.62 & \redc 98.25 & \redc 0.575 & \redc 0.393 \\
\oursAgnTwo{} & 87.25 & \orangec 2.99 & 94.50 & \yellowc 3.10 & 95.00 & 3.13 & 97.50 & \orangec 3.09 & 93.56 & \yellowc 3.08 & \orangec 81.57 & 98.17 & 0.563 & 0.381 \\
\oursMAK{8} & \redc 100.00 & 3.03 & \redc 100.00 & \orangec 3.07 & \redc 100.00 & \orangec 3.10 & \redc 100.00 & \yellowc 3.10 & \redc 100.00 & \orangec 3.07 & \yellowc 81.50 & \orangec 98.24 & \orangec 0.574 & \orangec 0.392 \\
\oursMATwo{} & \orangec 98.75 & \yellowc 3.00 & \orangec 99.25 & 3.12 & \orangec 99.75 & \yellowc 3.11 & \orangec 99.75 & 3.11 & \orangec 99.38 & \yellowc 3.08 & \redc 81.62 & \yellowc 98.20 & \yellowc 0.564 & \yellowc 0.382 \\
\midrule
\rowcolor[HTML]{EDEDED} \multicolumn{15}{l}{\textbf{Story Generation}} \\
CycleShift & 95.00 & 4.90 & \redc 100.00 & 4.85 & \redc 100.00 & 4.98 & \redc 100.00 & 5.17 & \orangec 98.75 & 4.97 & 66.89 & 92.44 & 0.349 & 0.221 \\
DepthW & 87.25 & 4.33 & 94.75 & 4.45 & 96.50 & \yellowc 4.79 & 96.75 & \yellowc 4.50 & 93.81 & 4.52 & 69.21 & 93.35 & 0.360 & \yellowc 0.232 \\
StealthInk & 82.50 & \yellowc 3.81 & 84.75 & \yellowc 3.93 & 89.50 & \orangec 4.02 & 92.50 & \orangec 3.87 & 87.31 & \yellowc 3.91 & \yellowc 73.62 & \yellowc 95.16 & \redc 0.367 & \redc 0.239 \\
MPAC & 92.50 & 4.90 & 94.50 & 5.07 & 94.50 & 5.24 & 96.25 & 4.98 & 94.44 & 5.05 & 66.69 & 92.65 & 0.350 & 0.217 \\
RSBH & 92.25 & 4.64 & 96.25 & 4.66 & 97.25 & 5.07 & \yellowc 97.25 & 4.98 & 95.75 & 4.84 & 67.08 & 93.41 & 0.349 & 0.219 \\
XMark & \orangec 96.75 & 4.72 & \yellowc 97.25 & 4.54 & \orangec 99.50 & 4.90 & \orangec 99.75 & 5.00 & \yellowc 98.31 & 4.79 & 67.55 & 93.56 & 0.350 & 0.224 \\
\cmidrule(lr){1-15}
\oursAgnK{8} & \redc 100.00 & \redc 3.63 & \redc 100.00 & \redc 3.71 & \redc 100.00 & \redc 3.81 & \redc 100.00 & \orangec 3.87 & \redc 100.00 & \redc 3.75 & \redc 75.89 & \orangec 95.54 & \yellowc 0.362 & 0.230 \\
\oursAgnTwo{} & \yellowc 95.75 & \orangec 3.71 & \orangec 98.50 & \orangec 3.75 & \yellowc 99.25 & \redc 3.81 & \orangec 99.75 & \redc 3.82 & \yellowc 98.31 & \orangec 3.77 & \orangec 75.81 & \redc 95.56 & \orangec 0.366 & \orangec 0.236 \\
\oursMAK{8} & \redc 100.00 & \redc 3.63 & \redc 100.00 & \redc 3.71 & \redc 100.00 & \redc 3.81 & \redc 100.00 & \orangec 3.87 & \redc 100.00 & \redc 3.75 & \redc 75.89 & \orangec 95.54 & \yellowc 0.362 & 0.230 \\
\oursMATwo{} & \redc 100.00 & \orangec 3.71 & \redc 100.00 & \orangec 3.75 & \redc 100.00 & \redc 3.81 & \redc 100.00 & \redc 3.82 & \redc 100.00 & \orangec 3.77 & \orangec 75.81 & \redc 95.56 & \orangec 0.366 & \orangec 0.236 \\
\bottomrule
\end{tabular}}
\end{table*}

Operationally, the model-aware decoder keeps the same candidate enumeration,
chunk assignments, and reconstructed uniforms as the model-agnostic decoder.
Instead of maximizing the text-only score $S_c(m')$ in
\eqref{eq:score}, it maximizes the log-likelihood score in
\eqref{eq:llr}:
\begin{equation}
\hat m_c^{\mathrm{MA}}
=\arg\max_{m'\in\{0,1\}^{k_c}}\ell_c(m').
\end{equation}
The decoded chunks are concatenated into the final message as before.

Theorem~\ref{thm:llr} assumes that every scored token was generated by
\ours{}. This makes its likelihood-ratio decoder vulnerable to editing attacks
that introduce foreign tokens. We therefore present the robust decoder in
Theorem~\ref{thm:robust}.

\begin{theorem}[Robust decoding under contamination]
\label{thm:robust}
Fix $\varepsilon\in(0,1)$. At each scored position $t$, suppose independently that the true candidate's uniform follows $\mathrm{Beta}(1/p_t,1)$ with probability $1-\varepsilon$ and $U(0,1)$ with probability $\varepsilon$, while every wrong candidate's uniform is $U(0,1)$. Score each candidate by
\begin{equation}
r_c(m')\;=\;\sum_{t\in \mathcal{T}_c}\log\Bigl[(1-\varepsilon)\,f_{p_t}\bigl(u^{(m')}_t(y_t)\bigr)+\varepsilon\Bigr],
\label{eq:robust}
\end{equation}
Under a wrong candidate this score has an exact upper-tail probability $Q_{c,\varepsilon}(s)=\Pr[r_c(m')\ge s]$. The robust decoder and its certificate are
\begin{equation}
\hat m_c=\arg\max_{m'}r_c(m'),
\qquad
\hat\delta_{c,\varepsilon}=\min\!\left\{1,(2^{k_c}-1)
Q_{c,\varepsilon}\bigl(r_c(\hat m_c)\bigr)\right\}.
\end{equation}
Returning $\hat m_c$ only when $\hat\delta_{c,\varepsilon}\le\delta$ bounds the probability of returning an incorrect chunk by $\delta$. Otherwise the decoder abstains.
\end{theorem}

\noindent\textit{Proof.} The proof can be found in Appendix~\ref{app:robust}.

Appendix~\ref{app:macert} derives and computes the null tail. Appendix~\ref{app:pseudocode} provides pseudocode for both decoders.

\section{Experiments}
\label{sec:results}
\setlength{\floatsep}{8pt plus 2pt minus 2pt}
\setlength{\textfloatsep}{10pt plus 2pt minus 2pt}

\paragraph{Baselines}
We compare against the multi-bit watermarking baselines \textbf{CycleShift} \citep{fernandez2023three}, \textbf{DepthW} \citep{li-etal-2024-identifying}, \textbf{StealthInk} \citep{jiang-etal-2025-stealthink}, \textbf{MPAC} \citep{yoo-etal-2024-advancing}, \textbf{RSBH} \citep{qu-etal-2025-provably}, and \textbf{XMark} \citep{xu-etal-2026-xmark}. Appendix~\ref{app:mirrormark} adds \textbf{BiMark} \citep{feng2025bimark} and \textbf{MirrorMark} \citep{jiang2026mirrormark} and compares the distortion-free schemes on Qwen3.5-4B.

\paragraph{Setup}
We evaluate on Qwen3.5-4B and Llama-3.1-8B. Full experimental settings and hyperparameters are provided in Appendix~\ref{app:experimental_setup}, and a runtime comparison in Appendix~\ref{app:runtime}.

\paragraph{Metrics}
\textbf{Bit accuracy} (BA, \%) is the percentage of correctly decoded message bits, averaged over users (chance is $50\%$). \textbf{Perplexity} (PPL) is the exponentiated mean negative log-likelihood of the generated tokens under a separate oracle model of the same family; for a distortion-free scheme the reference is the unwatermarked row. \textbf{BLEU} \citep{papineni-etal-2002-bleu} scores translations against the WMT14 references. The distortion metrics compare the watermarked text with unwatermarked sampling: \textbf{Top-1/Top-5} is the percentage of emitted tokens among the model's $1$ or $5$ most likely next tokens, while \textbf{R-1/R-L} (ROUGE-1/ROUGE-L F-score, \citealp{lin-2004-rouge}) and \textbf{BERTScore} (BSc., \citealp{zhang2020bertscore}) measure the similarity of each watermarked text to the unwatermarked sample from the same prompt and seed. Appendix~\ref{app:metrics} gives full definitions.

\begin{figure}[t]
  \centering
  \includegraphics[width=0.80\textwidth,trim=0 6.5pt 0 6.5pt,clip]{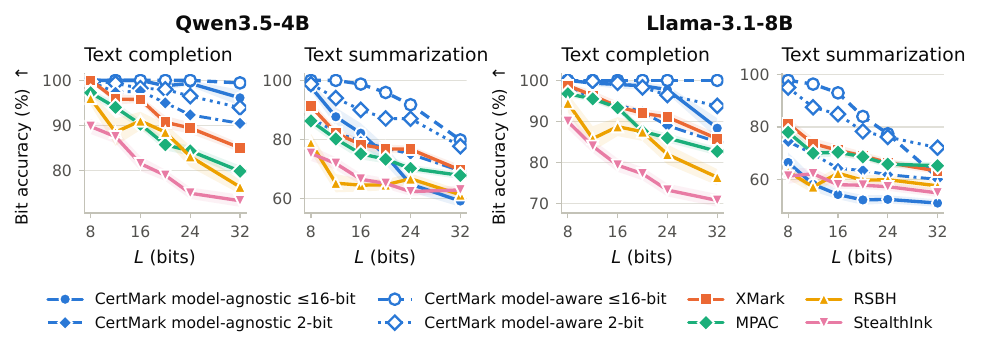}
  \caption{Bit accuracy versus message length $L$ at $T=150$ on text completion and summarization, one block per model. The default chunk width is the largest divisor of $L$ that is at most $16$ bits, hence $8,12,16,10,12,16$ for $L=8,\dots,32$. $\pm1$ standard error shaded.}
  \label{fig:bits_xmodel}
\end{figure}

\paragraph{Main benchmark and message length.} Across all three tasks, \ours{}
preserves unwatermarked text quality while reliably recovering payloads.
Figure~\ref{fig:bits_xmodel} shows that this remains true as payloads grow: the
model-aware decoder is the most reliable, while smaller chunks mainly help on
entropy-limited summarization. Appendix~\ref{app:mirrormark} compares other
distortion-free schemes, and Appendix~\ref{app:_llama318b} confirms the same
trends on Llama-3.1-8B.

\begin{table}[t]
\centering\scriptsize\setlength{\tabcolsep}{2.5pt}\renewcommand{\arraystretch}{0.85}
\begin{minipage}[t]{0.49\textwidth}
\centering
\caption{WMT14 German-to-English translation.}
\label{tab:mt}
\resizebox{\linewidth}{!}{%
\begin{tabular}{lrrrrrr}
\toprule
& \multicolumn{3}{c}{Qwen3.5-4B} & \multicolumn{3}{c}{Llama-3.1-8B} \\
\cmidrule(lr){2-4}\cmidrule(lr){5-7}
Method & BA$\uparrow$ & PPL$\downarrow$ & BLEU$\uparrow$ & BA$\uparrow$ & PPL$\downarrow$ & BLEU$\uparrow$ \\
\midrule
RSBH & 59.38 & 2.51 & 20.64 & 57.62 & 3.49 & 24.84 \\
MPAC & 70.00 & 2.58 & 20.98 & 64.62 & 3.50 & \yellowc 26.80 \\
XMark & \yellowc 74.38 & \yellowc 2.49 & 20.32 & \yellowc 71.25 & 3.46 & 25.99 \\
StealthInk & 64.88 & \orangec 2.22 & \orangec 23.00 & 55.75 & \orangec 3.19 & \redc 27.08 \\
\cmidrule(lr){1-7}
\oursAgnK{16} & 61.50 & \redc 2.18 & \redc 23.15 & 48.88 & \redc 3.08 & 26.34 \\
\oursAgnTwo{} & 69.88 & \orangec 2.22 & \yellowc 22.41 & 64.00 & \yellowc 3.24 & \orangec 26.82 \\
\oursMAK{16} & \redc 85.75 & \redc 2.18 & \redc 23.15 & \redc 88.38 & \redc 3.08 & 26.34 \\
\oursMATwo{} & \orangec 80.50 & \orangec 2.22 & \yellowc 22.41 & \orangec 84.75 & \yellowc 3.24 & \orangec 26.82 \\
\midrule
unwatermarked & -- & 2.21 & 22.87 & -- & 3.12 & 26.70 \\
\bottomrule
\end{tabular}}
\end{minipage}\hfill
\begin{minipage}[t]{0.49\textwidth}
\centering
\caption{Long C4 completion ($L=64$).}
\label{tab:long}
\resizebox{\linewidth}{!}{%
\begin{tabular}{lrrrrrr}
\toprule
& \multicolumn{3}{c}{Qwen3.5-4B} & \multicolumn{3}{c}{Llama-3.1-8B} \\
\cmidrule(lr){2-4}\cmidrule(lr){5-7}
Method & BA$\uparrow$ & PPL$\downarrow$ & BSc.$\uparrow$ & BA$\uparrow$ & PPL$\downarrow$ & BSc.$\uparrow$ \\
\midrule
RSBH & 85.03 & 5.38 & .8208 & 89.25 & 4.63 & .8341 \\
MPAC & 87.28 & 5.46 & .8170 & 88.38 & 4.94 & .8279 \\
XMark & 93.38 & 5.48 & .8199 & 96.00 & 4.74 & .8351 \\
StealthInk & 79.81 & 4.52 & .8222 & 80.69 & \yellowc 3.75 & \redc .8392 \\
\cmidrule(lr){1-7}
\oursAgnK{16} & \redc 99.75 & \orangec 4.29 & \yellowc .8262 & \orangec 99.44 & \orangec 3.66 & \yellowc .8365 \\
\oursAgnTwo{} & \yellowc 96.53 & \yellowc 4.30 & \orangec .8266 & 94.62 & \redc 3.63 & \orangec .8371 \\
\oursMAK{16} & \redc 99.75 & 4.31 & .8246 & \redc 100.00 & \orangec 3.66 & \yellowc .8365 \\
\oursMATwo{} & \orangec 98.84 & \redc 4.28 & \redc .8292 & \yellowc 99.25 & \redc 3.63 & \orangec .8371 \\
\midrule
unwatermarked & -- & 4.22 & -- & -- & 3.57 & -- \\
\bottomrule
\end{tabular}}
\end{minipage}
\end{table}

\paragraph{Machine translation.} Table~\ref{tab:mt} evaluates WMT14 German-to-English translation~\citep{bojar-etal-2014-findings}. \oursMAK{16} gives the strongest payload recovery on both models while preserving translation quality. 2-bit chunks help the model-agnostic decoder in this low-entropy task but not the model-aware one.

\paragraph{Long messages and long texts.} Table~\ref{tab:long} shows that \ours{} retains near-complete recovery on long C4 completions ($L=64$, four 16-bit chunks) for both models without degrading quality. Model-aware decoding removes the remaining errors. Smaller chunks add little with ample evidence.

\begin{figure}[t]
  \centering
  \includegraphics[width=0.80\textwidth,trim=0 6.5pt 0 6.5pt,clip]{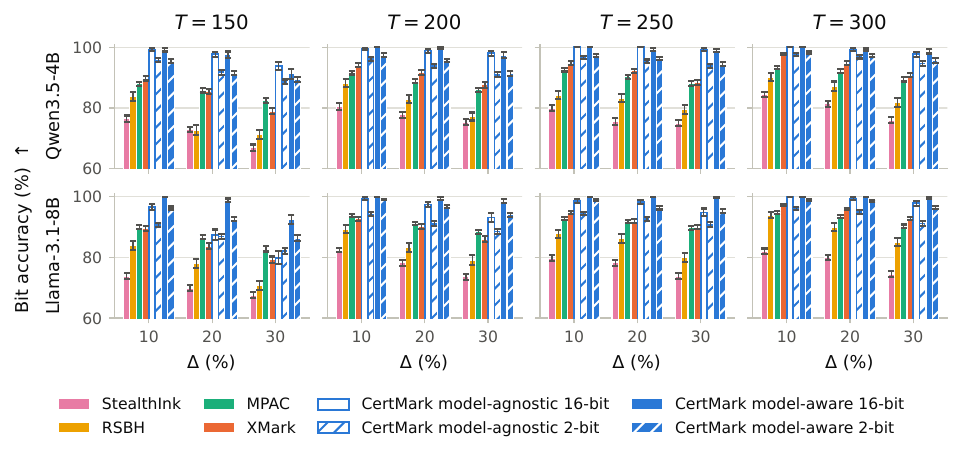}
  \caption{Copy-paste robustness on text completion with $L=16$: bit accuracy after $\Delta\%$ of each text is replaced by unwatermarked spans, one row per model and one column per token budget $T$.}
  \label{fig:attacks}
\end{figure}

\begin{table}[t]
\centering\scriptsize\setlength{\tabcolsep}{3pt}\renewcommand{\arraystretch}{0.78}
\caption{Empirical check of the certificate (Theorem~\ref{thm:cert}, and its model-aware form in Appendix~\ref{app:macert}). Each column fixes a level $\delta$: the decoder returns a chunk's message only if $\hat\delta_c\le\delta$ and abstains ($\bot$) otherwise. Both decoders read the same saved generations of the message-length sweep ($T=150$; $L\in\{8,\dots,32\}$ on completion and summarization, $L=8$ on story generation; $50$ users each), $950$ watermarked chunks and $950$ paired unwatermarked chunks per model. Every entry is a fraction of chunks. \emph{Cert.\ error}: watermarked chunks on which a wrong message is returned; the theorem bounds its probability by $\delta$, and with $950$ chunks a single error is $.0011$. \emph{Null cert.\ rate}: unwatermarked chunks that are nevertheless certified, which should stay near $\delta$. \emph{Abst.}: watermarked chunks of each task on which the decoder abstains; lower means more messages recovered with a certificate. The model-aware decoder uses the robust score with $\varepsilon=0.1$. For every row, lower is better.}
\label{tab:cert}
\begin{tabular}{lcccccccccccccc}
\toprule
& \multicolumn{7}{c}{\oursAgn{}} & \multicolumn{7}{c}{\oursMA{}} \\
\cmidrule(lr){2-8}\cmidrule(lr){9-15}
$\delta$ & $10^{-4}$ & $10^{-3}$ & $0.01$ & $0.05$ & $0.1$ & $0.2$ & $0.5$ & $10^{-4}$ & $10^{-3}$ & $0.01$ & $0.05$ & $0.1$ & $0.2$ & $0.5$ \\
\midrule
\multicolumn{15}{l}{\textit{Qwen3.5-4B}} \\
cert.\ error ($\le\delta$) & .0000 & .0000 & .0063 & .0221 & .0305 & .0537 & .1168 & .0000 & .0011 & .0074 & .0147 & .0168 & .0274 & .0474 \\
null cert.\ rate ($\approx\delta$) & .0021 & .0063 & .0137 & .0621 & .1200 & .2147 & .4074 & .0000 & .0000 & .0074 & .0442 & .1000 & .1895 & .3863 \\
abst., completion & .178 & .129 & .082 & .056 & .044 & .027 & .013 & .040 & .022 & .009 & .004 & .004 & .002 & .000 \\
abst., summarization & .902 & .840 & .727 & .627 & .582 & .498 & .324 & .604 & .509 & .387 & .251 & .207 & .158 & .069 \\
abst., story & .160 & .060 & .020 & .000 & .000 & .000 & .000 & .000 & .000 & .000 & .000 & .000 & .000 & .000 \\
\midrule
\multicolumn{15}{l}{\textit{Llama-3.1-8B}} \\
cert.\ error ($\le\delta$) & .0000 & .0000 & .0032 & .0211 & .0463 & .1053 & .2126 & .0000 & .0011 & .0042 & .0116 & .0189 & .0411 & .0821 \\
null cert.\ rate ($\approx\delta$) & .0000 & .0011 & .0116 & .0463 & .0968 & .1789 & .3989 & .0000 & .0011 & .0116 & .0463 & .0979 & .1842 & .3958 \\
abst., completion & .384 & .296 & .196 & .124 & .100 & .080 & .047 & .009 & .004 & .000 & .000 & .000 & .000 & .000 \\
abst., summarization & 1.000 & 1.000 & .993 & .951 & .889 & .771 & .569 & .856 & .776 & .662 & .536 & .482 & .371 & .238 \\
abst., story & .980 & .940 & .800 & .580 & .500 & .400 & .160 & .060 & .000 & .000 & .000 & .000 & .000 & .000 \\
\bottomrule
\end{tabular}
\end{table}

\paragraph{Robustness to text-editing attacks.} Figure~\ref{fig:attacks} evaluates copy-paste attacks across both models and token budgets. \ours{} remains robust and is the strongest method overall. Appendix~\ref{app:attacks} extends the comparison to token substitution and deletion.

\paragraph{Empirical check of the certificate.}
\label{par:certcheck}
Table~\ref{tab:cert} checks Theorem~\ref{thm:cert}, and its model-aware form in Appendix~\ref{app:macert}, on saved watermarked chunks and paired unwatermarked chunks; its caption defines each row. The certified error stays within $\delta$ up to sampling noise, and the null certification rate tracks the selected level. Relaxing the level reduces abstention but permits more erroneous and null certifications, exhibiting the expected confidence--coverage trade-off. The model-aware decoder abstains less across tasks, particularly when the available evidence is weak, showing that access to the model probabilities improves coverage while retaining the certificate.

\section{Conclusion}

\ours{} is, to our knowledge, the first distortion-free multi-bit watermark with certified recovery. It improves bit accuracy without increasing perplexity and remains fast to encode and decode.

\clearpage
\subsection*{AI use statement}

Large language models were used for editorial support, including polishing the
manuscript's writing and presentation, and for drafting implementation code.
The authors made all substantive research decisions and contributed all core
technical ideas.

\subsection*{Ethics statement}

This research adheres to the ICLR Code of Ethics. All experiments use publicly
available benchmark datasets and
do not involve new data collection or human-subject experiments; no personally
identifiable information is intentionally used. Our method is intended to
support provenance and accountability by embedding recoverable metadata in
language-model outputs. Multi-bit watermarking is nevertheless dual-use: it
could facilitate covert user tracking, and a compromised key or forged payload
could lead to false attribution. The certified decoder bounds the probability
of an incorrect recovery only under the assumptions stated in the paper and
does not eliminate these deployment risks. Responsible use therefore requires
secure key management, appropriate notice and consent, access controls, clear
communication of abstentions and uncertainty, and safeguards against treating
watermark evidence alone as conclusive in high-stakes decisions. We report our
assumptions, limitations, and robustness results transparently and have taken
care to follow best practices for research integrity.

\subsection*{Reproducibility statement}

Our code is available at \url{https://github.com/Batorskq/CertMark}. It contains the
\ours{} encoder, both decoders and their certificates, integrated into the
XMark evaluation harness, together with one script per reported \ours{}
experiment and the fixed key, random seeds and user messages.
Appendix~\ref{app:experimental_setup} provides details of the models, datasets,
experimental protocol, evaluation metrics, and hyperparameters, while
Appendix~\ref{app:runtime} describes the runtime evaluation. Additional model,
baseline, and robustness results appear in Appendices~\ref{app:_llama318b}--\ref{app:attacks}.
Appendix~\ref{app:pseudocode} gives pseudocode for the encoder, both decoders,
and the certification procedures, and Appendix~\ref{app:prf} specifies the
pseudorandom-function instantiation. Full derivations of the theoretical
results are provided in Appendix~\ref{app:macert} and
Appendices~\ref{app:gumbel}--\ref{app:additional}. To the best of our knowledge,
the manuscript and accompanying artifacts contain the details necessary to
reproduce our results.

\bibliographystyle{iclr2027_conference}
\bibliography{custom}

\appendix

\section{The certificate of the model-aware decoder}
\label{app:macert}

The certificate of Theorem~\ref{thm:cert} rests on one fact: under a wrong candidate the scored uniforms are independent $U(0,1)$ variables whatever the text, so the text-only score has the null law $\Gamma(n_c,1)$ and the certificate is the number of wrong candidates times the Gamma tail at the observed score. The model-aware decoder changes the per-token score, not this fact. Its certificate therefore has the same form, with the Gamma tail replaced by the null tail of the new score, and it needs the same two extra inputs as the decoder itself, the prompt and the deployed sampler. This appendix spells the construction out in equations; Appendices~\ref{app:llr} and~\ref{app:robust} hold the proofs.

Throughout, fix a chunk $c$ and condition on the text $y$, the scored-position sets $\mathbf{T}$, the prompt $x$ and the model. This fixes $p_t=p(y_t\mid x,y_{<t})\in(0,1]$ for $t\in T_c$ and, with it, the per-token likelihood ratio $f_p(u)=p^{-1}u^{1/p-1}$ of \eqref{eq:lr_token}.

\paragraph{Score.} The model-aware score \eqref{eq:llr} splits into a candidate-independent constant and a weighted sum of log-uniforms,
\begin{equation}
\begin{gathered}
\ell_c(m')=A_c+R_{c,w}(m'),\qquad w_t=\frac{1}{p_t}-1\ge0,\\
A_c=\sum_{t\in T_c}\log\frac{1}{p_t},\qquad
R_{c,w}(m')=\sum_{t\in T_c}w_t\log u^{(m')}_t(y_t).
\end{gathered}
\label{eq:macert_split}
\end{equation}
Since $A_c$ is the same for every candidate, $\hat m_c=\arg\max_{m'}\ell_c(m')=\arg\max_{m'}R_{c,w}(m')$, and the certificate is built from $R_{c,w}$ alone. A token the sampler was certain of, $p_t=1$, has $w_t=0$ and drops out; write $T_c^{+}=\{t\in T_c:w_t>0\}$.

\paragraph{Null law.} For a wrong candidate $m'\neq m_c$, the uniforms $u^{(m')}_t(y_t)$, $t\in T_c$, are read at PRF inputs the encoder never queried, so they are independent $U(0,1)$ variables conditional on $(y,\mathbf{T})$, exactly as in Theorem~\ref{thm:null}; conditioning on $x$ and the model adds nothing, because $p_t$ is a function of the text and the prompt. With $E_t=-\log u^{(m')}_t(y_t)\sim\mathrm{Exp}(1)$ independent,
\begin{equation}
\begin{gathered}
Z_{c,w}:=-R_{c,w}(m')=\sum_{t\in T_c^{+}}w_tE_t,\qquad
F_{c,w}(z)=\Pr[Z_{c,w}\le z],\\
\mathbb{E}\bigl[e^{-sZ_{c,w}}\bigr]=\prod_{t\in T_c^{+}}\frac{1}{1+sw_t}\qquad(s\ge0),
\end{gathered}
\label{eq:macert_null}
\end{equation}
a weighted sum of independent unit exponentials whose law depends on the text only through the weights. Its mean is $\sum_t w_t=\sum_t(1-p_t)/p_t$. For the correct candidate, $u_t(y_t)\sim\mathrm{Beta}(1/p_t,1)$ by Theorem~\ref{thm:gain}, so $-\log u_t(y_t)\sim\mathrm{Exp}(1/p_t)$ and $-R_{c,w}(m_c)$ has mean $\sum_t w_tp_t=\sum_t(1-p_t)$, smaller by the factor $p_t$ at every token: the less likely the emitted token, the wider the gap the score opens between the correct candidate and every wrong one.

\paragraph{Certificate and decision.} A large score is a small $Z_{c,w}$, so the upper tail of the score is the lower tail of the null law, $\Pr[R_{c,w}(m')\ge r]=F_{c,w}(-r)$. The certificate and the certified decoder are
\begin{equation}
\begin{gathered}
\hat\delta_{c,w}=\min\Bigl\{1,\;(2^{k_c}-1)\,F_{c,w}\bigl(-R_{c,w}(\hat m_c)\bigr)\Bigr\}
=\min\Bigl\{1,\;(2^{k_c}-1)\,F_{c,w}\bigl(A_c-\ell_c(\hat m_c)\bigr)\Bigr\},\\
D_\delta(y)=\begin{cases}\hat m_c,&\hat\delta_{c,w}\le\delta,\\[2pt] \bot,&\hat\delta_{c,w}>\delta,\end{cases}
\end{gathered}
\label{eq:macert_cert}
\end{equation}
with $\hat\delta_{c,w}=1$ when $T_c^{+}=\varnothing$. This is \eqref{eq:certdef}--\eqref{eq:rule} with the Gamma tail replaced by $F_{c,w}$, and the guarantee is the same:
\begin{equation}
\Pr_\kappa\bigl[D_\delta(Y_{1:n})\notin\{m_c,\bot\}\bigr]\le\delta .
\label{eq:macert_guarantee}
\end{equation}
The argument is that of Theorem~\ref{thm:cert}. For $\delta<1$ let $z_\delta$ solve $F_{c,w}(z_\delta)=\delta/(2^{k_c}-1)$; then $\hat\delta_{c,w}\le\delta$ exactly when $-R_{c,w}(\hat m_c)\le z_\delta$, a wrong candidate can be accepted only if its own $-R_{c,w}$ is at most $z_\delta$, and the union bound over the $2^{k_c}-1$ wrong candidates gives $(2^{k_c}-1)F_{c,w}(z_\delta)=\delta$. Decoding every chunk at level $\delta/C$ bounds the probability of any wrong chunk in the message by $\delta$.

\paragraph{Computing the tail.} $F_{c,w}$ is the distribution function of a sum of up to a few hundred independent exponentials with distinct rates. Its closed form, an alternating sum of exponentials, is numerically unusable at this size, and the convolution recursion of Appendix~\ref{app:llr} is exact but slow. The implementation evaluates it through the cumulant generating function of the null score,
\begin{equation}
\begin{gathered}
K(\lambda)=\log\mathbb{E}\bigl[e^{\lambda R_{c,w}(m')}\bigr]=-\sum_{t\in T_c^{+}}\log(1+\lambda w_t),\\
K'(\lambda)=-\sum_{t\in T_c^{+}}\frac{w_t}{1+\lambda w_t},\qquad
K''(\lambda)=\sum_{t\in T_c^{+}}\frac{w_t^{2}}{(1+\lambda w_t)^{2}},
\end{gathered}
\label{eq:macert_cgf}
\end{equation}
by the Lugannani--Rice saddlepoint approximation. For the observed score $r=R_{c,w}(\hat m_c)$ above the null mean, let $\hat\lambda>0$ solve $K'(\hat\lambda)=r$; then
\begin{equation}
\begin{gathered}
F_{c,w}(-r)=\Pr[R_{c,w}(m')\ge r]\;\approx\;1-\Phi(\hat w)-\phi(\hat w)\Bigl(\frac{1}{\hat w}-\frac{1}{\hat u}\Bigr),\\
\hat w=\sqrt{2\bigl(\hat\lambda r-K(\hat\lambda)\bigr)},\qquad \hat u=\hat\lambda\sqrt{K''(\hat\lambda)},
\end{gathered}
\label{eq:macert_saddle}
\end{equation}
with $\Phi$ and $\phi$ the standard normal distribution function and density. The approximation is of an exactly known law, not of a modelling assumption, and it is most accurate in the far tail, where certificates are decided; the exact recursion remains available when an exact value is required.

\paragraph{The robust decoder.} With $\varepsilon>0$ the decoder scores $r_c(m')=\sum_{t\in T_c}s_{p_t,\varepsilon}\bigl(u^{(m')}_t(y_t)\bigr)$ with $s_{p,\varepsilon}(u)=\log[(1-\varepsilon)f_p(u)+\varepsilon]$, the score \eqref{eq:robust}. The wrong-candidate uniforms are $U(0,1)$ whatever was done to the text, so the null law, its upper tail and the certificate are
\begin{equation}
\begin{gathered}
Z_{c,\varepsilon}=\sum_{t\in T_c}s_{p_t,\varepsilon}(U_t),\qquad U_t\ \text{i.i.d.}\ U(0,1),\qquad
Q_{c,\varepsilon}(s)=\Pr[Z_{c,\varepsilon}\ge s],\\
\hat\delta_{c,\varepsilon}=\min\Bigl\{1,\;(2^{k_c}-1)\,Q_{c,\varepsilon}\bigl(r_c(\hat m_c)\bigr)\Bigr\},
\end{gathered}
\label{eq:macert_robust}
\end{equation}
and $D_\delta$ returns $\hat m_c$ when $\hat\delta_{c,\varepsilon}\le\delta$, with the guarantee \eqref{eq:macert_guarantee} by the same union bound. Each summand is bounded, $\log\varepsilon\le s_{p,\varepsilon}(u)\le\log\bigl(\varepsilon+(1-\varepsilon)/p\bigr)$, and for $p<1$ has the explicit distribution function
\begin{equation}
\Pr\bigl[s_{p,\varepsilon}(U)\le x\bigr]=\Bigl(\frac{p\,(e^{x}-\varepsilon)}{1-\varepsilon}\Bigr)^{p/(1-p)},
\qquad \log\varepsilon<x<\log\Bigl(\varepsilon+\frac{1-\varepsilon}{p}\Bigr),
\label{eq:macert_summand}
\end{equation}
so $Q_{c,\varepsilon}$ follows from the convolution recursion of Appendix~\ref{app:robust}. The implementation carries that recursion out on a lattice of step $\Delta=0.002$: the mass that \eqref{eq:macert_summand} assigns to each cell $((k-1)\Delta,k\Delta]$ is placed at the cell's upper end, the lattice masses of the summands are convolved by the fast Fourier transform, and the tail is summed from the first lattice point at or above the observed score. Rounding every summand upwards makes the computed tail an upper bound on $Q_{c,\varepsilon}$, so the certificate stays valid and is at most slightly conservative; it agrees with a $2\times10^{6}$-draw Monte Carlo estimate to within a few percent and, unlike Monte Carlo, has no resolution floor, which matters for $2^{16}-1$ wrong candidates. The contamination model of Theorem~\ref{thm:robust} shapes the score, not the guarantee: the null law in \eqref{eq:macert_robust} holds for every wrong candidate on every text, so the certificate of the robust decoder remains valid on edited text, and the floor $\log\varepsilon$ on every summand is what keeps a single pasted low-probability token from deciding the chunk. All model-aware results in this paper use the robust score with $\varepsilon=0.1$.

\section{Experimental setup and hyperparameters}
\label{app:experimental_setup}

We evaluate Qwen3.5-4B and Llama-3.1-8B, using their instruction variants when required by the task. Following the XMark protocol, we test text completion on C4 \citep{raffel2020exploring}, text summarization on CNN/DailyMail \citep{hermann2015teaching}, and story generation on WritingPrompts \citep{fan-etal-2018-hierarchical}, with 50 users, 8-bit messages, and token budgets $T\in\{150,200,250,300\}$. We use temperature $0.7$, top-$p$ sampling with $p=0.9$, $\delta=2$ for biased baselines, and a separate model from the same family to measure perplexity. Appendix~\ref{app:metrics} defines every reported metric. Every experiment reports the same four \ours{} variants, always in this order: \oursAgn{}, the decoder of Section~\ref{sec:method}, and \oursMA{}, the decoder of Section~\ref{sec:modelaware}, each with the default chunking, whose label gives the chunk width in bits, and with 2-bit chunks. The default chunk width is the largest divisor of $L$ that is at most $16$, so one chunk carries the whole message for $L\le16$. The two decoders read the generations of the same encoder.

\section{Evaluation metrics}
\label{app:metrics}

Every number in the paper is computed by the harness of \citet{xu-etal-2026-xmark}, or offline from the generations it saves, with one recipe for all methods. A run has $50$ users. User $i$ holds an $L$-bit message $m_i$ and two prompts, each answered with exactly $T/2$ tokens, so $T$ is the number of watermarked tokens from which the decoder recovers one message. Each watermarked text is paired with an unwatermarked sample of the same model from the same prompt under the same generation seed and decoding settings. Unless stated otherwise, a reported value is the mean over users, and shaded bands and error bars are $\pm1$ standard error over users.

\paragraph{Bit accuracy (BA).} The decoder reads the two texts of user $i$ and returns an $L$-bit message $\hat m_i$, for \ours{} the concatenation of the chunk decisions of Section~\ref{sec:method}. Bit accuracy is the percentage of correctly recovered bits,
\begin{equation*}
\mathrm{BA}=\frac{100}{50\,L}\sum_{i=1}^{50}\sum_{j=1}^{L}\mathbf{1}\bigl[\hat m_{i,j}=m_{i,j}\bigr],
\end{equation*}
so chance is $50\%$ and $100\%$ means that every user's message was recovered exactly. Everywhere except in Table~\ref{tab:cert} the decoders always return a message, that is level $\delta=1$ without abstention. The ``Avg.'' columns and the distortion columns of Table~\ref{tab:xmark_table1} average over the four token budgets. Under an attack (Figures~\ref{fig:attacks} and~\ref{fig:attacks_edit}) BA is computed on the edited texts and additionally averaged over three random draws of the edit.

\paragraph{Perplexity (PPL).} Fluency is scored by an oracle model of the same family that did not generate the text: Qwen3.5-9B for Qwen3.5-4B, and for Llama-3.1-8B the other 8B sibling, so Llama-3.1-8B-Instruct scores the completions of the base model and the base model scores the two chat tasks. The oracle is run on prompt and response with the prompt tokens masked, and the perplexity of a user is $\exp$ of the mean negative log-likelihood of the generated tokens, pooled over the user's two texts; PPL is the mean over users. Lower is better, but the reference for a distortion-free scheme is the ``unwatermarked'' row, the same quantity on the paired unwatermarked samples, rather than the smallest attainable value. The two decoders of \ours{} read the same generations, so they share PPL and every distortion measure below.

\paragraph{Top-$k$ agreement.} At every step of watermarked generation the harness records the $k$ most likely next tokens under the model's own distribution for the current watermarked prefix, before the watermark acts on the logits. Top-$k$ is the percentage of emitted tokens that lie in this set, for $k\in\{1,5,10\}$, averaged over texts and users. It measures how far the watermark pulls the sampler away from the model's preferences. Since the unwatermarked sampler at temperature $0.7$ emits its top token only part of the time, a distortion-free scheme matches the rate of unwatermarked sampling rather than $100\%$.

\paragraph{BERTScore (BSc.).} BERTScore F1 \citep{zhang2020bertscore} between each watermarked text and its paired unwatermarked sample, with the package's default English backbone (RoBERTa-large, layer 17), averaged over the texts of a user and then over users. It measures semantic agreement with the untouched sampler's output, not quality against a human reference.

\paragraph{ROUGE.} ROUGE-1, ROUGE-2, ROUGE-L and ROUGE-Lsum F-measures \citep{lin-2004-rouge} with Porter stemming, computed with the \texttt{rouge-score} package, the watermarked text as prediction and the paired unwatermarked sample as reference. The harness re-downloads the metric on every call, which fails on our compute nodes and logs NaN, so ROUGE is recomputed offline from the saved generations with the harness's own recipe: the package's bootstrap aggregate over the texts of a user, then the mean over users. Like Top-$k$ and BERTScore, ROUGE here measures agreement with unwatermarked text, so a distortion-free scheme scores like two independent samples of the model rather than $1$.

\paragraph{BLEU.} Only the translation task has references. BLEU \citep{papineni-etal-2002-bleu} is corpus BLEU computed with sacrebleu \citep{post-2018-call} at its defaults (\texttt{13a} tokenization, up to four-grams, exponential smoothing) over all watermarked translations of all users against the English references of WMT14 newstest2014, each translation matched to its reference through the German source in its prompt. The ``unwatermarked'' row is the same computation on the paired unwatermarked translations. BLEU is therefore the one text metric that scores quality against a human reference.

\paragraph{Certificate metrics.} For Table~\ref{tab:cert} the certified decoder at level $\delta$ returns its message only if $\hat\delta_c\le\delta$ and $\bot$ otherwise. Certified error is the fraction of watermarked chunks on which it returns a wrong message; null certification rate is the fraction of paired unwatermarked chunks on which it returns any message; abstention is the fraction of watermarked chunks of a task on which it returns $\bot$.

\section{Runtime comparison}
\label{app:runtime}

\begin{table}[htbp]
\centering\scriptsize\setlength{\tabcolsep}{2.5pt}\renewcommand{\arraystretch}{0.82}
\caption{Runtime on one H100. Gen.: ms/token; Dec.: s/user.}
\label{tab:runtime}
\begin{tabular}{lrrrr}
\toprule
& \multicolumn{2}{c}{Qwen3.5-4B} & \multicolumn{2}{c}{Llama-3.1-8B} \\
\cmidrule(lr){2-3} \cmidrule(lr){4-5}
Method & Gen. & Dec. & Gen. & Dec. \\
\midrule
CycleShift & 19.6 & .15 & 15.0 & .07 \\
DepthW & 19.0 & 11.56 & 14.9 & 12.03 \\
StealthInk & 20.0 & .18 & 15.3 & .10 \\
MPAC & 18.9 & .03 & 14.8 & .03 \\
RSBH & 19.7 & 2.18 & 15.1 & 1.18 \\
XMark & 54.9 & .28 & 27.6 & .15 \\
\cmidrule(lr){1-5}
\oursAgnK{8} & 20.1 & .01 & 15.5 & .01 \\
\oursAgnTwo{} & 20.2 & .01 & 15.5 & .01 \\
\oursMAK{8} & 20.1 & .42 & 15.5 & .20 \\
\oursMATwo{} & 20.2 & .35 & 15.5 & .19 \\
\midrule
unwatermarked & 19.1 & -- & 14.1 & -- \\
\bottomrule
\end{tabular}
\end{table}

Table~\ref{tab:runtime} measures generation time per emitted token and decoding time per user for two $75$-token completions, averaged over $20$ users. \ours{} generates near the unwatermarked speed, and \oursAgnK{8} is the fastest decoder tested. \oursMA{} adds a model pass but remains faster than the heavier baselines.

\section{Main benchmark on Llama-3.1-8B}
\label{app:_llama318b}
Table~\ref{tab:table1_llama318b} gives the large three-task benchmark table for Llama-3.1-8B (completion) and Llama-3.1-8B-Instruct (chat tasks). Its main findings and all other experiments on this model are discussed in Section~\ref{sec:results}.
\begin{table*}[t]
\centering\footnotesize\setlength{\tabcolsep}{2.5pt}
\caption{Table~\ref{tab:xmark_table1} repeated with Llama-3.1-8B (completion) and Llama-3.1-8B-Instruct (chat tasks) as the generating model, with perplexity under the other 8B sibling of the family since the 70B oracle does not fit one GPU. Every number is measured in our harness; layout and colours follow the main table.}
\label{tab:table1_llama318b}
\resizebox{\ifdim\width>\textwidth\textwidth\else\width\fi}{!}{%
\begin{tabular}{lcccccccccccccc}
\toprule
\multirow{2}{*}{Method} & \multicolumn{2}{c}{$T=150$} & \multicolumn{2}{c}{$T=200$} & \multicolumn{2}{c}{$T=250$} & \multicolumn{2}{c}{$T=300$} & \multicolumn{2}{c}{Avg.} & \multicolumn{4}{c}{Distortion (avg.\ over $T$)} \\
\cmidrule(lr){2-3} \cmidrule(lr){4-5} \cmidrule(lr){6-7} \cmidrule(lr){8-9} \cmidrule(lr){10-11} \cmidrule(lr){12-15}
 & BA$\uparrow$ & PPL$\downarrow$ & BA$\uparrow$ & PPL$\downarrow$ & BA$\uparrow$ & PPL$\downarrow$ & BA$\uparrow$ & PPL$\downarrow$ & BA$\uparrow$ & PPL$\downarrow$ & Top-1$\uparrow$ & Top-5$\uparrow$ & R-1$\uparrow$ & R-L$\uparrow$ \\
\midrule
\rowcolor[HTML]{EDEDED} \multicolumn{15}{l}{\textbf{Text Completion}} \\
CycleShift & \redc 100.00 & 5.63 & \redc 100.00 & 5.75 & \redc 100.00 & 5.59 & \redc 100.00 & 5.51 & \redc 100.00 & 5.62 & 57.62 & 88.55 & 0.321 & 0.186 \\
DepthW & 98.25 & 5.28 & \yellowc 99.50 & 5.12 & \yellowc 99.00 & 5.12 & \redc 100.00 & 5.19 & \yellowc 99.19 & 5.18 & 60.75 & 89.56 & 0.337 & 0.195 \\
StealthInk & 90.00 & \yellowc 4.61 & 91.00 & \yellowc 4.35 & 94.25 & \yellowc 4.58 & 96.25 & \yellowc 4.28 & 92.88 & \yellowc 4.46 & \yellowc 64.98 & \yellowc 92.37 & \yellowc 0.347 & \yellowc 0.202 \\
MPAC & 96.75 & 5.84 & 98.00 & 5.48 & \orangec 99.25 & 5.45 & \orangec 99.00 & 5.67 & 98.25 & 5.61 & 57.66 & 88.34 & 0.329 & 0.188 \\
RSBH & 94.50 & 5.29 & 96.25 & 5.23 & 95.75 & 5.35 & \yellowc 97.50 & 5.39 & 96.00 & 5.31 & 58.31 & 90.17 & 0.331 & 0.188 \\
XMark & \yellowc 98.75 & 5.63 & \redc 100.00 & 5.15 & \redc 100.00 & 5.04 & \redc 100.00 & 5.05 & \orangec 99.69 & 5.22 & 59.31 & 90.27 & 0.333 & 0.191 \\
\cmidrule(lr){1-15}
\oursAgnK{8} & \redc 100.00 & \orangec 4.29 & \redc 100.00 & \redc 4.16 & \redc 100.00 & \redc 4.09 & \redc 100.00 & \orangec 4.08 & \redc 100.00 & \orangec 4.15 & \redc 68.06 & \redc 93.58 & \orangec 0.358 & \redc 0.215 \\
\oursAgnTwo{} & \orangec 99.00 & \redc 4.28 & \orangec 99.75 & \orangec 4.20 & \redc 100.00 & \orangec 4.11 & \redc 100.00 & \redc 3.99 & \orangec 99.69 & \redc 4.14 & \orangec 67.80 & \orangec 93.47 & \redc 0.359 & \orangec 0.206 \\
\oursMAK{8} & \redc 100.00 & \orangec 4.29 & \redc 100.00 & \redc 4.16 & \redc 100.00 & \redc 4.09 & \redc 100.00 & \orangec 4.08 & \redc 100.00 & \orangec 4.15 & \redc 68.06 & \redc 93.58 & \orangec 0.358 & \redc 0.215 \\
\oursMATwo{} & \redc 100.00 & \redc 4.28 & \redc 100.00 & \orangec 4.20 & \redc 100.00 & \orangec 4.11 & \redc 100.00 & \redc 3.99 & \redc 100.00 & \redc 4.14 & \orangec 67.80 & \orangec 93.47 & \redc 0.359 & \orangec 0.206 \\
\midrule
\rowcolor[HTML]{EDEDED} \multicolumn{15}{l}{\textbf{Text Summarization}} \\
CycleShift & 66.50 & 4.31 & 83.00 & 4.50 & \yellowc 88.00 & 4.58 & 91.75 & 4.80 & 82.31 & 4.55 & 82.77 & 99.08 & 0.613 & 0.453 \\
DepthW & 57.00 & 4.31 & 61.00 & 4.37 & 68.50 & 4.47 & 68.00 & 4.65 & 63.62 & 4.45 & 84.29 & 99.28 & 0.621 & 0.460 \\
StealthInk & 61.25 & \orangec 4.14 & 74.00 & \yellowc 4.26 & 73.50 & \orangec 4.26 & 77.00 & \yellowc 4.35 & 71.44 & \yellowc 4.25 & \yellowc 87.26 & \orangec 99.74 & \yellowc 0.631 & \yellowc 0.477 \\
MPAC & 78.00 & 4.57 & \yellowc 83.50 & 4.51 & 84.00 & 4.53 & 91.25 & 4.74 & 84.19 & 4.59 & 82.53 & 98.99 & 0.611 & 0.451 \\
RSBH & 63.00 & 4.65 & 72.25 & 4.52 & 77.50 & 4.59 & 85.25 & 4.76 & 74.50 & 4.63 & 81.84 & 99.19 & 0.608 & 0.446 \\
XMark & \yellowc 81.25 & 4.47 & 81.50 & 4.52 & \orangec 89.00 & 4.51 & \yellowc 95.50 & 4.69 & \yellowc 86.81 & 4.55 & 82.58 & \yellowc 99.40 & 0.619 & 0.454 \\
\cmidrule(lr){1-15}
\oursAgnK{8} & 66.50 & \redc 4.10 & 73.75 & \redc 4.10 & 77.00 & \redc 4.19 & 87.50 & \redc 4.24 & 76.19 & \redc 4.16 & \redc 88.25 & \redc 99.78 & \orangec 0.633 & \orangec 0.486 \\
\oursAgnTwo{} & 74.50 & \yellowc 4.19 & 78.75 & \orangec 4.19 & 84.00 & \yellowc 4.28 & 87.25 & \orangec 4.30 & 81.12 & \orangec 4.24 & \orangec 87.59 & \orangec 99.74 & \redc 0.636 & \redc 0.490 \\
\oursMAK{8} & \redc 97.75 & \redc 4.10 & \redc 100.00 & \redc 4.10 & \redc 100.00 & \redc 4.19 & \redc 100.00 & \redc 4.24 & \redc 99.44 & \redc 4.16 & \redc 88.25 & \redc 99.78 & \orangec 0.633 & \orangec 0.486 \\
\oursMATwo{} & \orangec 95.00 & \yellowc 4.19 & \orangec 98.75 & \orangec 4.19 & \redc 100.00 & \yellowc 4.28 & \orangec 99.75 & \orangec 4.30 & \orangec 98.38 & \orangec 4.24 & \orangec 87.59 & \orangec 99.74 & \redc 0.636 & \redc 0.490 \\
\midrule
\rowcolor[HTML]{EDEDED} \multicolumn{15}{l}{\textbf{Story Generation}} \\
CycleShift & 85.50 & 9.04 & \yellowc 95.00 & 8.03 & \yellowc 94.75 & 7.33 & \orangec 97.00 & 6.72 & 93.06 & 7.78 & 77.56 & 98.16 & 0.406 & \yellowc 0.264 \\
DepthW & 65.25 & 8.57 & 77.50 & 7.87 & 83.00 & 7.45 & 81.75 & 6.47 & 76.88 & 7.59 & 79.24 & 98.45 & 0.410 & 0.259 \\
StealthInk & 72.50 & \redc 7.86 & 76.00 & \yellowc 7.52 & 79.50 & \redc 6.28 & 80.00 & \orangec 6.03 & 77.00 & \orangec 6.92 & \yellowc 83.28 & \yellowc 99.37 & \yellowc 0.425 & \orangec 0.281 \\
MPAC & 84.00 & 9.06 & 86.00 & 8.10 & 89.75 & 7.16 & \yellowc 94.25 & 6.66 & 88.50 & 7.75 & 77.14 & 97.94 & 0.400 & 0.253 \\
RSBH & 73.75 & 8.85 & 78.75 & 8.25 & 87.25 & 7.47 & 93.00 & 6.80 & 83.19 & 7.84 & 76.72 & 98.55 & 0.402 & 0.253 \\
XMark & \yellowc 90.50 & 8.68 & 92.50 & 7.73 & \orangec 97.25 & 7.20 & \orangec 97.00 & 6.76 & \yellowc 94.31 & 7.59 & 77.42 & 98.68 & 0.404 & 0.254 \\
\cmidrule(lr){1-15}
\oursAgnK{8} & 80.25 & \orangec 7.90 & 88.50 & \redc 7.03 & 93.25 & \orangec 6.35 & 93.75 & \redc 5.91 & 88.94 & \redc 6.80 & \redc 85.19 & \redc 99.48 & \orangec 0.426 & \orangec 0.281 \\
\oursAgnTwo{} & 82.75 & \yellowc 8.32 & 86.75 & \orangec 7.18 & 89.75 & \yellowc 6.51 & 92.00 & \yellowc 6.04 & 87.81 & \yellowc 7.01 & \orangec 85.04 & \orangec 99.42 & \redc 0.430 & \redc 0.288 \\
\oursMAK{8} & \redc 100.00 & \orangec 7.90 & \redc 100.00 & \redc 7.03 & \redc 100.00 & \orangec 6.35 & \redc 100.00 & \redc 5.91 & \redc 100.00 & \redc 6.80 & \redc 85.19 & \redc 99.48 & \orangec 0.426 & \orangec 0.281 \\
\oursMATwo{} & \orangec 99.00 & \yellowc 8.32 & \orangec 99.75 & \orangec 7.18 & \redc 100.00 & \yellowc 6.51 & \redc 100.00 & \yellowc 6.04 & \orangec 99.69 & \yellowc 7.01 & \orangec 85.04 & \orangec 99.42 & \redc 0.430 & \redc 0.288 \\
\bottomrule
\end{tabular}}
\end{table*}
\section{Comparison to other distortion-free methods}
\label{app:mirrormark}

\ours{} is not the first distortion-free watermark. Zero-bit schemes have preserved the sampling law for some time \citep{kuditipudi2024robust,christ2024undetectable,hu2024unbiased,wu2024resilient,dathathri2024scalable}, and StealthInk \citep{jiang-etal-2025-stealthink}, BiMark \citep{feng2025bimark}, DISC \citep{boroujeny2024multibit} and MirrorMark \citep{jiang2026mirrormark} carry a multi-bit payload without altering it, the first two by a keyed reweighting whose mean over the key is the original law, the last two by moving the message into the sampling randomness. To the best of our knowledge, \ours{} is the only method that is distortion-free and also returns a certificate on the decoded message. Table~\ref{tab:table1_qwen354b_mirrormark} compares the multi-bit schemes on Qwen3.5-4B under the protocol of Table~\ref{tab:xmark_table1}: StealthInk repeats its rows of that table, BiMark runs the authors' code with the paper's defaults, ten layers and $\delta=1$, and MirrorMark, which has no public code, is our reimplementation from the paper's description with its defaults, carrying $L=8$ as four 2-bit symbols. Since all of these schemes leave the sampling law unchanged, the perplexity columns agree to within the noise of the $50$-user sample and the differences fall on either side, StealthInk sitting slightly higher on completion and story generation. Bit accuracy is what separates them: \oursAgnK{8} matches or exceeds every baseline in every cell, with the largest lead on summarization, $98.8$ on average against $93.3$ for BiMark, $92.7$ for MirrorMark and $81.6$ for StealthInk, and \oursMA{} recovers every message on all three tasks. Appendix~\ref{app:attacks} includes the three baselines in the attack comparison.

\begin{table*}[htbp]
\centering\footnotesize\setlength{\tabcolsep}{2.5pt}
\caption{Distortion-free multi-bit schemes on Qwen3.5-4B: StealthInk, BiMark, MirrorMark and the four \ours{} rows of Table~\ref{tab:xmark_table1}, under the same protocol, tasks and budgets, with perplexity under Qwen3.5-9B; every number is measured in our harness. Red marks the best value in each column.}
\label{tab:table1_qwen354b_mirrormark}
\resizebox{\ifdim\width>\textwidth\textwidth\else\width\fi}{!}{%
\begin{tabular}{lcccccccccc}
\toprule
\multirow{2}{*}{Method} & \multicolumn{2}{c}{$T=150$} & \multicolumn{2}{c}{$T=200$} & \multicolumn{2}{c}{$T=250$} & \multicolumn{2}{c}{$T=300$} & \multicolumn{2}{c}{Avg.} \\
\cmidrule(lr){2-3} \cmidrule(lr){4-5} \cmidrule(lr){6-7} \cmidrule(lr){8-9} \cmidrule(lr){10-11}
 & BA$\uparrow$ & PPL$\downarrow$ & BA$\uparrow$ & PPL$\downarrow$ & BA$\uparrow$ & PPL$\downarrow$ & BA$\uparrow$ & PPL$\downarrow$ & BA$\uparrow$ & PPL$\downarrow$ \\
\midrule
\rowcolor[HTML]{EDEDED} \multicolumn{11}{l}{\textbf{Text Completion}} \\
StealthInk & 89.75 & 6.15 & 90.50 & 5.79 & 95.25 & 5.73 & 94.25 & 5.61 & 92.44 & 5.82 \\
BiMark & 98.25 & 5.79 & 99.75 & 5.81 & 99.75 & 5.40 & 99.25 & 5.26 & 99.25 & 5.56 \\
MirrorMark & \redc 100.00 & 5.83 & \redc 100.00 & 5.66 & \redc 100.00 & \redc 5.12 & \redc 100.00 & 5.27 & \redc 100.00 & 5.47 \\
\cmidrule(lr){1-11}
\oursAgnK{8} & \redc 100.00 & 5.65 & \redc 100.00 & 5.23 & \redc 100.00 & 5.29 & \redc 100.00 & 5.11 & \redc 100.00 & \redc 5.32 \\
\oursAgnTwo{} & 99.75 & \redc 5.63 & \redc 100.00 & 5.55 & \redc 100.00 & 5.34 & 99.75 & 5.20 & 99.88 & 5.43 \\
\oursMAK{8} & \redc 100.00 & 5.69 & \redc 100.00 & \redc 5.20 & \redc 100.00 & 5.35 & \redc 100.00 & \redc 5.09 & \redc 100.00 & 5.33 \\
\oursMATwo{} & \redc 100.00 & 5.64 & \redc 100.00 & 5.52 & \redc 100.00 & 5.36 & \redc 100.00 & 5.24 & \redc 100.00 & 5.44 \\
\midrule
\rowcolor[HTML]{EDEDED} \multicolumn{11}{l}{\textbf{Text Summarization}} \\
StealthInk & 75.25 & 2.95 & 80.75 & 3.10 & 82.25 & 3.11 & 88.00 & 3.14 & 81.56 & 3.08 \\
BiMark & 89.50 & 2.93 & 92.00 & 3.05 & 95.75 & 3.07 & 96.00 & 3.09 & 93.31 & 3.04 \\
MirrorMark & 86.25 & \redc 2.91 & 90.00 & \redc 2.99 & 97.25 & 3.11 & 97.25 & 3.11 & 92.69 & \redc 3.03 \\
\cmidrule(lr){1-11}
\oursAgnK{8} & 97.75 & 3.05 & 98.75 & 3.05 & 98.50 & \redc 3.06 & \redc 100.00 & \redc 3.08 & 98.75 & 3.06 \\
\oursAgnTwo{} & 87.25 & 2.99 & 94.50 & 3.10 & 95.00 & 3.13 & 97.50 & 3.09 & 93.56 & 3.08 \\
\oursMAK{8} & \redc 100.00 & 3.03 & \redc 100.00 & 3.07 & \redc 100.00 & 3.10 & \redc 100.00 & 3.10 & \redc 100.00 & 3.07 \\
\oursMATwo{} & 98.75 & 3.00 & 99.25 & 3.12 & 99.75 & 3.11 & 99.75 & 3.11 & 99.38 & 3.08 \\
\midrule
\rowcolor[HTML]{EDEDED} \multicolumn{11}{l}{\textbf{Story Generation}} \\
StealthInk & 82.50 & 3.81 & 84.75 & 3.93 & 89.50 & 4.02 & 92.50 & 3.87 & 87.31 & 3.91 \\
BiMark & 95.25 & \redc 3.52 & 97.75 & 3.75 & 98.00 & 3.71 & 98.75 & 3.88 & 97.44 & 3.71 \\
MirrorMark & 98.50 & \redc 3.52 & 98.75 & \redc 3.57 & 99.25 & \redc 3.63 & 99.50 & \redc 3.70 & 99.00 & \redc 3.60 \\
\cmidrule(lr){1-11}
\oursAgnK{8} & \redc 100.00 & 3.63 & \redc 100.00 & 3.71 & \redc 100.00 & 3.81 & \redc 100.00 & 3.87 & \redc 100.00 & 3.75 \\
\oursAgnTwo{} & 95.75 & 3.71 & 98.50 & 3.75 & 99.25 & 3.81 & 99.75 & 3.82 & 98.31 & 3.77 \\
\oursMAK{8} & \redc 100.00 & 3.63 & \redc 100.00 & 3.71 & \redc 100.00 & 3.81 & \redc 100.00 & 3.87 & \redc 100.00 & 3.75 \\
\oursMATwo{} & \redc 100.00 & 3.71 & \redc 100.00 & 3.75 & \redc 100.00 & 3.81 & \redc 100.00 & 3.82 & \redc 100.00 & 3.77 \\
\bottomrule
\end{tabular}}
\end{table*}

\section{Further attacks on Qwen3.5-4B}
\label{app:attacks}

Figure~\ref{fig:attacks_edit} extends Figure~\ref{fig:attacks} on Qwen3.5-4B to two token-level edits, for every method of Table~\ref{tab:xmark_table1} except DepthW, whose decoder enumerates all $2^{L}$ codes and needs about an hour per user at $L=16$, and for BiMark and MirrorMark (Appendix~\ref{app:mirrormark}). CycleShift, BiMark and MirrorMark, absent from Figure~\ref{fig:attacks}, are added, with their copy-paste results in the first row. The protocol is that of Figure~\ref{fig:attacks}: text completion, $L=16$, $50$ users with two texts each, every method decoded by its own text-only detector on the same edited texts, and three random draws of each edit per text; the model-aware \ours{} decoder also reads the prompt.

\begin{figure}[htbp]
  \centering
  \includegraphics[width=\textwidth]{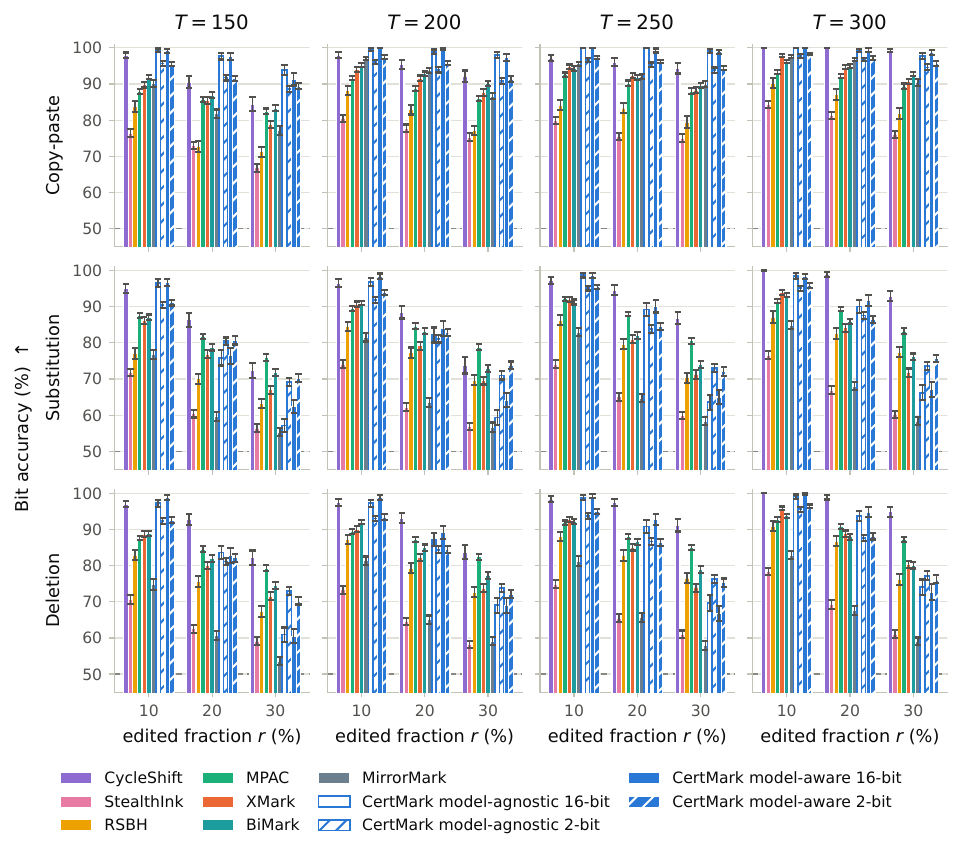}
  \caption{Robustness to token edits on Qwen3.5-4B, text completion, $L=16$: bit accuracy after a fraction $r$ of the tokens of each watermarked text is replaced by spans of unwatermarked text (copy-paste, as in Figure~\ref{fig:attacks}), replaced by random vocabulary tokens, or deleted; one row per attack, one column per token budget $T$. Colours and markers as in Figure~\ref{fig:attacks}; the dashed line is chance. Means over $50$ users and three random draws of the edit, $\pm1$ standard error.}
  \label{fig:attacks_edit}
\end{figure}

\paragraph{Token edits.} The two edits act alike, because each destroys the contexts that overlap the edited position: about four scored positions per edit for the three-token context of \ours{}, two for the one-token contexts of CycleShift and MPAC. Accordingly, \oursMAK{16} and \oursAgnK{16} are the most robust methods at $r=10\%$, between $96.5$ and $99.7$ at every budget, where MPAC stays below $93$ and XMark below $96$; at $20\%$ they still match or exceed XMark, and MPAC from $T=200$ on, but trail CycleShift; and at $30\%$, where a text keeps roughly a quarter of its scored positions, the 16-bit codebook runs out of evidence and they fall below all three, to between $57$ and $74$. The 2-bit chunks degrade more gracefully, since each chunk needs far less evidence, and keep $69$ to $77$ at $30\%$. The model-aware decoder brings no advantage under these edits, because the damage is a loss of contexts rather than a contamination of scored tokens, which is what Theorem~\ref{thm:robust} addresses. CycleShift, the most robust method under heavy editing, pays for it with a bias on every token and the highest perplexity of Table~\ref{tab:xmark_table1}. Of the distortion-free baselines, BiMark degrades like MPAC, from $87$ to $94$ at $10\%$ to $72$ to $80$ at $30\%$, since its two-token context loses few positions per edit and each bit is a majority over its own votes, at the price of never reaching full recovery on clean text. MirrorMark is the least robust: its scheduler assigns positions inside context-anchored frames, so one edit shifts every position up to the next anchor, and it drops to $75$ to $85$ at $10\%$ and to between $53$ and $59$ at $30\%$, close to chance.

\raggedbottom
\section{Pseudocode}
\label{app:pseudocode}

Algorithms~\ref{alg:encoder}--\ref{alg:decoder_ma} restate the method as pseudocode, in the notation of Table~\ref{tab:notation}. There is one encoder: both decoders read its generations and differ only in the per-token score. Each decoder first decodes \emph{without} certification, returning the highest-scoring candidate of every chunk; this is the decoder behind every benchmark table and figure. Given a level $\delta$, it then decodes \emph{with} certification by handing its scores to \textsc{Certify} (Algorithm~\ref{alg:certify}), which keeps the candidate or abstains; this is the decoder of Table~\ref{tab:cert}. Every argument tuple of $\PRF_\kappa$ is encoded by $\langle\cdot\rangle$ of \eqref{eq:argenc}, $\rho$ is the map \eqref{eq:rho}, and the loops over the $2^{k_c}$ candidates are vectorised over the codebook in the implementation.

\begin{algorithm}[H]
\caption{\ours{} encoder, one generation, shared by both decoders; Eqs.~\eqref{eq:chunkidx}--\eqref{eq:emit}.}
\label{alg:encoder}
\begin{algorithmic}[1]
\Require
\algin{key $\kappa$}
\algin{message $m=(m_1,\dots,m_C)$ with $m_c\in\{0,1\}^{k_c}$}
\algin{prompt $x$ and the deployed sampler $p(\cdot\mid x,\cdot)$, after temperature and truncation}
\algin{length $n$ and context width $h$}
\Ensure
\algin{text $y_{1:n}$, distributed exactly as unwatermarked sampling (Theorem~\ref{thm:exact})}
\State $\mathcal{C}\gets\varnothing$ \Comment{contexts used so far in this generation}
\For{$t=1,\dots,n$}
  \State $p_t(\cdot)\gets p(\cdot\mid x,y_{<t})$ \Comment{after temperature and truncation}
  \State $c_t\gets y_{t-h:t-1}$ \Comment{the available prefix when $t\le h$}
  \If{$t\le h$ \textbf{or} $c_t\in\mathcal{C}$}
    \State $a_t\gets(\mathsf{fresh},c_t,t)$ \Comment{carries no message; the position makes the argument new}
  \Else
    \State $\mathcal{C}\gets\mathcal{C}\cup\{c_t\}$;\quad $i_t\gets\PRF_\kappa(\mathsf{chunk},c_t)\bmod C$;\quad $a_t\gets(\mathsf{msg},c_t,m_{i_t})$
  \EndIf
  \State $u_t(v)\gets\rho\bigl(\PRF_\kappa(a_t,v)\bigr)$ for every $v\in V$
  \State $y_t\gets\arg\max_{v\in V}\bigl[\log p_t(v)-\log(-\log u_t(v))\bigr]$ \Comment{Gumbel-max, Lemma~\ref{lem:gumbel}}
\EndFor
\State \Return $y_{1:n}$
\end{algorithmic}
\end{algorithm}

\begin{algorithm}[H]
\caption{\textsc{Certify}: certified decision for one chunk, used by Algorithms~\ref{alg:decoder_agn} and~\ref{alg:decoder_ma}; Eqs.~\eqref{eq:certdef}--\eqref{eq:rule}.}
\label{alg:certify}
\begin{algorithmic}[1]
\Require
\algin{scores $S(m')$, $m'\in\{0,1\}^{k}$, and their maximiser $\hat m$}
\algin{number of scored tokens $n$}
\algin{level $\delta\in(0,1]$}
\algin{upper tail $Q(s)=\Pr[S(m')\ge s]$ of the score under a wrong candidate}
\Ensure
\algin{$\hat m$ when it is certified at level $\delta$, else $\bot$; then $\Pr[\text{output}\notin\{m_c,\bot\}]\le\delta$ (Theorem~\ref{thm:cert})}
\If{$n=0$, or every scored token has $p_t=1$ in Algorithm~\ref{alg:decoder_ma}} $\hat\delta\gets1$
\Else\ $\hat\delta\gets\min\bigl\{1,\;(2^{k}-1)\,Q\bigl(S(\hat m)\bigr)\bigr\}$ \Comment{union bound over the $2^{k}-1$ wrong candidates}
\EndIf
\State \Return $\hat m$ if $\hat\delta\le\delta$, else $\bot$ \Comment{level $\delta/C$ per chunk bounds the whole-message error by $\delta$}
\end{algorithmic}
\end{algorithm}

\begin{algorithm}[H]
\caption{\oursAgn{} decoder, from the text and the key alone; Eqs.~\eqref{eq:visited}--\eqref{eq:score}.}
\label{alg:decoder_agn}
\begin{algorithmic}[1]
\Require
\algin{key $\kappa$}
\algin{text $y_{1:n}$}
\algin{chunk widths $k_1,\dots,k_C$ and context width $h$}
\algin{for certified decoding, a level $\delta\in(0,1]$}
\Ensure
\algin{decoded message $\hat m=(\hat m_1,\dots,\hat m_C)$}
\algin{with certification, each $\hat m_c\in\{0,1\}^{k_c}\cup\{\bot\}$}
\State $S_c(m')\gets0$ for all $c$ and $m'\in\{0,1\}^{k_c}$;\quad $n_c\gets0$;\quad $\mathcal{T}\gets\varnothing$
\For{$t=h+1,\dots,n$}
  \State $c_t\gets y_{t-h:t-1}$;\quad \textbf{if} $c_t=c_s$ for some $s\in\mathcal{T}$ \textbf{then continue} \Comment{repeated context}
  \State $\mathcal{T}\gets\mathcal{T}\cup\{t\}$;\quad $c\gets\PRF_\kappa(\mathsf{chunk},c_t)\bmod C$;\quad $n_c\gets n_c+1$
  \For{$m'\in\{0,1\}^{k_c}$}
    \State $u\gets\rho\bigl(\PRF_\kappa(\mathsf{msg},c_t,m',y_t)\bigr)$;\quad $S_c(m')\gets S_c(m')-\log(1-u)$
  \EndFor
\EndFor
\For{$c=1,\dots,C$}
  \State $\hat m_c\gets\arg\max_{m'}S_c(m')$ \Comment{without certification: the decoder of every benchmark table}
  \State \textbf{if} certifying \textbf{then} $\hat m_c\gets\textsc{Certify}\bigl(S_c(\cdot),\hat m_c,n_c,\delta,Q_c\bigr)$ with $Q_c(s)=\Pr[\Gamma(n_c,1)\ge s]$ (Theorem~\ref{thm:null})
\EndFor
\State \Return $\hat m$
\end{algorithmic}
\end{algorithm}

\begin{algorithm}[H]
\caption{\oursMA{} decoder, from the text, the key, the prompt and the model; Eqs.~\eqref{eq:llr} and~\eqref{eq:robust}.}
\label{alg:decoder_ma}
\begin{algorithmic}[1]
\Require
\algin{as in Algorithm~\ref{alg:decoder_agn}}
\algin{prompt $x$ and the deployed sampler $p$}
\algin{contamination rate $\varepsilon\in[0,1)$}
\Ensure
\algin{decoded message $\hat m=(\hat m_1,\dots,\hat m_C)$}
\algin{with certification, each $\hat m_c\in\{0,1\}^{k_c}\cup\{\bot\}$}
\State $p_t\gets p(y_t\mid x,y_{<t})$ for $t=1,\dots,n$ \Comment{one forward pass over prompt and text}
\State $r_c(m')\gets0$ for all $c$ and $m'$;\quad $\mathcal{T}_c\gets\varnothing$;\quad $\mathcal{T}\gets\varnothing$
\For{$t=h+1,\dots,n$}
  \State $c_t\gets y_{t-h:t-1}$;\quad \textbf{if} $c_t=c_s$ for some $s\in\mathcal{T}$ \textbf{then continue}
  \State $\mathcal{T}\gets\mathcal{T}\cup\{t\}$;\quad $c\gets\PRF_\kappa(\mathsf{chunk},c_t)\bmod C$;\quad $\mathcal{T}_c\gets \mathcal{T}_c\cup\{t\}$
  \For{$m'\in\{0,1\}^{k_c}$}
    \State $u\gets\rho\bigl(\PRF_\kappa(\mathsf{msg},c_t,m',y_t)\bigr)$;\quad $r_c(m')\gets r_c(m')+\log\bigl[(1-\varepsilon)\,p_t^{-1}u^{1/p_t-1}+\varepsilon\bigr]$
  \EndFor
\EndFor
\Statex \hspace{\algorithmicindent} $\varepsilon=0$ gives the log-likelihood ratio $\ell_c$ of \eqref{eq:llr}; $\varepsilon>0$ gives the robust score $r_c$ of \eqref{eq:robust}.
\For{$c=1,\dots,C$}
  \State $\hat m_c\gets\arg\max_{m'}r_c(m')$ \Comment{without certification}
  \State \textbf{if} certifying \textbf{then} $\hat m_c\gets\textsc{Certify}\bigl(r_c(\cdot),\hat m_c,|\mathcal{T}_c|,\delta,Q_c\bigr)$, where $Q_c(s)=F_{c,w}(A_c-s)$ with $w_t=1/p_t-1$ and $A_c=\sum_{t\in \mathcal{T}_c}\log(1/p_t)$ if $\varepsilon=0$ (Theorem~\ref{thm:llr}), and $Q_c=Q_{c,\varepsilon}$ if $\varepsilon>0$ (Theorem~\ref{thm:robust})
\EndFor
\State \Return $\hat m$
\end{algorithmic}
\end{algorithm}

\section{The pseudorandom function}
\label{app:prf}

This appendix gives the details of the pseudorandom function that Section~\ref{sec:method} uses. Let $\PRF_\kappa:\{0,1\}^*\to\mathcal{S}$, $\mathcal{S}=\{0,1\}^{64}$, be a pseudorandom function keyed by $\kappa$ \citep{goldreich1986construct}, where $\{0,1\}^*=\bigcup_{n=0}^{\infty}\{0,1\}^n$ is the set of all finite binary strings. Its values at distinct inputs are computationally indistinguishable from independent uniform elements of $\mathcal{S}$.
To pass structured arguments to the PRF, let $\mathcal{X}$ contain the typed tags, token sequences and integers used in the method, and fix an injective, prefix-free encoding
\begin{equation}
\begin{aligned}
\langle\cdot\rangle&:\bigcup_{r=1}^{\infty}\mathcal{X}^r\to\{0,1\}^*,\\
\PRF_\kappa(x_1,\ldots,x_r)&:=\PRF_\kappa\!\left(\langle x_1,\ldots,x_r\rangle\right).
\end{aligned}
\label{eq:argenc}
\end{equation}
Thus each tuple becomes one finite binary string, and the three tags $\mathsf{chunk}$, $\mathsf{msg}$ and $\mathsf{fresh}$ domain-separate the three uses of the PRF. Interpreting $z\in\mathcal{S}$ as an unsigned integer, we convert its output to
\begin{equation}
\rho(z)=\bigl(\lfloor z/2^{11}\rfloor+\tfrac12\bigr)2^{-53},
\label{eq:rho}
\end{equation}
which is uniform on the $53$-bit midpoint grid in $(0,1)$ when $z$ is uniform. Throughout the analysis we use the standard random-function idealisation, treating $\rho(\PRF_\kappa(a))$ at distinct inputs $a$ as independent $U(0,1)$ variables.

\paragraph{Instantiation.}
Our implementation realises $\PRF_\kappa$ with $64$-bit integer arithmetic. Let $\phi$ be the SplitMix64 finaliser \citep{steele2014fast}, a bijection of $\{0,1\}^{64}$ made of three xor-shifts and two multiplications by fixed odd constants. A context $c_t=(y_{t-h},\dots,y_{t-1})$ is hashed to a seed by starting from $\phi(\kappa\cdot g_1\oplus g_2)$ and absorbing one token at a time, $s\leftarrow\phi\bigl(s\oplus(y+1)\,g_3\bigr)$, with fixed odd constants $g_1,g_2,g_3$. The three uses of the PRF are separated by fixed salts: the chunk index is $\phi(s\oplus\mathsf{salt}_{\mathsf{chunk}})\bmod C$, a message tuple $(\mathsf{msg},c_t,m)$ maps to $\phi\bigl(s\oplus\phi((m+1)\,g_2\oplus\mathsf{salt}_{\mathsf{msg}})\bigr)$, and a fresh tuple $(\mathsf{fresh},c_t,t)$ maps to $\phi\bigl(s\oplus\mathsf{salt}_{\mathsf{fresh}}\oplus\phi(t\,g_3)\bigr)$. The uniform of token $v$ under a seed $\sigma$ is $\rho\bigl(\phi(\sigma+(v+1)\,g_1)\bigr)$ with $\rho$ as in \eqref{eq:rho}, so the whole vocabulary is scored with one vectorised pass and the decoder recomputes any single entry from the text and the key. All arithmetic is modulo $2^{64}$. This is a fast non-cryptographic hash, chosen for speed; a keyed cryptographic PRF, such as a block cipher in counter mode or HMAC, can be substituted without any change to the method or the analysis, at a higher cost per token.

\section{Proof of Lemma~\ref{lem:gumbel}}
\label{app:gumbel}

\certmarkGumbel*

\begin{proof}
Fix the step $t$ and its probability distribution $p_t$ on the finite
vocabulary $V$. Thus $p_t(v)\ge0$ for every $v$ and
$\sum_{v\in V}p_t(v)=1$. Throughout the proof, $p_t$ is fixed and all
probabilities are taken over the independent draws $u(v)\sim U(0,1)$.

\smallskip\noindent
\textbf{Restricting to tokens with positive probability.}
Let $V_+=\{v\in V:p_t(v)>0\}$. This set is nonempty because the
probabilities sum to one. With the convention $\log 0=-\infty$, a token
$v\notin V_+$ has $S(v)=-\infty$. In contrast, $S(v)$ is finite for every
$v\in V_+$, since $0<u(v)<1$ implies $0<-\log u(v)<\infty$.
Consequently, the maximiser must belong to $V_+$, and
$\Pr[y=v]=0=p_t(v)$ for every $v\notin V_+$. All divisions by $p_t(v)$
below are restricted to $v\in V_+$.

\smallskip\noindent
\textbf{Transforming the uniforms.}
For $v\in V_+$, define $E(v)=-\log u(v)$. For every $s\ge0$,
\begin{equation}
\begin{aligned}
\Pr[E(v)>s]
&=\Pr[-\log u(v)>s]\\
&=\Pr[u(v)<e^{-s}]
=e^{-s}.
\end{aligned}
\end{equation}
The last equality uses $e^{-s}\in(0,1]$ and the distribution function of
$U(0,1)$. Thus $E(v)$ is exponential with rate $1$. The variables $E(v)$
remain independent, since each is a deterministic function of its own
independent uniform draw.

This transformation also verifies the Gumbel terminology in the lemma.
For $G(v)=-\log E(v)$ and every $z\in\mathbb{R}$,
\begin{equation}
\Pr[G(v)\le z]
=\Pr[E(v)\ge e^{-z}]
=\exp(-e^{-z}),
\end{equation}
where replacing $>$ by $\ge$ does not change the probability because
$E(v)$ has a continuous distribution. Hence the $G(v)$ are independent
standard Gumbel variables and $S(v)=\log p_t(v)+G(v)$.

\smallskip\noindent
\textbf{Expressing the maximiser through exponential variables.}
Define
\begin{equation}
T(v)=\frac{E(v)}{p_t(v)},\qquad v\in V_+.
\end{equation}
These variables are independent because the denominators are fixed.
For $s\ge0$, their survival functions and densities are
\begin{equation}
\begin{aligned}
\Pr[T(v)>s]
&=\Pr[E(v)>p_t(v)s]=e^{-p_t(v)s},\\
f_{T(v)}(s)
&=\frac{\mathrm{d}}{\mathrm{d}s}\bigl(1-e^{-p_t(v)s}\bigr)
=p_t(v)e^{-p_t(v)s},\qquad s>0.
\end{aligned}
\end{equation}
In particular, $T(v)$ is exponential with rate $p_t(v)$. On $V_+$ the
score can be rewritten as
\begin{equation}
S(v)=\log p_t(v)-\log E(v)
=-\log\!\left(\frac{E(v)}{p_t(v)}\right)
=-\log T(v).
\end{equation}
Since $r\mapsto-\log r$ is strictly decreasing on $(0,\infty)$, the
token with the largest score is exactly the token with the smallest
value of $T(v)$:
\begin{equation}
y=\arg\max_{v\in V_+}S(v)
=\arg\min_{v\in V_+}T(v).
\end{equation}
This minimiser is unique with probability one. Indeed, for two distinct
tokens $v,w\in V_+$, independence implies that conditioning on $T(w)$
leaves $T(v)$ with its continuous distribution. The conditional
probability that $T(v)$ equals that particular value of $T(w)$ is
therefore zero. Averaging gives $\Pr[T(v)=T(w)]=0$, and taking the union
over the finitely many pairs shows that the probability of any tie is
zero.

\smallskip\noindent
\textbf{Computing the probability of selecting each token.}
Fix $v\in V_+$. Conditional on $T(v)=s$, the token $v$ is selected
exactly when $T(w)>s$ for every $w\in V_+\setminus\{v\}$, apart from
the probability-zero ties just considered. Independence gives the
conditional probability of this event as
$\prod_{w\in V_+\setminus\{v\}}\Pr[T(w)>s]$. Integrating over the
density of $T(v)$ therefore yields
\begin{equation}
\begin{aligned}
\Pr[y=v]
&=\int_0^\infty f_{T(v)}(s)
  \prod_{w\in V_+\setminus\{v\}}\Pr[T(w)>s]\,\mathrm{d}s\\
&=\int_0^\infty p_t(v)e^{-p_t(v)s}
  \prod_{w\in V_+\setminus\{v\}}e^{-p_t(w)s}\,\mathrm{d}s\\
&=p_t(v)\int_0^\infty
  \exp\!\left(-s\sum_{w\in V_+}p_t(w)\right)\,\mathrm{d}s\\
&=p_t(v)\int_0^\infty e^{-s}\,\mathrm{d}s
=p_t(v).
\end{aligned}
\end{equation}
Here the third line combines the exponential factors, the fourth uses
$\sum_{w\in V_+}p_t(w)=1$, and the final equality uses
$\int_0^\infty e^{-s}\,\mathrm{d}s=[-e^{-s}]_0^\infty=1$.

If $V_+$ contains only $v$, the product over competing tokens is an
empty product, equal to $1$, and $p_t(v)=1$, so the same calculation
covers deterministic sampling. Together with the zero-probability case
above, this proves $\Pr[y=v]=p_t(v)$ for every $v\in V$.

\smallskip\noindent
\textbf{Equivalence to the implemented score.}
For each $v\in V_+$, the implementation uses
\begin{equation}
\frac{\log u(v)}{p_t(v)}
=-\frac{E(v)}{p_t(v)}
=-T(v).
\end{equation}
Maximising this quantity is again equivalent to minimising $T(v)$, and
therefore selects the same token as maximising $S(v)$. Assigning score
$-\infty$ to tokens outside $V_+$ gives the same exclusion as in the
original rule. This equivalence is an algebraic identity for the
draws in $(0,1)$ when the scores are evaluated in exact arithmetic.
\end{proof}

\section{Proof of Theorem~\ref{thm:exact}}
\label{app:exact}

\certmarkExact*

\begin{proof}
Fix the prompt $x$, message $m$, generation length $n$, model and sampling
settings. We consider one generation, with the seen-context set initially
empty and the encoder's context initialization fixed. We work under the
continuous random-function idealisation stated after \eqref{eq:rho}:
values at distinct token-sampling inputs are independent $U(0,1)$ variables.
All probabilities below are over this idealised keyed randomness, denoted
by $\Pr_\kappa$ as in the theorem. We first establish the conditional law
of a single token and then derive the law of the entire sequence.

\smallskip\noindent
\textbf{Step 1: condition on the information available before sampling.}
Fix a step $t\in\{1,\dots,n\}$. Let $\mathcal{H}_t$ denote the information
generated by the emitted prefix $Y_{<t}$, all earlier PRF queries and their
answers, and the current chunk-selection query \eqref{eq:chunkidx} together
with its answer, before any sampling uniforms for step $t$ are read. Including this last query is
necessary because $i_t$ depends on the PRF and need not be determined by
the prefix alone. Conditional on $\mathcal{H}_t$, the following are fixed:
\begin{equation}
p_t(\cdot)=p(\cdot\mid x,Y_{<t}),\qquad
c_t,\qquad \mathcal{C}_{t-1},\qquad i_t,\qquad a_t.
\end{equation}
In particular, $a_t$ is chosen using only information already revealed.
At $t=1$, the earlier query transcript and emitted prefix are empty;
the same conditioning includes the first chunk-selection query.

\smallskip\noindent
\textbf{Step 2: verify that every sampling input is new.}
For each $v\in V$, write $b_{t,v}=\langle a_t,v\rangle$ for the encoded
input used to obtain $u_t(v)$. We check both branches of \eqref{eq:encoder}.

\emph{First occurrence of the context.}
If $c_t\notin\mathcal{C}_{t-1}$, then
$a_t=(\mathsf{msg},c_t,m_{i_t})$. Every earlier message-carrying input has
the form $\langle\mathsf{msg},c_s,m_{i_s},w\rangle$ for some $s<t$.
Because $c_t$ has not appeared before, $c_s\neq c_t$, so none of these
inputs equals $b_{t,v}$.

\emph{Repeated context.}
If $c_t\in\mathcal{C}_{t-1}$, then
$a_t=(\mathsf{fresh},c_t,t)$. Every earlier input with the tag
$\mathsf{fresh}$ contains an earlier position $s<t$. Its position field
therefore differs from $t$, even when its context equals $c_t$.

In either branch, the tags $\mathsf{msg}$ and $\mathsf{fresh}$ distinguish
the two families of sampling inputs from each other, and both differ from
the tag $\mathsf{chunk}$ used for chunk selection. Thus $b_{t,v}$ also
differs from every chunk-selection input, including the current one.
Finally, for distinct tokens $v\neq w$, the inputs $b_{t,v}$ and $b_{t,w}$
differ in their token field. The injective encoding in \eqref{eq:argenc}
preserves all these distinctions. Hence the inputs used at step $t$ are
pairwise distinct and have never been queried earlier in this generation.

\smallskip\noindent
\textbf{Step 3: obtain independent uniforms conditional on the history.}
A random function can be revealed one query at a time: at a previously
unqueried input, draw an independent value from its prescribed distribution
and store it; at a repeated input, return the stored value. This produces
the same law as sampling the whole random function in advance. It also
shows why choosing a new input from earlier query answers does not change
the law of its unrevealed value.

By Step 2, every input $b_{t,v}$ is new. Under the continuous idealisation,
for any numbers $r_v\in[0,1]$, $v\in V$,
\begin{equation}
\Pr_\kappa\!\left[
u_t(v)\le r_v\ \text{for all }v\in V
\,\middle|\,\mathcal{H}_t\right]
=\prod_{v\in V}r_v.
\end{equation}
This is the joint distribution function of independent $U(0,1)$ variables.
Consequently, conditional on $\mathcal{H}_t$, the entire vector
$u_t(\cdot)$ has exactly the distribution required by Lemma~\ref{lem:gumbel}.

\smallskip\noindent
\textbf{Step 4: identify the next-token distribution.}
Given $\mathcal{H}_t$, the vector $p_t$ is a fixed probability distribution
by Step 1, and the sampling uniforms are independent $U(0,1)$ variables
by Step 3. Applying Lemma~\ref{lem:gumbel} to the encoder's rule
\eqref{eq:emit} therefore gives, for every $v\in V$,
\begin{equation}
\Pr_\kappa[Y_t=v\mid\mathcal{H}_t]
=p_t(v)=p(v\mid x,Y_{<t}).
\end{equation}
Tokens with $p_t(v)=0$ have score $-\infty$ and are never emitted;
among tokens with positive probability, ties occur with probability zero.
Thus the equality includes zero-probability tokens as well.

The prefix $Y_{<t}$ is part of $\mathcal{H}_t$. Averaging over the additional
query information by the law of iterated conditional expectation gives
\begin{equation}
\begin{aligned}
\Pr_\kappa[Y_t=v\mid Y_{<t}]
&=\mathbb{E}_\kappa\!\left[
\Pr_\kappa[Y_t=v\mid\mathcal{H}_t]
\,\middle|\,Y_{<t}\right]\\
&=\mathbb{E}_\kappa\!\left[p(v\mid x,Y_{<t})\,\middle|\,Y_{<t}\right]\\
&=p(v\mid x,Y_{<t}),
\end{aligned}
\end{equation}
where the last equality holds because the prompt and prefix already
determine the deployed sampling distribution. This is precisely the
conditional next-token law of ordinary, unwatermarked generation.

\smallskip\noindent
\textbf{Step 5: derive the sequence distribution.}
Fix any $w\in V^n$, and write
$q_t=\Pr_\kappa[Y_{1:t}=w_{1:t}]$, with $q_0=1$ for the empty prefix.
If $q_{t-1}>0$, the multiplication rule and Step 4 give
\begin{equation}
\begin{aligned}
q_t
&=q_{t-1}\,
\Pr_\kappa[Y_t=w_t\mid Y_{<t}=w_{<t}]\\
&=q_{t-1}\,p(w_t\mid x,w_{<t}).
\end{aligned}
\end{equation}
If $q_{t-1}=0$, then $q_t=0$ because the event
$\{Y_{1:t}=w_{1:t}\}$ is contained in $\{Y_{<t}=w_{<t}\}$, so the same
recurrence still holds. Starting from $q_0=1$ and applying this recurrence
successively for $t=1,\dots,n$ yields
\begin{equation}
\Pr_\kappa[Y_{1:n}=w]
=q_n=\prod_{t=1}^{n}p(w_t\mid x,w_{<t}),
\end{equation}
which proves \eqref{eq:exact}, including sequences with zero probability.
The right-hand side is the sequence law of the unwatermarked sampler.
It does not depend on $m$, since the message changes the sampling inputs
but leaves $p(\cdot\mid x,w_{<t})$ unchanged. As $x$ and $m$ were arbitrary,
the result holds for every prompt and message.
\end{proof}

\section{Proof of Theorem~\ref{thm:null}}
\label{app:null}

\certmarkNull*

\begin{proof}
Fix the prompt, model, sampling settings and embedded message $m$. Also
fix a chunk $c$ and a candidate $m'\neq m_c$. We use the continuous
random-function idealisation stated after \eqref{eq:rho}, and take all
probabilities over the idealised keyed randomness. We will first show
that the values reconstructed for $m'$ are independent uniforms even
after conditioning on the information used to generate the text. We
then derive the distribution of their score contributions and their sum.

\smallskip\noindent
\textbf{Fixing the text and the scored positions.}
Let $\mathcal{H}$ denote the information in the complete encoder query
transcript: all PRF inputs queried during the generation and all their
answers. Include also the chunk-selection queries used by the decoder,
before it evaluates any candidate scores. The generated text is
determined by the encoder transcript and the fixed prompt and model.
Consequently, conditional on $\mathcal{H}$, the text $y$, its contexts,
the visited positions $\mathcal{T}$, their chunk assignments $i_t$, and
the set $\mathcal{T}_c$ are all fixed. This extra conditioning matters because
$\mathcal{T}_c$ depends on the chunk-selection PRF values as well as on the text.

If $\mathcal{T}_c=\varnothing$, the sum in \eqref{eq:score} is empty and therefore
$S_c(m')=0$. This proves the assertion for $n_c=0$. In what follows,
assume $n_c\ge1$ and, for every $t\in \mathcal{T}_c$, write
\begin{equation}
b_t=\langle\mathsf{msg},c_t,m',y_t\rangle,
\qquad U_t=\rho\bigl(\PRF_\kappa(b_t)\bigr)
=u_t^{(m')}(y_t).
\end{equation}
The inputs $b_t$ are determined by $\mathcal{H}$; their PRF values will
be the randomness used to score the wrong candidate.

\smallskip\noindent
\textbf{Showing that the candidate inputs are distinct and unqueried.}
The decoder retains only the first occurrence of each context in
\eqref{eq:visited}. Thus $c_t\neq c_s$ for distinct $t,s\in \mathcal{T}_c$, and
the injective encoding in \eqref{eq:argenc} implies $b_t\neq b_s$.

Now fix $t\in \mathcal{T}_c$ and compare $b_t$ with every type of query in
$\mathcal{H}$. A chunk-selection query has tag $\mathsf{chunk}$, and a
fresh-sampling query has tag $\mathsf{fresh}$; each differs from the
tag $\mathsf{msg}$ in $b_t$. A message-carrying encoder query at any
position $s$ has the form
\begin{equation}
\langle\mathsf{msg},c_s,m_{i_s},v\rangle,
\qquad v\in V.
\end{equation}
If $c_s\neq c_t$, its context field differs from that of $b_t$. If
$c_s=c_t$, the context-based assignment \eqref{eq:chunkidx} gives
$i_s=i_t=c$. The encoder therefore uses $m_{i_s}=m_c$, which differs
from $m'$, so the message field differs. This comparison covers every
position in the generation, including positions after $t$, and every
token queried by the encoder. Hence none of the $b_t$ was queried in
forming $\mathcal{H}$. The conclusion also holds if $m'$ happens to
equal the value carried by another chunk, since an identical context
always selects the same chunk.

\smallskip\noindent
\textbf{Obtaining independent uniforms under the conditioning.}
As in Appendix~\ref{app:exact}, a random function can be revealed by
drawing and storing an independent value when an input is first
queried. Conditional on a complete query transcript, values at all
unqueried inputs retain their original independent distributions.
Here the $b_t$ are distinct unqueried inputs chosen from the transcript,
so their values can still be drawn independently after that transcript
has been fixed. In particular, for any $r_t\in[0,1]$, $t\in \mathcal{T}_c$,
\begin{equation}
\Pr_\kappa\!\left[
U_t\le r_t\ \text{for all }t\in \mathcal{T}_c
\,\middle|\,\mathcal{H}\right]
=\prod_{t\in \mathcal{T}_c}r_t.
\end{equation}
Thus the $U_t$ are conditionally independent $U(0,1)$ variables.
Conditioning on the generated text does not bias them: the entire
generation used other PRF inputs, and its outcome is already part of
$\mathcal{H}$.

\smallskip\noindent
\textbf{Deriving the distribution of one score contribution.}
Set $X_t=-\log(1-U_t)$. Since $0<U_t<1$, each $X_t$ is positive.
For $z\ge0$, the conditional survival function is
\begin{equation}
\begin{aligned}
\Pr_\kappa[X_t>z\mid\mathcal{H}]
&=\Pr_\kappa[-\log(1-U_t)>z\mid\mathcal{H}]\\
&=\Pr_\kappa[1-U_t<e^{-z}\mid\mathcal{H}]\\
&=\Pr_\kappa[U_t>1-e^{-z}\mid\mathcal{H}]\\
&=1-(1-e^{-z})=e^{-z}.
\end{aligned}
\end{equation}
The last line uses $1-e^{-z}\in[0,1)$ and the uniform law of $U_t$.
For $z<0$, the survival probability is $1$. Therefore each $X_t$ has
the exponential distribution with rate $1$, with density $e^{-z}$
for $z>0$ and zero density for $z<0$. The $X_t$ remain conditionally
independent because each is a function of a different $U_t$.

\smallskip\noindent
\textbf{Deriving the distribution of the sum.}
Enumerate $\mathcal{T}_c=\{t_1,\ldots,t_{n_c}\}$, and let
$Z_r=\sum_{j=1}^{r}X_{t_j}$. Conditional on $\mathcal{H}$, we claim that
for each integer $1\le r\le n_c$, its density is
\begin{equation}
f_r(z)=\frac{z^{r-1}e^{-z}}{(r-1)!},\qquad z>0,
\end{equation}
and is zero for $z<0$. For $r=1$, this is exactly the exponential
density derived above. Suppose it holds for some $r<n_c$. Since
$Z_r$ and $X_{t_{r+1}}$ are independent and nonnegative, the density
of their sum is the convolution of their densities. For $z>0$,
\begin{equation}
\begin{aligned}
f_{r+1}(z)
&=\int_0^z f_r(s)e^{-(z-s)}\,\mathrm{d}s\\
&=\int_0^z\frac{s^{r-1}e^{-s}}{(r-1)!}
   e^{-(z-s)}\,\mathrm{d}s\\
&=\frac{e^{-z}}{(r-1)!}\int_0^z s^{r-1}\,\mathrm{d}s\\
&=\frac{e^{-z}}{(r-1)!}\frac{z^r}{r}
=\frac{z^r e^{-z}}{r!}.
\end{aligned}
\end{equation}
This proves the claim by induction. Taking $r=n_c$ and recalling
$S_c(m')=Z_{n_c}$ identifies its conditional density as that of
$\Gamma(n_c,1)$, with shape $n_c$ and rate $1$.

\smallskip\noindent
\textbf{Returning to conditioning only on the text and scored positions.}
Let $F_r$ be the distribution function of $\Gamma(r,1)$ for $r\ge1$,
and set $F_0(z)=\mathbf{1}\{z\ge0\}$ for the empty score. The preceding
argument gives
$\Pr_\kappa[S_c(m')\le z\mid\mathcal{H}]=F_{n_c}(z)$.
Since $\mathcal{H}$ determines both $Y_{1:n}$ and $\mathcal{T}_c$, iterated
conditional expectation yields, for every $z\in\mathbb{R}$,
\begin{equation}
\begin{aligned}
\Pr_\kappa[S_c(m')\le z\mid Y_{1:n},\mathcal{T}_c]
&=\mathbb{E}_\kappa\!\left[
\Pr_\kappa[S_c(m')\le z\mid\mathcal{H}]
\,\middle|\,Y_{1:n},\mathcal{T}_c\right]\\
&=\mathbb{E}_\kappa\!\left[F_{n_c}(z)\,\middle|\,Y_{1:n},\mathcal{T}_c\right]\\
&=F_{n_c}(z).
\end{aligned}
\end{equation}
The last equality holds because $n_c=|\mathcal{T}_c|$ is fixed by the conditioning.
This proves \eqref{eq:null} for every possible emitted text and scored
set, with the empty case handled separately. The resulting law depends
on them only through $n_c$; no model probabilities, prompt, sampling
temperature or message values enter its parameters.

The same calculation also applies to a text fixed independently of the
PRF. In that case, condition on that text and its chunk-selection
queries; the candidate inputs are again distinct and unqueried, so the
uniform, exponential and Gamma calculations above apply unchanged.
\end{proof}

\section{Proof of Theorem~\ref{thm:gain}}
\label{app:gain}

\noindent The main text states the candidate distributions. The full statement
below additionally gives their expected scores; it keeps the theorem's number
and the equation numbers reserved in the main text.
\begingroup
\setcounter{savedtheorem}{\value{theorem}}%
\setcounterref{theorem}{thm:gain}\addtocounter{theorem}{-1}%
\def\theHtheorem{full.\thetheorem}\def\theHequation{full.\arabic{equation}.\the\inputlineno}%
\begin{theorem}[Candidate distributions and expected scores, full statement]
\label{thm:gain_full}
Treat $\PRF$ as a random function. Fix the deployed model, prompt $x$,
sampling settings, embedded message $m$, and a chunk $c$ with $k_c\ge1$.
For $t\in \mathcal{T}_c$, condition on $y_{<t}$ and
$p=p_t(y_t)\in(0,1]$, the sampler's probability of the emitted token.
For every candidate $m'$, its reconstructed value and expected per-token
score satisfy
\begin{align}
u_t^{(m')}(y_t)&\sim
\begin{cases}
\mathrm{Beta}(1/p,1), & m'=m_c \quad\text{(correct candidate)},\\[2pt]
U(0,1), & m'\neq m_c \quad\text{(wrong candidate)},
\end{cases}
\tag{\ref{eq:candidate_laws}}\\[4pt]
\mathbb{E}\!\left[-\log\bigl(1-u_t^{(m')}(y_t)\bigr)\,\middle|\,y_{<t},p\right]
&=\begin{cases}
g(p):=\psi(1/p+1)+\gamma_E, & m'=m_c,\\[2pt]
1, & m'\neq m_c,
\end{cases}
\tag{\gainEquationNumber}\label{eq:gain}
\end{align}
where $u_t^{(m_c)}(y_t)=u_t(y_t)$, $\psi$ is the digamma function and
$\gamma_E$ is the Euler--Mascheroni constant; $g(1)=1$ and $g$ is strictly
decreasing on $(0,1]$.
\end{theorem}
\setcounter{theorem}{\value{savedtheorem}}%
\endgroup

\begin{proof}
Fix the prompt, model, sampling settings and embedded message. Fix a
chunk $c$ and a position $t$, and throughout the proof restrict to the
event $t\in \mathcal{T}_c$, as in the theorem. This event says that the context
first occurs at $t$ and that its chunk assignment is $c$; both facts
are determined before the token $Y_t$ is sampled. All probabilities
are over the continuous random-function idealisation stated after
\eqref{eq:rho}.

\smallskip\noindent
\textbf{Conditioning before the token is drawn.}
Let $\mathcal{H}_t$ be the information available just before reading
the sampling uniforms at step $t$, including the prefix, all earlier
PRF queries and their answers, and the current chunk-selection query
and its answer, as in Appendix~\ref{app:exact}. Conditional on
$\mathcal{H}_t$, the distribution $p_t$, context $c_t$ and assignment
$i_t=c$ are fixed. Since this is the first occurrence of the context,
the encoder uses the tuple $(\mathsf{msg},c_t,m_c)$, which is exactly
the tuple reconstructed by the correct candidate. Hence
\begin{equation}
u_t^{(m_c)}(v)=u_t(v),\qquad v\in V.
\end{equation}
Appendix~\ref{app:exact} establishes that this vector has independent
$U(0,1)$ entries conditional on $\mathcal{H}_t$. This is the point at
which the fresh-input rule is used. Let
$V_+=\{v\in V:p_t(v)>0\}$. Tokens outside $V_+$ cannot be emitted,
as shown in Appendix~\ref{app:gumbel}, so $p=p_t(Y_t)$ belongs to
$(0,1]$ with probability one.

\smallskip\noindent
\textbf{The joint law of the emitted token and its uniform.}
Appendix~\ref{app:gumbel} shows that maximising the Gumbel score is
equivalent to maximising $\log u_t(v)/p_t(v)$. Exponentiation is
strictly increasing, so
\begin{equation}
Y_t=\arg\max_{v\in V_+}\frac{\log u_t(v)}{p_t(v)}
=\arg\max_{v\in V_+}u_t(v)^{1/p_t(v)}.
\end{equation}
Fix $v\in V_+$ and write $q=p_t(v)>0$. If $u_t(v)=s\in(0,1)$, then
this token is selected precisely when, for every competing token
$w\in V_+\setminus\{v\}$,
\begin{equation}
u_t(w)^{1/p_t(w)}\le s^{1/q}
\quad\Longleftrightarrow\quad
u_t(w)\le s^{p_t(w)/q}.
\end{equation}
The probability-zero ties can be ignored by Appendix~\ref{app:gumbel}.
Given $\mathcal{H}_t$ and $u_t(v)=s$, the other coordinates remain
independent uniforms. Their joint probability of satisfying these
inequalities is therefore
\begin{equation}
\begin{aligned}
\Pr_\kappa[Y_t=v\mid\mathcal{H}_t,u_t(v)=s]
&=\prod_{w\in V_+\setminus\{v\}}s^{p_t(w)/q}\\
&=s^{\left(\sum_{w\in V_+\setminus\{v\}}p_t(w)\right)/q}
=s^{(1-q)/q}.
\end{aligned}
\end{equation}
The final equality uses $\sum_{w\in V_+}p_t(w)=1$. When $q=1$,
there are no positive-probability competitors and the empty product
is $1$, in agreement with the same formula.

The conditional density of $u_t(v)$ is $1$ on $(0,1)$. Integrating
over its value thus gives, for $0<x\le1$,
\begin{equation}
\begin{aligned}
\Pr_\kappa[u_t(v)\le x,\ Y_t=v\mid\mathcal{H}_t]
&=\int_0^x
  \Pr_\kappa[Y_t=v\mid\mathcal{H}_t,u_t(v)=s]\,\mathrm{d}s\\
&=\int_0^x s^{1/q-1}\,\mathrm{d}s\\
&=\bigl[q\,s^{1/q}\bigr]_{s=0}^{s=x}
=q\,x^{1/q}.
\end{aligned}
\end{equation}
Taking $x=1$ recovers $\Pr_\kappa[Y_t=v\mid\mathcal{H}_t]=q$.
Since $q>0$, division by this probability is valid and yields
\begin{equation}
\begin{aligned}
\Pr_\kappa[u_t(Y_t)\le x\mid\mathcal{H}_t,Y_t=v]
&=\frac{\Pr_\kappa[u_t(v)\le x,\ Y_t=v\mid\mathcal{H}_t]}
        {\Pr_\kappa[Y_t=v\mid\mathcal{H}_t]}\\
&=\frac{q\,x^{1/q}}q=x^{1/q},\qquad 0<x<1.
\end{aligned}
\end{equation}

\smallskip\noindent
\textbf{Conditioning on the emitted token's probability.}
The preceding conditional law depends on the token $v$ only through
its probability $q=p_t(v)$. Write $p=p_t(Y_t)$ for the probability
assigned to the token that was actually emitted. The pair
$(\mathcal{H}_t,Y_t)$ determines both $Y_{<t}$ and $p$, so iterated
conditional expectation gives, for $0<x<1$,
\begin{equation}
\begin{aligned}
\Pr_\kappa[u_t(Y_t)\le x\mid Y_{<t},p]
&=\mathbb{E}_\kappa\!\left[
\Pr_\kappa[u_t(Y_t)\le x\mid\mathcal{H}_t,Y_t]
\,\middle|\,Y_{<t},p\right]\\
&=\mathbb{E}_\kappa\!\left[x^{1/p_t(Y_t)}\,\middle|\,Y_{<t},p\right]
=x^{1/p}.
\end{aligned}
\end{equation}
Thus averaging over the additional query information, or over several
tokens with the same probability $p$, leaves this distribution
unchanged. Its distribution function is $0$ for $x\le0$ and $1$ for
$x\ge1$, and its density on $(0,1)$ is
\begin{equation}
\frac{\mathrm{d}}{\mathrm{d}x}x^{1/p}
=\frac1p x^{1/p-1}.
\end{equation}
This is the density of $\mathrm{Beta}(1/p,1)$, proving the
correct-candidate case of \eqref{eq:candidate_laws}.

\smallskip\noindent
\textbf{The wrong-candidate distribution.}
Fix $m'\neq m_c$. Appendix~\ref{app:null} proves that the input
$\langle\mathsf{msg},c_t,m',Y_t\rangle$ is never queried by the
encoder: an encoder query with that context uses $m_c$, and the
other query families have different tags. Moreover, the value at
this input remains uniform conditional on the complete generation
transcript and the chunk-selection queries. Denote that information
by $\mathcal{H}$, as in Appendix~\ref{app:null}. It determines the
emitted text and hence also $Y_{<t}$ and $p$. Consequently,
\begin{equation}
\begin{aligned}
\Pr_\kappa[u_t^{(m')}(Y_t)\le x\mid Y_{<t},p]
&=\mathbb{E}_\kappa\!\left[
\Pr_\kappa[u_t^{(m')}(Y_t)\le x\mid\mathcal{H}]
\,\middle|\,Y_{<t},p\right]\\
&=\mathbb{E}_\kappa[x\mid Y_{<t},p]=x,
\qquad 0<x<1.
\end{aligned}
\end{equation}
Together with the endpoint values $0$ and $1$, this is the
distribution function of $U(0,1)$, proving the second case of
\eqref{eq:candidate_laws}.

\smallskip\noindent
\textbf{Computing the expected score from the Beta density.}
Let $B\sim\mathrm{Beta}(a,1)$ with $a>0$, so that its density is
$a\,x^{a-1}$ on $(0,1)$. For every integer $j\ge1$, its $j$th moment is
\begin{equation}
\begin{aligned}
\mathbb{E}[B^j]
&=\int_0^1 x^j\,a\,x^{a-1}\,\mathrm{d}x\\
&=a\left[\frac{x^{a+j}}{a+j}\right]_{x=0}^{x=1}
=\frac{a}{a+j}.
\end{aligned}
\end{equation}
Integrating the geometric series for $1/(1-x)$ on $[0,b]$ gives
$-\log(1-b)=\sum_{j=1}^{\infty}b^j/j$ for $0\le b<1$.

\Needspace{10\baselineskip}
In particular, the nonnegative partial sums
$L_N=\sum_{j=1}^N B^j/j$ increase to $-\log(1-B)$ with probability one.
The monotone convergence theorem therefore justifies interchanging
expectation and this infinite sum:
\begin{equation}
\begin{aligned}
\mathbb{E}[-\log(1-B)]
&=\lim_{N\to\infty}\mathbb{E}[L_N]\\
&=\lim_{N\to\infty}\sum_{j=1}^N\frac{\mathbb{E}[B^j]}j\\
&=\sum_{j=1}^{\infty}\frac{a}{j(a+j)}\\
&=\sum_{j=1}^{\infty}\left(\frac1j-\frac1{j+a}\right).
\end{aligned}
\end{equation}
This expectation is finite, since
$0\le a/[j(a+j)]\le a/j^2$ and $\sum_{j=1}^{\infty}j^{-2}$ converges.
For the digamma function $\psi(z)=\Gamma'(z)/\Gamma(z)$, its series
representation gives
\begin{equation}
\psi(a+1)=-\gamma_E+
\sum_{j=1}^{\infty}\left(\frac1j-\frac1{j+a}\right),
\qquad a>0.
\end{equation}
Combining these two identities yields
$\mathbb{E}[-\log(1-B)]=\psi(a+1)+\gamma_E$.
Applying this identity to the correct candidate's conditional law,
with $a=1/p$, gives
\begin{equation}
\mathbb{E}\!\left[-\log\bigl(1-u_t(Y_t)\bigr)
\,\middle|\,Y_{<t},p\right]
=\psi(1/p+1)+\gamma_E=g(p).
\end{equation}

For a wrong candidate, $U(0,1)=\mathrm{Beta}(1,1)$, so $a=1$.
The corresponding series telescopes:
\begin{equation}
\sum_{j=1}^{\infty}\left(\frac1j-\frac1{j+1}\right)
=\lim_{N\to\infty}\left(1-\frac1{N+1}\right)=1.
\end{equation}
This proves the wrong-candidate mean in \eqref{eq:gain} and also
shows that $g(1)=1$. At $p=1$, the correct-candidate Beta law is
itself $U(0,1)$, so the two distributions and their expected scores
coincide.

\smallskip\noindent
\textbf{Strict decrease of the expected score.}
Substituting $a=1/p$ into the convergent series above gives the
equivalent expression
\begin{equation}
g(p)=\sum_{j=1}^{\infty}\frac{1}{j(1+jp)},\qquad p>0.
\end{equation}
For $0<p_1<p_2\le1$, subtracting the two convergent series yields
\begin{equation}
\begin{aligned}
g(p_1)-g(p_2)
&=\sum_{j=1}^{\infty}
\left(\frac{1}{j(1+jp_1)}-\frac{1}{j(1+jp_2)}\right)\\
&=(p_2-p_1)\sum_{j=1}^{\infty}
\frac{1}{(1+jp_1)(1+jp_2)}>0.
\end{aligned}
\end{equation}
The last inequality holds because $p_2-p_1>0$ and every summand is
positive. Hence $g$ is strictly decreasing on $(0,1]$. Together with
$g(1)=1$, this also gives $g(p)>1$ whenever $p<1$, completing the proof.
\end{proof}

\section{Proof of Theorem~\ref{thm:cert}}
\label{app:cert}

\certmarkCert*

\begin{proof}
Fix the model, prompt, sampling settings and embedded message $m$.
Also fix the chunk $c$ and the level $\delta\in(0,1]$, before observing
the scores. All probabilities are over the idealised keyed randomness.
Write
\begin{equation}
\mathcal{M}_c=\{0,1\}^{k_c},\qquad
K=|\mathcal{M}_c|=2^{k_c},\qquad
\mathcal{W}_c=\mathcal{M}_c\setminus\{m_c\}.
\end{equation}
Thus $\mathcal{W}_c$ contains exactly $K-1$ wrong candidates. Choose any
fixed rule for breaking ties in the maximisation defining $\hat m_c$.
The event whose probability we must bound is
\begin{equation}
E_c=\bigl\{D_\delta(Y_{1:n})\notin\{m_c,\bot\}\bigr\},
\end{equation}
namely, the event that the decoder accepts an incorrect chunk.

\smallskip\noindent
\textbf{Boundary cases and conditioning.}
If $\delta=1$, then $\Pr_\kappa[E_c]\le1=\delta$ simply because $E_c$
is an event. If $K=1$, there are no wrong candidates, so $E_c$ is empty.
We may therefore assume $0<\delta<1$ and $K\ge2$.

Condition on any pair $Y_{1:n}=y$ and $\mathcal{T}_c=\tau$ with positive
probability, and put $r=|\tau|$. To keep the notation readable, write
\begin{equation}
\Pr_{y,\tau}[\,\cdot\,]
:=\Pr_\kappa[\,\cdot\mid Y_{1:n}=y,\ \mathcal{T}_c=\tau].
\end{equation}
Under this conditioning, the number of scored positions is fixed at
$n_c=r$. Conditioning on $\mathcal{T}_c$ as well as on the text is necessary
because the scored positions also depend on the chunk-selection PRF
values. The candidate scores themselves remain random under this
conditioning.

If $r=0$, the definition gives $\hat\delta_c(y)=1>\delta$, and the
decoder abstains. Hence $\Pr_{y,\tau}[E_c]=0$. In the remainder of the
conditional argument, assume $r\ge1$.

\smallskip\noindent
\textbf{The conditional null tail and its threshold.}
Let $G_r$ have the Gamma distribution with shape $r$ and rate $1$,
and denote its survival function by
\begin{equation}
Q_r(s):=\Pr[G_r\ge s]
=\int_s^\infty\frac{z^{r-1}e^{-z}}{(r-1)!}\,\mathrm{d}z,
\qquad s\ge0.
\end{equation}
Theorem~\ref{thm:null} states precisely that every fixed wrong
candidate has this conditional score distribution. Consequently,
\begin{equation}
\Pr_{y,\tau}[S_c(m')\ge s]=Q_r(s)
\qquad\text{for every }m'\in\mathcal{W}_c\text{ and }s\ge0.
\end{equation}
Here $Q_r$ is a deterministic function once $r$ has been fixed; it
does not depend on any realised candidate score.

The integral defining $Q_r$ shows that it is continuous, with
$Q_r(0)=1$ and $Q_r(s)\to0$ as $s\to\infty$. It is also strictly
decreasing: whenever $0\le a<b$,
\begin{equation}
Q_r(a)-Q_r(b)
=\int_a^b\frac{z^{r-1}e^{-z}}{(r-1)!}\,\mathrm{d}z>0,
\end{equation}
because the integrand is positive for every $z>0$. Since
$0<\delta/(K-1)<1$, continuity and strict monotonicity imply that
there is exactly one $s_\delta>0$ satisfying
\begin{equation}
Q_r(s_\delta)=\frac{\delta}{K-1},
\qquad\text{or equivalently}\qquad
(K-1)Q_r(s_\delta)=\delta.
\end{equation}
This threshold depends on $r$, $K$ and $\delta$, all of which are
fixed in the conditional argument.

\smallskip\noindent
\textbf{Rewriting the acceptance rule as a score threshold.}
Abbreviate $\hat m=\hat m_c(y)$. Since $r\ge1$, the certificate in
\eqref{eq:certdef} becomes
\begin{equation}
\hat\delta_c(y)=\min\{1,(K-1)Q_r(S_c(\hat m))\}.
\end{equation}
For any $a\ge0$ and $\delta<1$, the inequality
$\min\{1,a\}\le\delta$ holds if and only if $a\le\delta$:
if $a\ge1$, the minimum equals $1>\delta$, and if $a<1$, the
minimum equals $a$. Applying this observation and then the strict
decrease of $Q_r$ gives
\begin{equation}
\begin{aligned}
\hat\delta_c(y)\le\delta
&\quad\Longleftrightarrow\quad
(K-1)Q_r(S_c(\hat m))\le\delta\\
&\quad\Longleftrightarrow\quad
Q_r(S_c(\hat m))\le Q_r(s_\delta)\\
&\quad\Longleftrightarrow\quad
S_c(\hat m)\ge s_\delta.
\end{aligned}
\end{equation}
All scores are nonnegative because their summands are
$-\log(1-u)$ with $u\in(0,1)$, so the survival function above is
being evaluated on its stated domain.

\smallskip\noindent
\textbf{Controlling an incorrectly selected candidate.}
Under the current conditioning, the decoder makes a certified error
exactly when its selected candidate is wrong and its score meets
the threshold. Therefore
\begin{equation}
\begin{aligned}
E_c
&=\{\hat m\in\mathcal{W}_c\}
  \cap\{S_c(\hat m)\ge s_\delta\}\\
&=\bigcup_{m'\in\mathcal{W}_c}
  \bigl(\{\hat m=m'\}\cap\{S_c(m')\ge s_\delta\}\bigr)\\
&\subseteq\bigcup_{m'\in\mathcal{W}_c}
  \{S_c(m')\ge s_\delta\}.
\end{aligned}
\end{equation}
The second line separates the error event according to which wrong
candidate is selected. The inclusion then drops the selection
condition: accepting a wrong candidate requires at least one wrong
candidate to cross the threshold. This step avoids assigning a Gamma
law to the selected maximum; Theorem~\ref{thm:null} is applied only
to each fixed wrong candidate.

By the conditional union bound, the conditional null law, and the
definition of $s_\delta$, respectively,
\begin{equation}
\begin{aligned}
\Pr_{y,\tau}[E_c]
&\le\sum_{m'\in\mathcal{W}_c}
      \Pr_{y,\tau}[S_c(m')\ge s_\delta]\\
&=\sum_{m'\in\mathcal{W}_c}Q_r(s_\delta)\\
&=(K-1)Q_r(s_\delta)\\
&=(K-1)\frac{\delta}{K-1}=\delta.
\end{aligned}
\end{equation}
The union bound requires no independence among candidate scores.
Nor does the argument require the distribution of the correct
candidate's score: that score affects which candidate is selected,
but the event inclusion already accounts for the selection. This
also shows that the bound holds for the chosen tie-breaking rule.

\smallskip\noindent
\textbf{Removing the conditioning.}
We have proved $\Pr_{y,\tau}[E_c]\le\delta$ for every pair $(y,\tau)$
with positive probability, including $\tau=\varnothing$. Averaging
over both the emitted text and the scored positions, the law of
total probability gives
\begin{equation}
\begin{aligned}
\Pr_\kappa[E_c]
&=\mathbb{E}_\kappa\!\left[
   \Pr_\kappa[E_c\mid Y_{1:n},\mathcal{T}_c]\right]\\
&\le\mathbb{E}_\kappa[\delta]=\delta.
\end{aligned}
\end{equation}
Together with the boundary cases, this proves \eqref{eq:cert} for
every $\delta\in(0,1]$. The model, prompt, sampling settings and
message were fixed throughout; the averaging is only over the
idealised key and the text and scored positions it determines.

\smallskip\noindent
\textbf{The whole-message corollary.}
For clarity, write $D_{\alpha,c}$ for the decoder applied to chunk
$c$ at level $\alpha$, and set
\begin{equation}
E_c^{\mathrm{msg}}
=\bigl\{D_{\delta/C,c}(Y_{1:n})\notin\{m_c,\bot\}\bigr\}.
\end{equation}
The chunkwise result gives
$\Pr_\kappa[E_c^{\mathrm{msg}}]\le\delta/C$ for each
$c=1,\ldots,C$. Returning any incorrect chunk is exactly the event
$\bigcup_{c=1}^C E_c^{\mathrm{msg}}$. A second union bound yields
\begin{equation}
\Pr_\kappa\!\left[\bigcup_{c=1}^C E_c^{\mathrm{msg}}\right]
\le\sum_{c=1}^C\Pr_\kappa[E_c^{\mathrm{msg}}]
\le\sum_{c=1}^C\frac{\delta}{C}=\delta.
\end{equation}
This conclusion does not require independence between chunks.
\end{proof}

\Needspace{12\baselineskip}
\section{Payload scaling}
\label{app:scaling}

The following theorem describes how decoding changes as the number of encoded message bits increases; its exponent $\Lambda_c$ is reused by Theorem~\ref{thm:nonce}.

\begin{theorem}[Payload scaling]
\label{thm:scaling}
For $p\in(0,1]$ and $\lambda\in(0,1)$,
define
\begin{equation}
\eta_p(\lambda)\;=\;\log(1-\lambda)\;-\;\log\frac{\Gamma(1/p+1)\,\Gamma(\lambda+1)}{\Gamma(1/p+\lambda+1)},
\label{eq:eta}
\end{equation}
$\eta_p(\lambda)$ is nonincreasing in $p$ and
$\eta_1(\lambda)=\log(1-\lambda^{2})<0$.  Condition on the
emitted text $Y_{1:n}=y$ and the scored-position sets
$\{\mathcal{T}_c\}_{c=1}^C$ selected by the chunk assignments. This
fixes $n_c$ and $p_t=p_t(y_t)$ for $t\in \mathcal{T}_c$. Then
\begin{equation}
\begin{aligned}
\Pr_{\kappa}\bigl[\hat m_c\neq m_c\,\big|\,Y_{1:n}=y,\{\mathcal{T}_c\}_{c=1}^C\bigr]
&\le\bigl(2^{k_c}-1\bigr)e^{-\Lambda_c}
<\exp\bigl(k_c\ln 2-\Lambda_c\bigr),\\
\Lambda_c&=\sup_{\lambda\in(0,1)}\sum_{t\in \mathcal{T}_c}\eta_{p_t}(\lambda).
\end{aligned}
\label{eq:scaling}
\end{equation}
An empty scored set has $\Lambda_c=0$. For the whole message
$\hat m=(\hat m_1,\dots,\hat m_C)$,
\begin{equation}
\Pr_{\kappa}[\hat m\neq m\mid Y_{1:n}=y,\{\mathcal{T}_c\}_{c=1}^C]
\le\sum_{c=1}^{C}(2^{k_c}-1)e^{-\Lambda_c}.
\end{equation}
\end{theorem}

\begin{proof}
Fix the model, prompt $x$, sampling settings and embedded message $m$.
We work under the continuous random-function idealisation stated after
\eqref{eq:rho}. Throughout the conditional error calculation, fix a
realisation of the emitted text $Y_{1:n}=y$ and the scored-position sets
$\{\mathcal{T}_c\}_{c=1}^C$ with positive probability. Write
$\Pr_*$ and $\mathbb{E}_*$ for probability and expectation conditional
on this pair. The sets $\mathcal{T}_c$, their sizes $n_c$, and the probabilities
$p_t=p(y_t\mid x,y_{<t})$ are then fixed. Each $p_t$ is positive,
since a token with zero sampling probability cannot appear in a text
of positive probability.

Theorem~\ref{thm:gain} gives the correct candidate's law at one step.
Here we need its joint law across scored positions after conditioning
on the \emph{entire} text and its chunk assignments. We establish that
law first, and then compare the correct score with each wrong score.

\smallskip\noindent
\textbf{The joint law of a token and the maximum used to select it.}
Let $\mathcal{A}$ denote the complete context-to-chunk assignment
table determined by the PRF inputs tagged $\mathsf{chunk}$. These
inputs are disjoint from all token-sampling inputs, so conditioning
on $\mathcal{A}$ leaves the random functions used for token sampling
independent with their original laws. For this part of the proof,
condition on $\mathcal{A}$ before generating the text.

At step $t$, let $\mathcal{H}_t$ include $\mathcal{A}$ and the full
query history before the sampling uniforms for that step are read.
As proved in Appendix~\ref{app:exact}, these sampling inputs are
distinct and have never been queried: the first occurrence of a
context uses its message tuple, while a repeated context uses a
tuple containing the new position $t$. Thus the coordinates of
$u_t(\cdot)$ are independent $U(0,1)$ variables conditional on
$\mathcal{H}_t$, and $p_t(\cdot)$ is fixed by that history.

Put $V_{t,+}=\{v:p_t(v)>0\}$ and, for $v\in V_{t,+}$, define
\begin{equation}
Z_t(v)=u_t(v)^{1/p_t(v)},\qquad
M_t=\max_{v\in V_{t,+}}Z_t(v).
\end{equation}
The equivalent sampling rule in Appendix~\ref{app:gumbel} gives
$Y_t=\arg\max_{v\in V_{t,+}}Z_t(v)$. For $0<z<1$,
\begin{equation}
\Pr_\kappa[Z_t(v)\le z\mid\mathcal{H}_t]=z^{p_t(v)},
\qquad
f_{Z_t(v)\mid\mathcal{H}_t}(z)=p_t(v)z^{p_t(v)-1}.
\end{equation}
These variables are conditionally independent, being functions of
distinct uniform coordinates, and ties have probability zero.
The event $\{M_t\le z,Y_t=v\}$ occurs when $Z_t(v)$ has some
value $s\le z$ and every competitor is at most $s$. Integrating
over $s$ gives
\begin{equation}
\begin{aligned}
\Pr_\kappa[M_t\le z,Y_t=v\mid\mathcal{H}_t]
&=\int_0^z p_t(v)s^{p_t(v)-1}
  \prod_{w\in V_{t,+}\setminus\{v\}}s^{p_t(w)}\,\mathrm{d}s\\
&=p_t(v)\int_0^z s^{\sum_{w\in V_{t,+}}p_t(w)-1}\,\mathrm{d}s\\
&=p_t(v)\int_0^z1\,\mathrm{d}s
=z\,p_t(v).
\end{aligned}
\end{equation}
The third line uses $\sum_{w\in V_{t,+}}p_t(w)=1$. The same
identity holds at $z=0,1$ by continuity, and for $p_t(v)=0$
both sides vanish. If there is only one possible token, the product
over competitors is empty and equals $1$, so this case is included.

\smallskip\noindent
\textbf{Extending this law to the entire generated text.}
Fix numbers $z_1,\ldots,z_n\in[0,1]$ and let
\begin{equation}
\mathcal{E}_j
=\{Y_{1:j}=y_{1:j},\ M_t\le z_t\text{ for }1\le t\le j\},
\qquad \mathcal{E}_0=\Omega.
\end{equation}
The event $\mathcal{E}_{j-1}$ is determined by $\mathcal{H}_j$.
On this event the prefix is $y_{<j}$, so the preceding single-step
identity and iterated conditional expectation yield
\begin{equation}
\begin{aligned}
\Pr_\kappa[\mathcal{E}_j\mid\mathcal{A}]
&=\mathbb{E}_\kappa\!\left[
  \mathbf{1}_{\mathcal{E}_{j-1}}
  \Pr_\kappa[M_j\le z_j,Y_j=y_j\mid\mathcal{H}_j]
  \,\middle|\,\mathcal{A}\right]\\
&=z_jp(y_j\mid x,y_{<j})
  \Pr_\kappa[\mathcal{E}_{j-1}\mid\mathcal{A}].
\end{aligned}
\end{equation}
Applying this recurrence for $j=1,\ldots,n$ gives
\begin{equation}
\Pr_\kappa[Y_{1:n}=y,\ M_t\le z_t\text{ for all }t\mid\mathcal{A}]
=\prod_{t=1}^n z_t p(y_t\mid x,y_{<t}).
\end{equation}
Setting every $z_t=1$ shows that
$\Pr_\kappa[Y_{1:n}=y\mid\mathcal{A}]
=\prod_{t=1}^n p(y_t\mid x,y_{<t})$.
Dividing by this positive probability therefore gives
\begin{equation}
\Pr_\kappa[M_t\le z_t\text{ for all }t\mid Y_{1:n}=y,\mathcal{A}]
=\prod_{t=1}^n z_t.
\end{equation}
This is the joint distribution function of independent uniforms.
It proves independence after conditioning on the full text, including
its future tokens, rather than only on each preceding prefix.

The pair $(y,\mathcal{A})$ determines all the scored sets $\{\mathcal{T}_c\}_{c=1}^C$.
For a fixed chunk $c$, set $z_t=1$ outside $\mathcal{T}_c$ in the last identity
and average over assignment tables consistent with $(y,\{\mathcal{T}_c\}_{c=1}^C)$.
The product over $\mathcal{T}_c$ is fixed under this conditioning, so
\begin{equation}
\Pr_*[M_t\le z_t\text{ for all }t\in \mathcal{T}_c]
=\mathbb{E}_*\!\left[\prod_{t\in \mathcal{T}_c}z_t\right]
=\prod_{t\in \mathcal{T}_c}z_t.
\end{equation}
Thus the selected $M_t$, $t\in \mathcal{T}_c$, remain independent $U(0,1)$
variables under $\Pr_*$.

At every scored position, the correct candidate reconstructs the
uniforms actually used by the encoder. Since $Y_t$ maximises $Z_t$,
\begin{equation}
B_t:=u_t^{(m_c)}(Y_t)=u_t(Y_t)=M_t^{p_t}.
\end{equation}
Consequently, for $0<b<1$,
\begin{equation}
\Pr_*[B_t\le b]
=\Pr_*[M_t\le b^{1/p_t}]=b^{1/p_t}.
\end{equation}
Writing $\alpha_t=1/p_t$, the variables $B_t$ are therefore independent
$\mathrm{Beta}(\alpha_t,1)$ variables with densities
$\alpha_t b^{\alpha_t-1}$
on $(0,1)$. In particular, the correct score has the representation
\begin{equation}
S:=S_c(m_c)=\sum_{t\in \mathcal{T}_c}X_t,
\qquad X_t=-\log(1-B_t),
\end{equation}
where the $X_t$ are conditionally independent.

\smallskip\noindent
\textbf{Independence of the wrong scores from the correct score.}
Write $\mathcal{W}_c=\{0,1\}^{k_c}\setminus\{m_c\}$ for the
wrong candidates. Let $\mathcal{H}$ be the complete encoder query
transcript together with the decoder's chunk-selection queries,
as in Appendix~\ref{app:null}. This information determines
$(Y_{1:n},\{\mathcal{T}_c\}_{c=1}^C)$ and the correct score $S$.

For $t\in \mathcal{T}_c$ and $m'\in\mathcal{W}_c$, the wrong candidate
uses the input
\begin{equation}
b_{t,m'}=\langle\mathsf{msg},c_t,m',Y_t\rangle.
\end{equation}
These inputs are distinct over all pairs $(t,m')$: scored positions
have distinct contexts, and distinct candidates have distinct
message fields. None is queried by the encoder. Indeed, a message
query with context $c_t$ uses the true value $m_c$, while a query
with another context has a different context field; the remaining
query families use different tags. This also rules out collisions
with encoder queries for other chunks, even if their true values
happen to equal $m'$.

Let $W$ be the vector of all reconstructed wrong-candidate uniforms
$u_t^{(m')}(Y_t)$, indexed by these pairs. By the random-function
argument in Appendix~\ref{app:null}, conditional on $\mathcal{H}$
its entries are independent uniforms. Hence for any vector $w$
with entries $w_{t,m'}\in[0,1]$,
\begin{equation}
\Pr_\kappa[W\le w\mid\mathcal{H}]
=\prod_{t\in \mathcal{T}_c}\prod_{m'\in\mathcal{W}_c}w_{t,m'},
\end{equation}
where $W\le w$ is interpreted coordinatewise. The right-hand side
is constant once the scored sets have been fixed. Since $S$ is
determined by $\mathcal{H}$, iterated conditional expectation gives
\begin{equation}
\begin{aligned}
\Pr_*[S\le s,W\le w]
&=\mathbb{E}_*\!\left[
  \mathbf{1}_{\{S\le s\}}
  \Pr_\kappa[W\le w\mid\mathcal{H}]\right]\\
&=\Pr_*[S\le s]
  \prod_{t\in \mathcal{T}_c}\prod_{m'\in\mathcal{W}_c}w_{t,m'}.
\end{aligned}
\end{equation}
This factorisation proves that $W$ is independent of $S$ under
$\Pr_*$, and that all its coordinates remain independent uniforms.
Each wrong score is a function of its own coordinates of $W$.
Thus, when $n_c\ge1$, each $S_c(m')$ is independent of $S$ and
has the $\Gamma(n_c,1)$ law derived in Appendix~\ref{app:null}.
The wrong scores are also mutually independent, although the union
bound below does not require that additional fact.

\smallskip\noindent
\textbf{Computing the two exponential moments.}
Assume for now that $n_c\ge1$, and fix $\lambda\in(0,1)$.
For a wrong candidate, integrate against its Gamma density:
\begin{equation}
\begin{aligned}
\mathbb{E}_*[e^{\lambda S_c(m')}]
&=\frac{1}{(n_c-1)!}
  \int_0^\infty s^{n_c-1}e^{-(1-\lambda)s}\,\mathrm{d}s\\
&=\frac{(1-\lambda)^{-n_c}}{(n_c-1)!}
  \int_0^\infty q^{n_c-1}e^{-q}\,\mathrm{d}q\\
&=(1-\lambda)^{-n_c}.
\end{aligned}
\end{equation}
The second line substitutes $q=(1-\lambda)s$, and the last integral
is $\Gamma(n_c)=(n_c-1)!$. The restriction $\lambda<1$ ensures
convergence. This is the moment-generating function of the wrong
score, evaluated at the positive argument $\lambda$.

For a correct-score contribution with $a=1/p$, define
\begin{equation}
\begin{aligned}
L_p(\lambda)
&:=\mathbb{E}[(1-B)^\lambda],\qquad B\sim\mathrm{Beta}(a,1),\\
L_p(\lambda)
&=a\int_0^1 b^{a-1}(1-b)^\lambda\,\mathrm{d}b\\
&=a\,\frac{\Gamma(a)\Gamma(\lambda+1)}
              {\Gamma(a+\lambda+1)}\\
&=\frac{\Gamma(a+1)\Gamma(\lambda+1)}
         {\Gamma(a+\lambda+1)}.
\end{aligned}
\end{equation}
Here the integral is Euler's beta integral, and the last equality
uses $\Gamma(a+1)=a\Gamma(a)$. For completeness, the beta identity
for $a,b>0$ follows from the Gamma integrals by the substitution
$u=rz$, $v=r(1-z)$, whose Jacobian is $r$:
\begin{equation}
\begin{aligned}
\Gamma(a)\Gamma(b)
&=\int_0^\infty\!\int_0^\infty
  u^{a-1}v^{b-1}e^{-(u+v)}\,\mathrm{d}u\,\mathrm{d}v\\
&=\left(\int_0^\infty r^{a+b-1}e^{-r}\,\mathrm{d}r\right)
  \left(\int_0^1z^{a-1}(1-z)^{b-1}\,\mathrm{d}z\right)\\
&=\Gamma(a+b)\int_0^1z^{a-1}(1-z)^{b-1}\,\mathrm{d}z.
\end{aligned}
\end{equation}
All integrands are nonnegative, which justifies separating these
integrals; division by $\Gamma(a+b)>0$ gives the identity used above
with $b=\lambda+1$.
Since $e^{-\lambda X_t}=(1-B_t)^\lambda$, we have established
\begin{equation}
\mathbb{E}_*[e^{-\lambda X_t}]=L_{p_t}(\lambda),
\qquad
\eta_p(\lambda)=\log(1-\lambda)-\log L_p(\lambda).
\end{equation}

\Needspace{19\baselineskip}
\smallskip\noindent
\textbf{Bounding the probability that one wrong score wins.}
For a fixed $m'\in\mathcal{W}_c$, the positivity of $\lambda$
gives
\begin{equation}
\{S_c(m')\ge S\}
=\{e^{\lambda(S_c(m')-S)}\ge1\}.
\end{equation}
Markov's inequality applies to this nonnegative exponential.
Using independence of $S_c(m')$ and $S$, and then independence
of the correct-score contributions, yields
\begin{equation}
\begin{aligned}
\Pr_*[S_c(m')\ge S]
&\le\mathbb{E}_*[e^{\lambda(S_c(m')-S)}]\\
&=\mathbb{E}_*[e^{\lambda S_c(m')}]
  \mathbb{E}_*[e^{-\lambda S}]\\
&=(1-\lambda)^{-n_c}
  \prod_{t\in \mathcal{T}_c}\mathbb{E}_*[e^{-\lambda X_t}]\\
&=(1-\lambda)^{-n_c}\prod_{t\in \mathcal{T}_c}L_{p_t}(\lambda)\\
&=\exp\!\left(-\sum_{t\in \mathcal{T}_c}
  [\log(1-\lambda)-\log L_{p_t}(\lambda)]\right)\\
&=\exp\!\left(-\sum_{t\in \mathcal{T}_c}\eta_{p_t}(\lambda)\right).
\end{aligned}
\end{equation}
The fifth line uses $|\mathcal{T}_c|=n_c$, placing one factor
$(1-\lambda)^{-1}$ with each term of the product.

\smallskip\noindent
\textbf{Optimising the exponent.}
Write $F_c(\lambda)=\sum_{t\in \mathcal{T}_c}\eta_{p_t}(\lambda)$.
For each fixed $p>0$, dominated convergence in the definition of
$L_p$ gives
\begin{equation}
\lim_{\lambda\downarrow0}L_p(\lambda)=1,
\qquad
\lim_{\lambda\uparrow1}L_p(\lambda)
=\mathbb{E}[1-B]=\frac{1}{1/p+1}>0.
\end{equation}
The integrands are bounded by $1$, and the last value also follows
by setting $\lambda=1$ in the Gamma ratio above. Thus $F_c$ is
continuous on $(0,1)$, tends to $0$ at the left endpoint, and,
when $n_c\ge1$, tends to $-\infty$ at the right endpoint.
Continuity on each compact subinterval, together with these
endpoint limits, shows that
\begin{equation}
0\le\Lambda_c:=\sup_{0<\lambda<1}F_c(\lambda)<\infty.
\end{equation}
The pairwise error bound holds for every $\lambda\in(0,1)$.
Taking the infimum of its right-hand side, and using continuity
and monotonicity of the exponential, therefore gives
\begin{equation}
\Pr_*[S_c(m')\ge S]
\le\inf_{0<\lambda<1}e^{-F_c(\lambda)}
=e^{-\sup_{0<\lambda<1}F_c(\lambda)}
=e^{-\Lambda_c}.
\end{equation}
No maximiser inside $(0,1)$ is required: a sequence of values
approaching the supremum gives the same bound.

\smallskip\noindent
\textbf{From pairwise comparisons to chunk and message errors.}
Fix any rule for breaking ties in the decoder's maximisation.
If $\hat m_c\neq m_c$, the selected wrong candidate has a score
at least as large as $S_c(m_c)=S$. Hence
\begin{equation}
\{\hat m_c\neq m_c\}
\subseteq\bigcup_{m'\in\mathcal{W}_c}\{S_c(m')\ge S\}.
\end{equation}
This inclusion remains valid even if a tie occurs. The conditional
union bound and $|\mathcal{W}_c|=2^{k_c}-1$ now give
\begin{equation}
\begin{aligned}
\Pr_*[\hat m_c\neq m_c]
&\le\sum_{m'\in\mathcal{W}_c}\Pr_*[S_c(m')\ge S]\\
&\le(2^{k_c}-1)e^{-\Lambda_c}\\
&<2^{k_c}e^{-\Lambda_c}
=\exp(k_c\ln2-\Lambda_c).
\end{aligned}
\end{equation}
The strict inequality uses $2^{k_c}-1<2^{k_c}$ and
$e^{-\Lambda_c}>0$, the latter following from finiteness of
$\Lambda_c$. This proves \eqref{eq:scaling} when $n_c\ge1$.

If $n_c=0$, every candidate score is the empty sum $0$, and
$\Lambda_c=\sup_{0<\lambda<1}0=0$. If there is at least one
wrong candidate, the claimed bound is valid because
\begin{equation}
\Pr_*[\hat m_c\neq m_c]\le1
\le2^{k_c}-1=(2^{k_c}-1)e^{-\Lambda_c}.
\end{equation}
If there are no wrong candidates, the decoder can only select
$m_c$, and both the error probability and the first bound are
$0$, for any value of $n_c$. These observations cover all
degenerate cases.

All chunks were analysed under the same conditioning on
$(Y_{1:n},\{\mathcal{T}_c\}_{c=1}^C)$. Since the decoded concatenation differs
from $m$ precisely when at least one chunk differs, a second
conditional union bound yields
\begin{equation}
\begin{aligned}
\Pr_*[\hat m\neq m]
&=\Pr_*\!\left[\bigcup_{c=1}^C\{\hat m_c\neq m_c\}\right]\\
&\le\sum_{c=1}^C\Pr_*[\hat m_c\neq m_c]\\
&\le\sum_{c=1}^C(2^{k_c}-1)e^{-\Lambda_c}.
\end{aligned}
\end{equation}
No independence between chunks is needed. If one conditions only
on the text, the scored sets are still random, so the corresponding
bound is obtained by averaging their right-hand side:
\begin{equation}
\Pr_\kappa[\hat m\neq m\mid Y_{1:n}=y]
\le\mathbb{E}_\kappa\!\left[
  \sum_{c=1}^C(2^{k_c}-1)e^{-\Lambda_c}
  \,\middle|\,Y_{1:n}=y\right].
\end{equation}

\smallskip\noindent
\textbf{Monotonicity of $\eta_p(\lambda)$ in the token probability.}
Fix $\lambda\in(0,1)$ and let $U\sim U(0,1)$. The variable
$B_p=U^p$ has distribution function $b^{1/p}$ on $(0,1)$,
so $B_p\sim\mathrm{Beta}(1/p,1)$ and
\begin{equation}
L_p(\lambda)=\mathbb{E}[(1-U^p)^\lambda].
\end{equation}
This represents the distributions for different $p$ using the
same uniform variable. If $0<p_1<p_2\le1$, then for every
$U\in(0,1)$,
\begin{equation}
U^{p_1}>U^{p_2},\qquad
1-U^{p_1}<1-U^{p_2},\qquad
(1-U^{p_1})^\lambda<(1-U^{p_2})^\lambda.
\end{equation}
Taking expectations gives $L_{p_1}(\lambda)<L_{p_2}(\lambda)$.
Since the logarithm is strictly increasing, subtracting these
logarithms from the same value $\log(1-\lambda)$ gives
$\eta_{p_1}(\lambda)>\eta_{p_2}(\lambda)$.
In particular, $\eta_p(\lambda)$ is nonincreasing in $p$, as claimed.

\smallskip\noindent
\textbf{The endpoint $p=1$.}
When $p=1$, the Beta variable is uniform and
\begin{equation}
L_1(\lambda)
=\int_0^1(1-u)^\lambda\,\mathrm{d}u
=\left[-\frac{(1-u)^{\lambda+1}}{\lambda+1}\right]_0^1
=\frac{1}{\lambda+1}.
\end{equation}
Substitution into the definition of $\eta$ yields
\begin{equation}
\eta_1(\lambda)
=\log(1-\lambda)+\log(1+\lambda)
=\log(1-\lambda^2)<0,
\end{equation}
because $0<1-\lambda^2<1$. For example, if every scored token
has $p_t=1$, then $F_c(\lambda)=n_c\log(1-\lambda^2)$ and
$\Lambda_c=0$, approached as $\lambda\downarrow0$ when $n_c>0$.
This also illustrates why the supremum in the theorem need not
be attained inside the open interval.
\end{proof}

\Needspace{12\baselineskip}
\section{Proof of Theorem~\ref{thm:llr}}
\label{app:llr}

\noindent The main text states only the final log-likelihood ratio. The full
statement below additionally gives the conditional candidate laws, the product
likelihood, the maximum-likelihood and most-powerful-test conclusions, and the
certified weighted decoder; it keeps the theorem's number and the main
log-likelihood equation number.
\begingroup
\setcounter{savedtheorem}{\value{theorem}}%
\setcounterref{theorem}{thm:llr}\addtocounter{theorem}{-1}%
\def\theHtheorem{full.\thetheorem}\def\theHequation{full.\arabic{equation}.\the\inputlineno}%
\begin{theorem}[Exact likelihood ratio, full statement]
\label{thm:llr_full}
Treat $\PRF$ as a random function. Fix a chunk $c$ with $k_c\ge1$
and condition on the emitted text and the scored-position sets
$\{\mathcal{T}_c\}_{c=1}^C$, fixing $p_t=p_t(y_t)$ for $t\in \mathcal{T}_c$.
All scored uniforms $u_t^{(m')}(y_t)$ are conditionally independent
across positions and candidates, with density
$f_{\mathrm{correct}}(\cdot\mid p_t)$ when $m'=m_c$ and density
$f_{\mathrm{wrong}}\equiv1$ otherwise. For a fixed candidate $m'$,
the likelihood ratio of $H_1:m_c=m'$ to $H_0:m_c\neq m'$, based
on that candidate's scored uniforms, is
\begin{equation}
L_c(m')\;=\;\prod_{t\in \mathcal{T}_c}\frac{f_{\mathrm{correct}}\bigl(u^{(m')}_t(y_t)\mid p_t\bigr)}{f_{\mathrm{wrong}}\bigl(u^{(m')}_t(y_t)\bigr)}\;=\;\prod_{t\in \mathcal{T}_c}f_{p_t}\bigl(u^{(m')}_t(y_t)\bigr),
\label{eq:lr_prod}
\end{equation}
and its logarithm, the quantity the decoder computes, is the sum
\begin{equation}
\begin{aligned}
\ell_c(m')=\log L_c(m')
&=\sum_{t\in \mathcal{T}_c}\log f_{p_t}\bigl(u^{(m')}_t(y_t)\bigr)\\
&=\sum_{t\in \mathcal{T}_c}\Bigl[
  \bigl(\tfrac{1}{p_t}-1\bigr)\log u^{(m')}_t(y_t)
  +\log\tfrac{1}{p_t}\Bigr].
\end{aligned}
\tag{\ref{eq:llr}}
\end{equation}
The decoder $\arg\max_{m'}\ell_c(m')$ is the conditional
maximum-likelihood decoder based on all candidates' scored uniforms.
For distinct $m',m''$, the test of $H_0:m_c=m''$ against
$H_1:m_c=m'$ that rejects for large $\ell_c(m')-\ell_c(m'')$
is most powerful at its conditional level, with randomisation at
the threshold if needed.

More generally, let $w_t\ge0$ be finite weights determined by the
text, prompt and model, and write
\begin{equation}
R_{c,w}(m')=\sum_{t\in \mathcal{T}_c}w_t\log u_t^{(m')}(y_t),
\qquad \hat m_{c,w}=\arg\max_{m'}R_{c,w}(m').
\end{equation}
Under a wrong candidate, $-R_{c,w}(m')$ has the law of
$Z_{c,w}=\sum_{t\in \mathcal{T}_c:w_t>0}w_tE_t$, where the $E_t$ are
independent $\mathrm{Exp}(1)$ variables; zero weights contribute
$0$. Let $F_{c,w}(z)=\Pr[Z_{c,w}\le z]$. The certificate
\begin{equation}
\hat\delta_{c,w}(y)
=\min\bigl\{1,(2^{k_c}-1)
  F_{c,w}\bigl(-R_{c,w}(\hat m_{c,w})\bigr)\bigr\},
\end{equation}
set to $1$ if all weights are zero, gives the guarantee
\eqref{eq:cert} when the decoder returns $\hat m_{c,w}$ if
$\hat\delta_{c,w}\le\delta$ and abstains otherwise.
\end{theorem}
\setcounter{theorem}{\value{savedtheorem}}%
\endgroup

\begin{proof}
Fix the model, prompt, sampling settings and the values carried by all
chunks other than $c$. Let $\mathcal{M}_c=\{0,1\}^{k_c}$ and
$K=|\mathcal{M}_c|\ge2$. To compare the possible values of chunk $c$,
write $H_a$ for the hypothesis that $m_c=a$, where
$a\in\mathcal{M}_c$. This comparison does not require a prior
distribution on the message.

Condition on a realised text $Y_{1:n}=y$ and scored-position sets
$\{\mathcal{T}_c\}_{c=1}^C$ with positive probability. These fix
$\mathcal{T}_c$, $n_c$ and $p_t=p(y_t\mid x,y_{<t})\in(0,1]$. Appendix~\ref{app:scaling}
shows that the text law conditional on the chunk-assignment table is
$\prod_t p(y_t\mid x,y_{<t})$, independently of the message. The
assignment table itself is also independent of the message. Hence
$(Y_{1:n},\{\mathcal{T}_c\}_{c=1}^C)$ has the same law under every $H_a$, and the
same conditioning is valid for all candidates. Denote conditional
probability and expectation under $H_a$ by $\Pr_a$ and $\mathbb{E}_a$.

\smallskip\noindent
\textbf{The joint observation and its conditional density.}
For $t\in \mathcal{T}_c$ and $b\in\mathcal{M}_c$, put
\begin{equation}
U_{t,b}=u_t^{(b)}(y_t),\qquad
U_b=(U_{t,b})_{t\in \mathcal{T}_c},\qquad
\mathcal{U}=(U_b)_{b\in\mathcal{M}_c}.
\end{equation}
Thus $U_b$ is the column of scored uniforms for candidate $b$,
and $\mathcal{U}$ is the array used to compare all candidates.

Under $H_a$, Appendix~\ref{app:scaling} proves that the coordinates
of $U_a$ are independent, with distribution functions $u^{1/p_t}$
on $(0,1)$. Differentiation gives their densities:
\begin{equation}
\frac{\mathrm{d}}{\mathrm{d}u}u^{1/p_t}
=\frac{1}{p_t}u^{1/p_t-1}=f_{p_t}(u).
\end{equation}
The other columns consist of independent uniforms and are also
independent of $U_a$. To justify this last joint assertion, let
$\mathcal{H}$ contain the complete generation transcript and the
decoder's chunk-selection queries, as in Appendix~\ref{app:null}.
It determines both the conditioning information and the correct
column $U_a$. Every entry in the remaining columns is read at a
distinct input that was never queried during generation:
at context $c_t$, the encoder uses $a$, while a wrong column
uses $b\neq a$. Distinct retained contexts and the domain-separating
tags rule out the other possible input collisions.

Consequently, conditional on $\mathcal{H}$, the array
$U_{-a}=(U_b)_{b\neq a}$ has the product uniform law, denoted by
$\nu$, independently of the values in that transcript. For any
measurable sets $A$ and $B$, iterated conditional expectation gives
\begin{equation}
\begin{aligned}
\Pr_a[U_a\in A,U_{-a}\in B]
&=\mathbb{E}_a\!\left[
  \mathbf{1}_{\{U_a\in A\}}
  \Pr_a[U_{-a}\in B\mid\mathcal{H}]\right]\\
&=\mathbb{E}_a\!\left[
  \mathbf{1}_{\{U_a\in A\}}\nu(B)\right]\\
&=\Pr_a[U_a\in A]\,\nu(B).
\end{aligned}
\end{equation}
This factorisation establishes the required independence. In
particular, for $n_c\ge1$, the conditional density of the whole
array at $u=(u_{t,b})\in(0,1)^{n_cK}$ under $H_a$ is
\begin{equation}
\begin{aligned}
q_a(u)
&=\left(\prod_{t\in \mathcal{T}_c}f_{p_t}(u_{t,a})\right)
  \left(\prod_{b\in\mathcal{M}_c\setminus\{a\}}
        \prod_{t\in \mathcal{T}_c}1\right)\\
&=\prod_{t\in \mathcal{T}_c}f_{p_t}(u_{t,a}).
\end{aligned}
\end{equation}
All these densities are positive on the same observation space.

\smallskip\noindent
\textbf{The likelihood ratio for one candidate.}
Fix a candidate $b$ and consider only its column $U_b$. Under
$H_1:m_c=b$, its density at $z=(z_t)_{t\in \mathcal{T}_c}$ is
\begin{equation}
h_1(z)=\prod_{t\in \mathcal{T}_c}f_{p_t}(z_t).
\end{equation}
Under any hypothesis $H_a$ with $a\neq b$, this column instead
has density
\begin{equation}
h_0(z)=\prod_{t\in \mathcal{T}_c}1=1.
\end{equation}
Although $H_0:m_c\neq b$ contains several possible messages,
all of them give exactly this same distribution for $U_b$.
Thus this candidate-specific likelihood ratio is unambiguous
without choosing a distribution over the wrong messages.
Evaluating $h_1/h_0$ at the observed column gives
\begin{equation}
\frac{h_1(U_b)}{h_0(U_b)}
=\prod_{t\in \mathcal{T}_c}
  \frac{f_{\mathrm{correct}}(U_{t,b}\mid p_t)}
       {f_{\mathrm{wrong}}(U_{t,b})}
=\prod_{t\in \mathcal{T}_c}f_{p_t}(U_{t,b})
=L_c(b),
\end{equation}
which proves \eqref{eq:lr_prod}.

Every factor is positive and finite because $p_t>0$ and
$0<U_{t,b}<1$. Taking logarithms is therefore valid, and gives
\begin{equation}
\begin{aligned}
\ell_c(b)
&=\log\prod_{t\in \mathcal{T}_c}
  \left[p_t^{-1}U_{t,b}^{\,1/p_t-1}\right]\\
&=\sum_{t\in \mathcal{T}_c}
  \left[\log p_t^{-1}+
        \log U_{t,b}^{\,1/p_t-1}\right]\\
&=\sum_{t\in \mathcal{T}_c}
  \left[\left(\frac1{p_t}-1\right)\log U_{t,b}
        +\log\frac1{p_t}\right].
\end{aligned}
\end{equation}
The second line uses the logarithm of a finite product, and the
third uses $\log u^\beta=\beta\log u$ for $u>0$. This proves
\eqref{eq:llr}. If $p_t=1$, both coefficients in that token's
contribution vanish, so its log-likelihood contribution is $0$.
If $\mathcal{T}_c$ is empty, the product is $1$ and the sum is $0$ for
every candidate.

\smallskip\noindent
\textbf{Maximum likelihood from the full array of scored uniforms.}
The preceding candidate-specific ratio uses one column at a time.
To justify comparing candidates, we use the density of the same
full observation $\mathcal{U}$ under each hypothesis. The joint
density already derived satisfies
\begin{equation}
q_a(\mathcal{U})
=\prod_{t\in \mathcal{T}_c}f_{p_t}(U_{t,a})
=L_c(a).
\end{equation}
Hence, as sets of maximisers,
\begin{equation}
\arg\max_{a\in\mathcal{M}_c}q_a(\mathcal{U})
=\arg\max_{a\in\mathcal{M}_c}L_c(a)
=\arg\max_{a\in\mathcal{M}_c}\ell_c(a).
\end{equation}
The last equality holds because the logarithm is strictly
increasing. Choosing any fixed rule among tied maximisers thus
gives a conditional maximum-likelihood decoder for this
observation. The observation here consists of the scored
uniforms for all candidates, together with the fixed text,
scored sets and model probabilities.

\smallskip\noindent
\textbf{The pairwise test and its rejection direction.}
Fix two distinct candidates $b,d$. For the simple hypotheses
\begin{equation}
H_0:m_c=d,\qquad H_1:m_c=b,
\end{equation}
the likelihood ratio based on $\mathcal{U}$ is
\begin{equation}
\mathcal{R}(\mathcal{U})
=\frac{q_b(\mathcal{U})}{q_d(\mathcal{U})}
=\frac{L_c(b)}{L_c(d)}
=\exp\bigl(\ell_c(b)-\ell_c(d)\bigr).
\end{equation}
A large value therefore favours $H_1$, so the rejection region
for $H_0$ uses a large value of $\ell_c(b)-\ell_c(d)$.

For completeness, we prove the most-powerful assertion. Let
$\alpha\in(0,1)$ be the prescribed conditional test level.
Choose a threshold $\tau>0$ satisfying
\begin{equation}
\Pr_d[\mathcal{R}>\tau]\le\alpha
\le\Pr_d[\mathcal{R}\ge\tau].
\end{equation}
Such a threshold is a quantile of the positive, finite random
variable $\mathcal{R}$. Define the rejection probability
\begin{equation}
\phi^*(u)=
\begin{cases}
1,&\mathcal{R}(u)>\tau,\\
\gamma,&\mathcal{R}(u)=\tau,\\
0,&\mathcal{R}(u)<\tau,
\end{cases}
\end{equation}
where $\gamma\in[0,1]$ is chosen so that
\begin{equation}
\mathbb{E}_d[\phi^*]
=\Pr_d[\mathcal{R}>\tau]
 +\gamma\Pr_d[\mathcal{R}=\tau]
=\alpha.
\end{equation}
If the boundary has positive probability, take
$\gamma=(\alpha-\Pr_d[\mathcal{R}>\tau])/
\Pr_d[\mathcal{R}=\tau]$; the quantile inequalities place this
number in $[0,1]$. If the boundary has probability zero, the
tail probability already equals $\alpha$ and any $\gamma$ works.

Let $\phi(u)\in[0,1]$ be any other test using this observation
with $\mathbb{E}_d[\phi]\le\alpha$. At every $u$,
\begin{equation}
(\phi^*(u)-\phi(u))(q_b(u)-\tau q_d(u))\ge0.
\end{equation}
Indeed, above the threshold both factors are nonnegative,
below it both are nonpositive, and on the boundary the second
factor is zero. Integrating this inequality over the observation
space yields
\begin{equation}
\begin{aligned}
\mathbb{E}_b[\phi^*]-\mathbb{E}_b[\phi]
&\ge\tau\bigl(\mathbb{E}_d[\phi^*]
             -\mathbb{E}_d[\phi]\bigr)\\
&=\tau\bigl(\alpha-\mathbb{E}_d[\phi]\bigr)\ge0.
\end{aligned}
\end{equation}
Thus $\phi^*$ has at least as much power as every competing
level-$\alpha$ test. This is the Neyman--Pearson argument in
the present conditional experiment. The endpoint levels
$\alpha=0,1$ are handled by never rejecting and always
rejecting, respectively; positivity of both densities ensures
that a level-zero test also has zero power.

If $\mathcal{T}_c=\varnothing$, the observation is an empty array and all
likelihoods are $1$. The same decision problem then consists
only of randomisation. Likewise, if every scored token has
$p_t=1$, all hypotheses induce the same product uniform law,
$\ell_c(b)-\ell_c(d)=0$, and a size-$\alpha$ test has power
$\alpha$. Boundary randomisation includes these cases.

\smallskip\noindent
\textbf{The exact null distribution of a weighted log score.}
Return to a fixed true message $m_c$, and let the finite weights
$w_t\ge0$ be as in the theorem. Their dependence only on the
text, prompt and model makes them fixed under our conditioning.
For a wrong candidate $b\neq m_c$, define
\begin{equation}
X_{t,w}=-w_t\log U_{t,b},\qquad
Z_b=\sum_{t\in \mathcal{T}_c}X_{t,w}=-R_{c,w}(b).
\end{equation}
If $w_t>0$, then for $z\ge0$ the uniform null law gives
\begin{equation}
\begin{aligned}
\Pr_{m_c}[X_{t,w}>z]
&=\Pr_{m_c}[-w_t\log U_{t,b}>z]\\
&=\Pr_{m_c}[\log U_{t,b}<-z/w_t]\\
&=\Pr_{m_c}[U_{t,b}<e^{-z/w_t}]\\
&=e^{-z/w_t}.
\end{aligned}
\end{equation}
Thus $X_{t,w}$ is exponential with rate $1/w_t$ and density
$w_t^{-1}e^{-z/w_t}$ for $z>0$. If $w_t=0$, then
$X_{t,w}=0$ deterministically; no exponential rate is assigned
to that term. The summands are independent because the
$U_{t,b}$ are independent and the weights are fixed.
Equivalently,
\begin{equation}
Z_b\ \stackrel{d}{=}\ Z_{c,w}
=\sum_{t\in \mathcal{T}_c:w_t>0}w_tE_t,
\qquad E_t\ \text{independent }\mathrm{Exp}(1).
\end{equation}

This distribution is specified exactly even when positive weights
repeat. To make its computation explicit, enumerate the active
positions as $t_1,\ldots,t_r$ and put $\lambda_j=1/w_{t_j}$.
For $r\ge1$, let $g_j$ be the density of the sum of its first
$j$ exponential terms. Independence gives the convolution
recursion, for $z>0$,
\begin{equation}
g_1(z)=\lambda_1e^{-\lambda_1z},\qquad
g_j(z)=\int_0^z g_{j-1}(v)\lambda_j
                  e^{-\lambda_j(z-v)}\,\mathrm{d}v
\quad (2\le j\le r).
\end{equation}
Consequently,
\begin{equation}
F_{c,w}(z)=
\begin{cases}
0,&z<0,\\
\displaystyle\int_0^z g_r(v)\,\mathrm{d}v,&z\ge0.
\end{cases}
\end{equation}
Each $g_j$ is positive on $(0,\infty)$: this holds for $g_1$,
and the recursion integrates a positive integrand over $(0,z)$.
Hence $F_{c,w}$ is continuous and strictly increasing on
$[0,\infty)$, with $F_{c,w}(0)=0$ and limit $1$ at infinity.

An equivalent expression for the Laplace transform is
\begin{equation}
\mathbb{E}[e^{-sZ_{c,w}}]
=\prod_{j=1}^r
  \int_0^\infty\lambda_je^{-(\lambda_j+s)z}\,\mathrm{d}z
=\prod_{j=1}^r\frac{\lambda_j}{\lambda_j+s}
=\prod_{j=1}^r(1+s w_{t_j})^{-1},
\qquad s\ge0.
\end{equation}
If all active weights equal $w>0$, then
$Z_{c,w}/w\sim\Gamma(r,1)$; unequal weights are covered by
the same convolution recursion. If $r=0$, the sum is zero
with probability one.

\Needspace{12\baselineskip}
\smallskip\noindent
\textbf{The appropriate tail and the certification bound.}
The weighted decoder maximises $R_{c,w}$, whereas its known
nonnegative null variable is $Z_{c,w}=-R_{c,w}$.
For any fixed score value $s\le0$,
\begin{equation}
\Pr[-Z_{c,w}\ge s]
=\Pr[Z_{c,w}\le-s]=F_{c,w}(-s).
\end{equation}
Thus the upper tail of the score corresponds to the
\emph{lower} tail of its negation. Substituting the selected
score into this null tail and multiplying by the number of
wrong candidates gives exactly
\begin{equation}
\hat\delta_{c,w}
=\min\{1,(K-1)F_{c,w}(-R_{c,w}(\hat m_{c,w}))\}.
\end{equation}

We now verify its error guarantee, including the selection of
$\hat m_{c,w}$. At level $\delta=1$ the claim is immediate
because every error probability is at most $1$. For
$0<\delta<1$, if there are no positive weights, the certificate
is defined to be $1$ and the decoder always abstains. It
remains to consider $r\ge1$.

Set $q_\delta=\delta/(K-1)\in(0,1)$. The continuity and strict
increase of $F_{c,w}$ give a unique $z_\delta>0$ such that
\begin{equation}
F_{c,w}(z_\delta)=q_\delta.
\end{equation}
Since $\delta<1$, clipping the certificate at $1$ does not
change its acceptance event. Therefore
\begin{equation}
\begin{aligned}
\hat\delta_{c,w}\le\delta
&\quad\Longleftrightarrow\quad
F_{c,w}(-R_{c,w}(\hat m_{c,w}))\le q_\delta\\
&\quad\Longleftrightarrow\quad
-R_{c,w}(\hat m_{c,w})\le z_\delta.
\end{aligned}
\end{equation}
Let $D_{\delta,w}$ denote the decoder that returns $\hat m_{c,w}$
on this event and returns $\bot$ otherwise. If it accepts a
wrong candidate, that candidate's negated score must meet the
last inequality. Hence, under the current conditioning,
\begin{equation}
\{D_{\delta,w}(y)\notin\{m_c,\bot\}\}
\subseteq\bigcup_{b\neq m_c}\{Z_b\le z_\delta\}.
\end{equation}
Every fixed wrong candidate has the same null distribution
$F_{c,w}$. Applying the union bound thus gives
\begin{equation}
\begin{aligned}
\Pr_{m_c}[D_{\delta,w}(y)\notin\{m_c,\bot\}]
&\le\sum_{b\neq m_c}\Pr_{m_c}[Z_b\le z_\delta]\\
&=(K-1)F_{c,w}(z_\delta)\\
&=(K-1)\frac{\delta}{K-1}=\delta.
\end{aligned}
\end{equation}
This step needs no independence between candidate scores and
works with any fixed tie-breaking rule. The bound holds for
every conditioning pair of positive probability, so averaging
over both the text and the scored sets gives
\begin{equation}
\begin{aligned}
&\Pr_\kappa[D_{\delta,w}(Y_{1:n})\notin\{m_c,\bot\}]\\
&\quad=\mathbb{E}_\kappa\!\left[
  \Pr_\kappa[D_{\delta,w}(Y_{1:n})\notin\{m_c,\bot\}
  \mid Y_{1:n},\{\mathcal{T}_c\}_{c=1}^C]\right]
\le\delta.
\end{aligned}
\end{equation}
This is the guarantee \eqref{eq:cert} for the weighted decoder.

\Needspace{20\baselineskip}
\smallskip\noindent
\textbf{Applying the certificate to the model-aware score.}
Finally, define
\begin{equation}
w_t^*=\frac1{p_t}-1\ge0,\qquad
A_c=\sum_{t\in \mathcal{T}_c}\log\frac1{p_t}.
\end{equation}
The log-likelihood formula proved above becomes
\begin{equation}
\ell_c(b)=A_c+R_{c,w^*}(b).
\end{equation}
The term $A_c$ does not depend on the candidate, so maximising
$\ell_c$ is equivalent to maximising $R_{c,w^*}$. For the
selected candidate $\hat m_\ell=\arg\max_b\ell_c(b)$, its
negated weighted score is
\begin{equation}
-R_{c,w^*}(\hat m_\ell)=A_c-\ell_c(\hat m_\ell).
\end{equation}
The corresponding certificate is therefore
\begin{equation}
\min\left\{1,(K-1)
F_{c,w^*}\bigl(A_c-\ell_c(\hat m_\ell)\bigr)\right\},
\end{equation}
with value $1$ when all $w_t^*$ vanish. In particular, tokens
with $p_t=1$ have zero weight and contribute neither to the
likelihood score nor to its null sum. If all scored tokens
are deterministic, or if there are no scored positions, the
decoder abstains at every level $\delta<1$. This completes
both the likelihood and certification claims.
\end{proof}

\Needspace{12\baselineskip}
\section{Proof of Theorem~\ref{thm:robust}}
\label{app:robust}

\noindent The main text states the contamination model, the robust score and its guarantees in brief. The full statement, which the proof below establishes, follows; it keeps the theorem's number and equation number.
\begingroup
\setcounter{savedtheorem}{\value{theorem}}%
\setcounterref{theorem}{thm:robust}\addtocounter{theorem}{-1}%
\def\theHtheorem{full.\thetheorem}\def\theHequation{full.\arabic{equation}.\the\inputlineno}%
\begin{theorem}[Robust decoding under contamination, full statement]
\label{thm:robust_full}
Fix $\varepsilon\in(0,1)$ and a chunk $c$ with $k_c\ge1$.
Condition on the text and scored-position sets
$\{\mathcal{T}_c\}_{c=1}^C$, and assume $p_t=p_t(y_t)\in(0,1]$
for every $t\in \mathcal{T}_c$. In the contamination model, the labels of
scored tokens are independent, with probabilities $1-\varepsilon$
of being \emph{watermarked} and $\varepsilon$ of being \emph{foreign}.
Conditional on these labels, the scored uniforms for each fixed
candidate are independent across positions. For the true candidate,
their laws are $\mathrm{Beta}(1/p_t,1)$ at watermarked positions
and $U(0,1)$ at foreign positions; for every wrong candidate they
are $U(0,1)$ at all positions, regardless of the labels.
Then the log-likelihood ratio of $H_1:m_c=m'$ to
$H_0:m_c\neq m'$, based on candidate $m'$'s scored uniforms, is
\begin{equation}
r_c(m')\;=\;\sum_{t\in \mathcal{T}_c}\log\Bigl[(1-\varepsilon)\,f_{p_t}\bigl(u^{(m')}_t(y_t)\bigr)+\varepsilon\Bigr],
\tag{\ref{eq:robust}}
\end{equation}
and (i)~each summand lies between $\log\varepsilon$ and
$\log\bigl(\varepsilon+(1-\varepsilon)/p_t\bigr)$;
(ii)~a token with $p_t=1$ contributes exactly $0$; and
(iii)~every wrong candidate's score has the exact null law
\begin{equation}
Z_{c,\varepsilon}
=\sum_{t\in \mathcal{T}_c}\log\bigl[\varepsilon+(1-\varepsilon)f_{p_t}(U_t)\bigr],
\qquad U_t\ \text{i.i.d. }U(0,1).
\end{equation}
Write $Q_{c,\varepsilon}(s)=\Pr[Z_{c,\varepsilon}\ge s]$ and
$\hat m_{c,\varepsilon}=\arg\max_{m'}r_c(m')$. The certificate
\begin{equation}
\hat\delta_{c,\varepsilon}
=\min\bigl\{1,(2^{k_c}-1)
Q_{c,\varepsilon}\bigl(r_c(\hat m_{c,\varepsilon})\bigr)\bigr\}
\end{equation}
gives the error bound of Theorem~\ref{thm:cert} under this model:
returning $\hat m_{c,\varepsilon}$ only when
$\hat\delta_{c,\varepsilon}\le\delta$ and abstaining otherwise
has probability at most $\delta$ of returning an incorrect chunk.
If $\mathcal{T}_c$ is empty or every $p_t=1$, this certificate equals $1$.
\end{theorem}
\setcounter{theorem}{\value{savedtheorem}}%
\endgroup

\begin{proof}
Fix the chunk $c$, the contamination parameter
$\varepsilon\in(0,1)$, the prompt, model and sampling settings.
Condition on the text $Y_{1:n}=y$ and scored-position sets
$\{\mathcal{T}_c\}_{c=1}^C$. Thus $\mathcal{T}_c$, $n_c$ and all
$p_t=p_t(y_t)\in(0,1]$ are fixed. Write
$\mathcal{M}_c=\{0,1\}^{k_c}$ and $K=|\mathcal{M}_c|\ge2$.
The parameter $\varepsilon$ is fixed before observing the
candidate uniforms. All conditional distributions below are
those of the contamination model stated in the theorem.

\smallskip\noindent
\textbf{Marginalising the contamination labels.}
Fix a candidate $a\in\mathcal{M}_c$ and abbreviate
$U_t=u_t^{(a)}(y_t)$ for $t\in \mathcal{T}_c$. Let $B_t=1$ mean that
position $t$ is watermarked and $B_t=0$ mean that it is foreign.
Under the model, the $B_t$ are independent, with
\begin{equation}
\Pr[B_t=1]=1-\varepsilon,\qquad
\Pr[B_t=0]=\varepsilon.
\end{equation}
For the hypothesis $H_1:m_c=a$, the conditional density of
$U_t$ is $f_{p_t}(u)$ when $B_t=1$ and $1$ when $B_t=0$.
The law of total probability therefore gives the marginal density
\begin{equation}
h_{p_t,\varepsilon}(u)
=(1-\varepsilon)f_{p_t}(u)+\varepsilon
=\frac{1-\varepsilon}{p_t}u^{1/p_t-1}+\varepsilon,
\qquad 0<u<1.
\end{equation}
This is a probability density: both component densities
integrate to $1$, so
\begin{equation}
\int_0^1h_{p_t,\varepsilon}(u)\,\mathrm{d}u
=(1-\varepsilon)\int_0^1f_{p_t}(u)\,\mathrm{d}u
 +\varepsilon\int_0^1 1\,\mathrm{d}u
=(1-\varepsilon)+\varepsilon=1.
\end{equation}

We also need the \emph{joint} density across scored positions.
Independence of the labels and independence of the uniforms
conditional on those labels are both used here. For a label
vector $b=(b_t)_{t\in \mathcal{T}_c}\in\{0,1\}^{\mathcal{T}_c}$, its probability
is $\prod_t(1-\varepsilon)^{b_t}\varepsilon^{1-b_t}$.
Conditional on this vector and $H_1$, the uniforms have joint
density $\prod_t f_{p_t}(u_t)^{b_t}$. Summing over all label
vectors yields
\begin{equation}
\begin{aligned}
q_1((u_t)_{t\in \mathcal{T}_c})
&=\sum_{b\in\{0,1\}^{\mathcal{T}_c}}
  \prod_{t\in \mathcal{T}_c}
  \left[(1-\varepsilon)^{b_t}\varepsilon^{1-b_t}
        f_{p_t}(u_t)^{b_t}\right]\\
&=\prod_{t\in \mathcal{T}_c}
  \left[\sum_{b_t\in\{0,1\}}
    (1-\varepsilon)^{b_t}\varepsilon^{1-b_t}
    f_{p_t}(u_t)^{b_t}\right]\\
&=\prod_{t\in \mathcal{T}_c}
  \left[\varepsilon+(1-\varepsilon)f_{p_t}(u_t)\right]\\
&=\prod_{t\in \mathcal{T}_c}h_{p_t,\varepsilon}(u_t).
\end{aligned}
\end{equation}
The second equality distributes the finite sum over all binary
label choices into one two-term sum per position. It shows
explicitly that independence across positions survives
marginalisation over the unobserved labels.

Under $H_0:m_c\neq a$, every $U_t$ is uniform even conditional
on the labels, and these uniforms are independent across
positions by assumption. Their conditional joint density is
therefore $1$ for every label vector. Averaging over the labels
leaves
\begin{equation}
q_0((u_t)_{t\in \mathcal{T}_c})
=\sum_{b\in\{0,1\}^{\mathcal{T}_c}}
  \prod_{t\in \mathcal{T}_c}(1-\varepsilon)^{b_t}\varepsilon^{1-b_t}
=\prod_{t\in \mathcal{T}_c}\bigl[(1-\varepsilon)+\varepsilon\bigr]=1.
\end{equation}
This is the same density for every possible true value other
than $a$. Hence the composite hypothesis $m_c\neq a$ induces
one common null law for this candidate's observed uniforms;
no prior over the other message values is required.

\smallskip\noindent
\textbf{The likelihood ratio and the robust score.}
Since $h_{p,\varepsilon}(u)\ge\varepsilon>0$ on $(0,1)$,
the likelihood ratio is positive and its logarithm is defined.
Evaluating the two densities at the candidate's observed
uniforms gives
\begin{equation}
\begin{aligned}
\log\frac{q_1((U_t)_{t\in \mathcal{T}_c})}{q_0((U_t)_{t\in \mathcal{T}_c})}
&=\log\prod_{t\in \mathcal{T}_c}h_{p_t,\varepsilon}(U_t)\\
&=\sum_{t\in \mathcal{T}_c}\log h_{p_t,\varepsilon}(U_t)\\
&=\sum_{t\in \mathcal{T}_c}
  \log\bigl[\varepsilon+(1-\varepsilon)f_{p_t}(U_t)\bigr]\\
&=r_c(a).
\end{aligned}
\end{equation}
This proves \eqref{eq:robust}. If $\mathcal{T}_c=\varnothing$, the
joint likelihoods are both the empty product $1$, their
ratio is $1$, and the robust score is the empty sum $0$.

\smallskip\noindent
\textbf{Bounds on a token's score contribution.}
For $p\in(0,1]$, the exponent $1/p-1$ is nonnegative.
Consequently, for $0<u<1$,
\begin{equation}
0\le u^{1/p-1}\le1,\qquad
0\le f_p(u)=\frac1p u^{1/p-1}\le\frac1p.
\end{equation}
Multiplication by $1-\varepsilon>0$ and addition of
$\varepsilon$ preserve these inequalities:
\begin{equation}
\varepsilon
\le h_{p,\varepsilon}(u)
\le\varepsilon+\frac{1-\varepsilon}{p}.
\end{equation}
The logarithm is increasing on $(0,\infty)$, so the
per-token score
\begin{equation}
s_{p,\varepsilon}(u)
:=\log h_{p,\varepsilon}(u)
\end{equation}
satisfies
\begin{equation}
\log\varepsilon
\le s_{p,\varepsilon}(u)
\le\log\left(\varepsilon+\frac{1-\varepsilon}{p}\right).
\end{equation}
Both endpoints are finite under the stated assumptions.
These bounds are algebraic and therefore hold for every
candidate and for either contamination label.

To state the negative-contribution bound precisely, let
$r_{c,-t}(a)$ be the sum of all the other fixed contributions.
Then
\begin{equation}
r_c(a)
=r_{c,-t}(a)+s_{p_t,\varepsilon}(U_t)
\ge r_{c,-t}(a)+\log\varepsilon
=r_{c,-t}(a)-\log(1/\varepsilon).
\end{equation}
Thus adding this summand can lower the existing sum by at most
$\log(1/\varepsilon)$. This proves claim (i) and the stated
bound on a token's additive negative contribution.

If $p=1$, then $f_1(u)=u^0=1$ for every $u\in(0,1)$.
In this case
\begin{equation}
h_{1,\varepsilon}(u)=(1-\varepsilon)+\varepsilon=1,
\qquad s_{1,\varepsilon}(u)=\log1=0.
\end{equation}
This proves claim (ii), for watermarked and foreign tokens alike.

\smallskip\noindent
\textbf{The exact null law of one summand.}
For a wrong candidate, the preceding joint-density calculation
shows that its $U_t$ are independent $U(0,1)$ variables.
Applying the fixed function $s_{p_t,\varepsilon}$ separately
to each coordinate preserves independence. Thus its score
has exactly the law
\begin{equation}
Z_{c,\varepsilon}
=\sum_{t\in \mathcal{T}_c}s_{p_t,\varepsilon}(U_t),
\qquad U_t\ \text{independent }U(0,1).
\end{equation}
This already identifies the null distribution using only
the fixed probabilities $p_t$ and $\varepsilon$. We now
derive its component distributions and an explicit recursion
for their sum.

Let $U\sim U(0,1)$ and first suppose $0<p<1$. Put
\begin{equation}
a_\varepsilon=\log\varepsilon,\qquad
b_{p,\varepsilon}
=\log\left(\varepsilon+\frac{1-\varepsilon}{p}\right).
\end{equation}
Because $(1-p)/p>0$, the function
\begin{equation}
u\longmapsto
s_{p,\varepsilon}(u)
=\log\left(\varepsilon+\frac{1-\varepsilon}{p}
                         u^{(1-p)/p}\right)
\end{equation}
is continuous and strictly increasing on $(0,1)$, with
endpoint limits $a_\varepsilon$ and $b_{p,\varepsilon}$.
For $a_\varepsilon<x<b_{p,\varepsilon}$, inversion gives
\begin{equation}
\begin{aligned}
s_{p,\varepsilon}(U)\le x
&\quad\Longleftrightarrow\quad
\varepsilon+\frac{1-\varepsilon}{p}U^{(1-p)/p}\le e^x\\
&\quad\Longleftrightarrow\quad
U^{(1-p)/p}\le\frac{p(e^x-\varepsilon)}{1-\varepsilon}\\
&\quad\Longleftrightarrow\quad
U\le\left(\frac{p(e^x-\varepsilon)}{1-\varepsilon}\right)^{p/(1-p)}.
\end{aligned}
\end{equation}
On this interval, the last threshold lies in $(0,1)$.
Using the uniform distribution function therefore yields
the exact null distribution function
\begin{equation}
F_{p,\varepsilon}(x)=
\begin{cases}
0,&x\le a_\varepsilon,\\[2pt]
\displaystyle
\left(\frac{p(e^x-\varepsilon)}{1-\varepsilon}\right)^{p/(1-p)},
  &a_\varepsilon<x<b_{p,\varepsilon},\\[7pt]
1,&x\ge b_{p,\varepsilon}.
\end{cases}
\end{equation}
Its endpoint limits are $0$ and $1$, so this distribution
is continuous. On the open support interval its density is
\begin{equation}
d_{p,\varepsilon}(x)
=\frac{p}{1-p}
 \left(\frac{p}{1-\varepsilon}\right)^{p/(1-p)}
 e^x(e^x-\varepsilon)^{p/(1-p)-1}>0.
\end{equation}
This expression follows by differentiating the interior
formula with the chain rule. When $p=1$, the summand is
instead identically zero, and its distribution function is
$F_{1,\varepsilon}(x)=\mathbf{1}\{x\ge0\}$.

\smallskip\noindent
\textbf{Obtaining the null distribution of the sum.}
Enumerate $\mathcal{T}_c=\{t_1,\ldots,t_{n_c}\}$ and let $G_j$ be the
distribution function of the sum of its first $j$ null
contributions. The initial empty sum gives
$G_0(x)=\mathbf{1}\{x\ge0\}$. Independence of the next uniform
and the previous sum gives, for $j=1,\ldots,n_c$,
\begin{equation}
G_j(x)=\int_0^1
  G_{j-1}\bigl(x-s_{p_{t_j},\varepsilon}(u)\bigr)\,\mathrm{d}u.
\end{equation}
Indeed, conditional on $U_{t_j}=u$, the sum is at most $x$
exactly when the previous sum is at most
$x-s_{p_{t_j},\varepsilon}(u)$; integration averages over the
uniform density $1$. Thus
$F_{c,\varepsilon}:=G_{n_c}$ is the exact null distribution
function. A position with $p_t=1$ leaves the recursion
unchanged because its score contribution is zero.

The inclusive upper tail used in the certificate is
\begin{equation}
Q_{c,\varepsilon}(x)
=\Pr[Z_{c,\varepsilon}\ge x]
=1-F_{c,\varepsilon}(x-),
\end{equation}
where $F(x-)=\lim_{z\uparrow x}F(z)$. The left limit makes
this identity valid even for a point mass.

For use in the error bound, let
$J_c=\{t\in \mathcal{T}_c:p_t<1\}$ be the positions with nonconstant
score contributions. If $J_c=\varnothing$, the null sum
is identically zero. Otherwise its support endpoints are
\begin{equation}
\underline{s}_c=|J_c|\log\varepsilon,\qquad
\overline{s}_c=\sum_{t\in J_c}
  \log\left(\varepsilon+\frac{1-\varepsilon}{p_t}\right).
\end{equation}
The sum has a continuous distribution and a density positive
on $(\underline{s}_c,\overline{s}_c)$. To see positivity,
each component has the positive density derived above.
If two densities are positive on intervals $(A_1,B_1)$
and $(A_2,B_2)$, then at any
$x\in(A_1+A_2,B_1+B_2)$ their convolution integrates a
positive function over
\begin{equation}
\bigl(\max\{A_1,x-B_2\},\ \min\{B_1,x-A_2\}\bigr),
\end{equation}
which is a nonempty interval. Repeating this argument proves
the assertion for every finite number of active positions.
The density also shows there are no point masses.
Consequently $Q_{c,\varepsilon}$ is continuous, equals $1$
at $\underline{s}_c$, equals $0$ at $\overline{s}_c$, and is
strictly decreasing between them.

\Needspace{12\baselineskip}
\smallskip\noindent
\textbf{Certification after selecting the largest robust score.}
Fix the actual true chunk $m_c$. Let $\Pr_*$ denote
probability in the contamination model conditional on the
text and scored sets fixed at the beginning of the proof.
Choose any fixed tie-breaking rule, and write
\begin{equation}
\hat m=\arg\max_{a\in\mathcal{M}_c}r_c(a),\qquad
\hat\delta=\min\{1,(K-1)Q_{c,\varepsilon}(r_c(\hat m))\}.
\end{equation}
The certified decoder $D_{\delta,\varepsilon}$ returns
$\hat m$ when $\hat\delta\le\delta$ and returns $\bot$
otherwise. We must bound the event that it accepts a wrong
candidate.

If $\delta=1$, the claim follows because every event has
probability at most $1$. Suppose $0<\delta<1$. If
$J_c=\varnothing$, every candidate's score is $0$, including
the case $\mathcal{T}_c=\varnothing$. The inclusive tail gives
\begin{equation}
Q_{c,\varepsilon}(0)=1,\qquad
\hat\delta=\min\{1,K-1\}=1.
\end{equation}
Since $K\ge2$, the decoder therefore abstains and its error
probability is zero. Using an inclusive tail is essential
here: the null law assigns all its mass to the observed
score $0$.

Now assume $J_c\neq\varnothing$. Since
$0<\delta/(K-1)<1$, continuity and strict decrease of the
null tail give a unique
$s_\delta\in(\underline{s}_c,\overline{s}_c)$ such that
\begin{equation}
Q_{c,\varepsilon}(s_\delta)=\frac{\delta}{K-1}.
\end{equation}
For $\delta<1$, clipping at $1$ does not change whether the
certificate is at most $\delta$. Monotonicity of the tail
then gives
\begin{equation}
\begin{aligned}
\hat\delta\le\delta
&\quad\Longleftrightarrow\quad
(K-1)Q_{c,\varepsilon}(r_c(\hat m))\le\delta\\
&\quad\Longleftrightarrow\quad
Q_{c,\varepsilon}(r_c(\hat m))
 \le Q_{c,\varepsilon}(s_\delta)\\
&\quad\Longleftrightarrow\quad
r_c(\hat m)\ge s_\delta.
\end{aligned}
\end{equation}
Therefore, under the fixed conditioning,
\begin{equation}
\begin{aligned}
\{D_{\delta,\varepsilon}(y)\notin\{m_c,\bot\}\}
&=\{\hat m\neq m_c\}\cap\{r_c(\hat m)\ge s_\delta\}\\
&\subseteq\bigcup_{a\neq m_c}\{r_c(a)\ge s_\delta\}.
\end{aligned}
\end{equation}
The inclusion holds because an accepted wrong candidate is
itself one of the wrong candidates whose score crosses the
threshold. It does not assign the null law to the selected
maximum.

\Needspace{23\baselineskip}
\smallskip\noindent
\textbf{Bounding and averaging the error probability.}
For each fixed wrong candidate, the null calculation gives
\begin{equation}
\Pr_*[r_c(a)\ge s_\delta]
=Q_{c,\varepsilon}(s_\delta)=\frac{\delta}{K-1}.
\end{equation}
Applying the conditional union bound over the $K-1$ wrong
candidates yields
\begin{equation}
\begin{aligned}
\Pr_*[D_{\delta,\varepsilon}(y)\notin\{m_c,\bot\}]
&\le\sum_{a\neq m_c}\Pr_*[r_c(a)\ge s_\delta]\\
&=(K-1)\frac{\delta}{K-1}=\delta.
\end{aligned}
\end{equation}
Independence between different candidates is not needed.
The mixture law of the correct candidate identifies the
score as a likelihood ratio; the certification bound uses
only the independent-uniform law within each wrong candidate.

The preceding bound holds for every conditioning pair
satisfying the theorem's assumptions. Averaging over the
text and scored sets, whenever they are random in the
contamination model, gives
\begin{equation}
\begin{aligned}
&\Pr[D_{\delta,\varepsilon}(Y_{1:n})\notin\{m_c,\bot\}]\\
&\quad=\mathbb{E}\!\left[
  \Pr[D_{\delta,\varepsilon}(Y_{1:n})\notin\{m_c,\bot\}
      \mid Y_{1:n},\{\mathcal{T}_c\}_{c=1}^C]\right]
\le\delta.
\end{aligned}
\end{equation}
The probability here is taken in the stated contamination
model. This establishes the certified error guarantee and
completes claim (iii), together with the likelihood formula,
the score bounds, and all degenerate cases.
\end{proof}

\Needspace{12\baselineskip}
\section{Additional theorems}
\label{app:additional}

Theorem~\ref{thm:exact} concerns a single generation. With a fixed key,
repeating the same prompt and message produces the same text, and two
generations can read the same sampling uniforms whenever their contexts
and active message values agree. To separate their sampling randomness,
draw a nonce $\nu$ uniformly from $\{0,1\}^{r}$ for every generation,
independently of the key and of all other nonces. Replace
\eqref{eq:encoder} by
\begin{equation}
\begin{aligned}
a_t&=
\begin{cases}
(\mathsf{fresh},\,\nu,\,c_t,\,t),
  &c_t\in\mathcal{C}_{t-1},\\[2pt]
(\mathsf{msg},\,\nu,\,c_t,\,m_{i_t}),
  &c_t\notin\mathcal{C}_{t-1},
\end{cases}\\
u_t(v)&=\rho\bigl(\PRF_\kappa(a_t,v)\bigr),
\qquad
\mathcal{C}_t=\mathcal{C}_{t-1}\cup\{c_t\}.
\end{aligned}
\label{eq:encoder_nonce}
\end{equation}
The seen-context set starts empty in each generation. Chunk selection
\eqref{eq:chunkidx} and token emission \eqref{eq:emit} are unchanged.

The decoder treats the nonce as part of each candidate. For chunk $c$,
it enumerates
\begin{equation}
\mathcal{P}_c=\{0,1\}^{r}\times\{0,1\}^{k_c},
\qquad K_c=|\mathcal{P}_c|=2^{r+k_c},
\end{equation}
and computes
\begin{equation}
\begin{aligned}
u_t^{(\nu',m')}(v)
&=\rho\bigl(\PRF_\kappa(\mathsf{msg},\nu',c_t,m',v)\bigr),\\
S_c(\nu',m')
&=\sum_{t\in \mathcal{T}_c}-\log\bigl(1-u_t^{(\nu',m')}(y_t)\bigr).
\end{aligned}
\end{equation}
Choose a fixed rule for breaking ties and write
$(\hat\nu_c,\hat m_c)=\arg\max_{(\nu',m')\in\mathcal{P}_c}S_c(\nu',m')$.
The returned chunk is the message component $\hat m_c$. When $r=0$,
there is a single possible nonce, so this construction reduces to the
original scheme up to the fixed extra field in the PRF arguments.

\begin{theorem}[Distortion-free generation under one deployed key]
\label{thm:nonce}
Treat $\PRF_\kappa$ as one random function shared by all generations.
Fix a nonnegative integer $r$, an integer $N\ge1$, the model and sampling
settings, and prompts $x^{(j)}$, messages $m^{(j)}\in\{0,1\}^{L}$ and
lengths $n_j$ for $j=1,\ldots,N$. Let $Y^{(j)}$ be generated by
\eqref{eq:encoder_nonce} and \eqref{eq:emit}, using independent uniform
nonces $\nu_j\in\{0,1\}^{r}$ that are independent of the random function.
Write $\Pr_{\kappa,\nu}$ for probability over that function and the
nonces, and define
\begin{equation}
P_j(w)=\prod_{t=1}^{n_j}p(w_t\mid x^{(j)},w_{<t}),
\qquad
D=\{\nu_1,\ldots,\nu_N\text{ are pairwise distinct}\}.
\end{equation}
Then the following hold.
\begin{enumerate}
\setlength{\itemsep}{4pt}
\item[(i)] \textbf{Single-generation law.}
For every $j$ and $w\in V^{n_j}$,
$\Pr_{\kappa,\nu}[Y^{(j)}=w]=P_j(w)$.

\item[(ii)] \textbf{Joint law for distinct nonces.}
If $N\le2^r$, then $\Pr_{\kappa,\nu}[D]>0$ and, for all
$w^{(j)}\in V^{n_j}$,
\begin{equation}
\Pr_{\kappa,\nu}\!\left[
  \bigcap_{j=1}^{N}\{Y^{(j)}=w^{(j)}\}\,\middle|\,D\right]
=\prod_{j=1}^{N}P_j(w^{(j)}).
\label{eq:joint_nonce}
\end{equation}
Thus, conditional on distinct nonces, the generations are independent
unwatermarked samples, even when their fixed prompts or messages coincide.

\item[(iii)] \textbf{Collision and total variation bounds.}
Let $P$ be the joint law of $(Y^{(1)},\ldots,Y^{(N)})$ and
$Q=\bigotimes_{j=1}^{N}P_j$. Then
\begin{equation}
d_{\mathrm{TV}}(P,Q)
\le\Pr_{\kappa,\nu}[D^c]
\le\min\left\{1,\frac{N(N-1)}{2^{r+1}}\right\}.
\end{equation}
This holds for every $N$, including $N>2^r$.

\item[(iv)] \textbf{Decoding and certification for one generation.}
Fix one generation and a chunk $c$ with $k_c\ge1$, and write $\nu$ for
its actual nonce. Conditional on its text, its scored-position sets
$\{\mathcal{T}_c\}_{c=1}^C$ and $\nu$, every pair
$(\nu',m')\ne(\nu,m_c)$ has score law $\Gamma(n_c,1)$ when $n_c\ge1$,
and score $0$ when $n_c=0$.

Let $G_n\sim\Gamma(n,1)$ for $n\ge1$, let $G_0=0$ almost surely,
and put $Q_n(s)=\Pr[G_n\ge s]$. The certificate
\begin{equation}
\hat\delta_c
=\min\left\{1,(K_c-1)Q_{n_c}\bigl(S_c(\hat\nu_c,\hat m_c)\bigr)\right\}
\end{equation}
has the guarantee of \eqref{eq:cert}: returning $\hat m_c$ when
$\hat\delta_c\le\delta$ and abstaining otherwise has probability
at most $\delta$ of returning an incorrect chunk, for
$\delta\in(0,1]$. Moreover, for $\Lambda_c$ defined in
\eqref{eq:scaling},
\begin{equation}
\Pr_{\kappa,\nu}[\hat m_c\ne m_c\mid Y_{1:n}=y,\{\mathcal{T}_c\}_{c=1}^C]
\le (K_c-1)e^{-\Lambda_c}
<\exp\bigl((k_c+r)\ln2-\Lambda_c\bigr).
\end{equation}
The corresponding whole-message bound is
$\sum_{c=1}^{C}(2^{k_c+r}-1)e^{-\Lambda_c}$.
An empty scored set has $\hat\delta_c=1$ and $\Lambda_c=0$.
\end{enumerate}
\end{theorem}

The nonce separates the sampling inputs of generations whose nonces
differ. The decoder searches $2^r$ times as many candidates, so its
certificate and error bound have the same candidate-count factor as
a message with $r$ additional bits. As in Theorem~\ref{thm:scaling},
about $r\ln2$ additional accumulated evidence offsets this factor.
The decoder searches for the nonce jointly with the message using
the text and key. Our experiments use $r=0$.

\begin{proof}
Keep the model, sampling settings, prompts, messages and lengths fixed
as in the theorem. Throughout the proof, $\Pr$ and $\mathbb{E}$ denote
probability and expectation over the single shared random function
and the independent nonces. The function is drawn once and reused;
it is not redrawn between generations.

Let $\mathcal{A}$ be the complete context-to-chunk assignment table
determined by the inputs tagged $\mathsf{chunk}$. These inputs are
disjoint from all sampling inputs, whose tags are $\mathsf{msg}$ or
$\mathsf{fresh}$. Under the continuous random-function idealisation
following \eqref{eq:rho}, conditioning on $\mathcal{A}$ therefore
leaves the sampling values independent $U(0,1)$ variables. The nonces
are also independent of $\mathcal{A}$ and of those values. We first
work with $\mathcal{A}$ fixed, and later average over it. This
conditioning is useful because the assignment table is shared even
when the nonces differ.

\smallskip\noindent
\textbf{Step 1: the law of one generation with its nonce fixed.}
Fix $j$ and a nonce value $a\in\{0,1\}^{r}$, and condition on
$(\mathcal{A},\nu_j=a)$. Let $\mathcal{H}^{(j)}_t$ contain this
information and all queries and answers in generation $j$ before
its step-$t$ sampling vector is read. It determines the prefix,
the context, the seen-context set, the active chunk and the
next-token law.

Every sampling input at step $t$ is new within this generation.
To verify this, consider the two cases in \eqref{eq:encoder_nonce}.
On a first occurrence of the context, the input for token $v$ is
$\langle\mathsf{msg},a,c_t,m_{i_t},v\rangle$. Its context field
differs from that of every earlier message input. On a repeated
context, the input is $\langle\mathsf{fresh},a,c_t,t,v\rangle$;
its position field differs from that of every earlier fresh input.
The tags separate these two families from each other and from
chunk-selection inputs. Finally, distinct tokens $v$ give distinct
inputs because the tuple encoding is injective.

A random function can be exposed as queries arrive: an unqueried
input receives an independent value, and a repeated input returns
its stored value. The step-$t$ input is chosen from the preceding
history, so adaptively choosing it introduces no conditioning on
its as-yet unobserved value. Thus, conditional on
$\mathcal{H}^{(j)}_t$, the sampling vector consists of independent
uniforms. Lemma~\ref{lem:gumbel} gives
\begin{equation}
\Pr[Y^{(j)}_t=v\mid\mathcal{H}^{(j)}_t]
=p(v\mid x^{(j)},Y^{(j)}_{<t}).
\end{equation}

For a fixed text $w\in V^{n_j}$, put
$E_t=\{Y^{(j)}_{1:t}=w_{1:t}\}$ and $E_0=\Omega$.
Iterated conditional expectation gives
\begin{equation}
\begin{aligned}
\Pr[E_t\mid\mathcal{A},\nu_j=a]
&=\mathbb{E}\!\left[
  \mathbf{1}_{E_{t-1}}
  \Pr[Y^{(j)}_t=w_t\mid\mathcal{H}^{(j)}_t]
  \,\middle|\,\mathcal{A},\nu_j=a\right]\\
&=p(w_t\mid x^{(j)},w_{<t})
  \Pr[E_{t-1}\mid\mathcal{A},\nu_j=a].
\end{aligned}
\end{equation}
On $E_{t-1}$ the prefix is $w_{<t}$, which justifies the second
equality. This recursion also covers zero-probability prefixes
without conditioning on such a prefix. Starting at $t=1$ yields
\begin{equation}
\Pr[Y^{(j)}=w\mid\mathcal{A},\nu_j=a]
=\prod_{t=1}^{n_j}p(w_t\mid x^{(j)},w_{<t})
=P_j(w).
\label{eq:step1_nonce}
\end{equation}
The right-hand side depends on neither the table nor the nonce.
Consequently,
\begin{equation}
\Pr[Y^{(j)}=w]
=\mathbb{E}\!\left[\Pr[Y^{(j)}=w\mid\mathcal{A},\nu_j]\right]
=P_j(w),
\end{equation}
which proves (i). This is a marginal calculation for generation $j$;
it does not condition on the outputs or queries of other generations.

\Needspace{8\baselineskip}
\smallskip\noindent
\textbf{Step 2: independence when the nonces are distinct.}
For a nonce value $a$, let $\mathcal{I}(a)$ be the set of all sampling
inputs of the forms
\begin{equation}
\langle\mathsf{msg},a,c,b,v\rangle,
\qquad
\langle\mathsf{fresh},a,c,t,v\rangle,
\end{equation}
with admissible contexts, message values, positions and tokens.
Injectivity of the tuple encoding implies
$\mathcal{I}(a)\cap\mathcal{I}(a')=\varnothing$ when $a\ne a'$:
their nonce fields differ. All these sets are also disjoint from
the chunk-selection inputs.

Given $\mathcal{A}$ and $\nu_j=a_j$, the text $Y^{(j)}$ is a
deterministic function of the random-function values on
$\mathcal{I}(a_j)$ and the fixed prompt, message and sampling
settings. Although the queried inputs depend on previous tokens,
all of them stay within this fixed set. For pairwise distinct
$a_1,\ldots,a_N$, the restrictions of the random function to
these sets are independent, even conditional on $\mathcal{A}$.
Their resulting texts are therefore independent under that
conditioning. With
$E_{\mathbf{w}}=\bigcap_{j=1}^{N}\{Y^{(j)}=w^{(j)}\}$,
\eqref{eq:step1_nonce} gives
\begin{equation}
\begin{aligned}
\Pr[E_{\mathbf{w}}\mid\mathcal{A},\nu_1=a_1,\ldots,\nu_N=a_N]
&=\prod_{j=1}^{N}
  \Pr[Y^{(j)}=w^{(j)}\mid\mathcal{A},\nu_j=a_j]\\
&=\prod_{j=1}^{N}P_j(w^{(j)}).
\end{aligned}
\end{equation}
Conditional independence alone would not in general imply
independence after averaging over a shared table. Here it does,
because this product does not depend on $\mathcal{A}$ or on the
particular distinct nonce values.

When $N\le2^r$, at least one distinct nonce tuple exists and has
positive probability, so $\Pr[D]>0$. Averaging the preceding
identity conditional on $D$ yields
\begin{equation}
\begin{aligned}
\Pr[E_{\mathbf{w}}\mid D]
&=\mathbb{E}\!\left[
  \Pr[E_{\mathbf{w}}\mid\mathcal{A},\nu_1,\ldots,\nu_N]
  \,\middle|\,D\right]\\
&=\prod_{j=1}^{N}P_j(w^{(j)}),
\end{aligned}
\end{equation}
proving (ii). When $N>2^r$, distinct nonces are impossible, so no
conditional law given $D$ is asserted.

\smallskip\noindent
\textbf{Step 3: bounding nonce collisions.}
For any two distinct generations $j<j'$,
\begin{equation}
\begin{aligned}
\Pr[\nu_j=\nu_{j'}]
&=\sum_{a\in\{0,1\}^{r}}\Pr[\nu_j=a,\nu_{j'}=a]\\
&=2^r\cdot 2^{-r}\cdot 2^{-r}=2^{-r}.
\end{aligned}
\end{equation}
The second equality uses independence and uniformity of the two
nonces. Since $D^c$ occurs exactly when at least one pair agrees,
the union bound gives
\begin{equation}
\Pr[D^c]
\le\sum_{1\le j<j'\le N}\Pr[\nu_j=\nu_{j'}]
=\binom{N}{2}2^{-r}
=\frac{N(N-1)}{2^{r+1}}.
\end{equation}
Combining this with $\Pr[D^c]\le1$ gives the stated collision
bound. Independence between the pairwise collision events is
not needed. For $N=1$ the union is empty and its probability
is zero; for $N>2^r$ the collision probability is one.

\smallskip\noindent
\textbf{Step 4: converting the collision bound to total variation.}
Write $\mathbf{Y}=(Y^{(1)},\ldots,Y^{(N)})$ and $q=\Pr[D^c]$.
We use the convention
\begin{equation}
d_{\mathrm{TV}}(P,Q)
=\sup_{E\subseteq\prod_{j=1}^{N}V^{n_j}}|P(E)-Q(E)|.
\end{equation}
If $q=1$, then $d_{\mathrm{TV}}(P,Q)\le1=q$, since $P(E)$
and $Q(E)$ are probabilities. Now suppose $q<1$, so $D$ has
positive probability. By (ii), for every set $E$ of text tuples,
\begin{equation}
\Pr[\mathbf{Y}\in E,D]
=(1-q)\,Q(E).
\end{equation}
Decomposing according to $D$ and $D^c$ therefore gives
\begin{equation}
\begin{aligned}
P(E)-Q(E)
&=(1-q)Q(E)+\Pr[\mathbf{Y}\in E,D^c]-Q(E)\\
&=\Pr[\mathbf{Y}\in E,D^c]-qQ(E).
\end{aligned}
\end{equation}
Both terms in the last difference lie in $[0,q]$, so its absolute
value is at most $q$. Taking the supremum over $E$ gives
$d_{\mathrm{TV}}(P,Q)\le q$. Step~3 completes (iii), including
the case in which conditioning on distinct nonces is impossible.

\Needspace{8\baselineskip}
\smallskip\noindent
\textbf{Step 5: the null law for every wrong nonce--message pair.}
Fix one generation and drop its superscript $j$. Fix its true
nonce $\nu=a$, a chunk $c$ with $k_c\ge1$, and write
\begin{equation}
K=2^{r+k_c},\qquad
\theta_*=(a,m_c),\qquad
\mathcal{W}_c=\mathcal{P}_c\setminus\{\theta_*\}.
\end{equation}
For the remainder of the proof, let $\Pr_*$ and $\mathbb{E}_*$
denote conditioning on
$(Y_{1:n}=y,\{\mathcal{T}_c\}_{c=1}^C,\nu=a)$, for a realisation of positive
probability. This fixes $\mathcal{T}_c$, $n_c$ and
$p_t=p(y_t\mid x,y_{<t})>0$ at each scored position.

Let $\mathcal{H}$ contain $\nu$, the complete query transcript
of this generation, and the decoder's chunk-selection queries
before candidate scoring. It determines the text, the scored sets
and the correct pair's sampling values. For
$\theta=(a',b)\in\mathcal{W}_c$ and $t\in \mathcal{T}_c$, the decoder reads
the input
\begin{equation}
d_{t,\theta}=\langle\mathsf{msg},a',c_t,b,y_t\rangle.
\end{equation}
These inputs are distinct over all $(t,\theta)$. Different scored
positions have different contexts, and different candidate pairs
differ in their nonce or message field.

None of these inputs appears in $\mathcal{H}$. A chunk or fresh
query has a different tag. An encoder message query with context
different from $c_t$ differs in its context field. A message query
with context $c_t$ uses chunk $c$, since the chunk assignment is
determined by that context, and therefore has the form
$\langle\mathsf{msg},a,c_t,m_c,v\rangle$. It differs from
$d_{t,\theta}$ in the nonce field if $a'\ne a$, and in the
message field if $a'=a$ and $b\ne m_c$. This comparison covers
all encoder positions and tokens, including queries after $t$.

Conditional on $\mathcal{H}$, the values at these distinct,
unqueried inputs are independent uniforms. More explicitly, set
$W_{t,\theta}=u_t^\theta(y_t)$. For numbers
$z_{t,\theta}\in[0,1]$, the random-function argument gives
\begin{equation}
\Pr[W_{t,\theta}\le z_{t,\theta}\text{ for all }t,\theta
      \mid\mathcal{H}]
=\prod_{t\in \mathcal{T}_c}\prod_{\theta\in\mathcal{W}_c}z_{t,\theta}.
\end{equation}
The right-hand side is fixed under $\Pr_*$, so averaging over
$\mathcal{H}$ proves the same factorisation under $\Pr_*$.
Thus every wrong pair has independent uniform scored values.
Other generations may have queried some of these inputs, but
their transcripts and outputs are not part of this conditioning.
Reading a random-function value does not change it, and integrating
out those other generations leaves the marginal law just derived.

For $U\sim U(0,1)$ and $s\ge0$,
\begin{equation}
\Pr[-\log(1-U)\ge s]
=\Pr[U\ge1-e^{-s}]=e^{-s}.
\end{equation}
Hence each wrong-pair score is a sum of $n_c$ independent
$\mathrm{Exp}(1)$ variables. It has law $\Gamma(n_c,1)$ for
$n_c\ge1$, by the convolution calculation in
Appendix~\ref{app:null}. When $n_c=0$, it is the empty sum $0$.
This proves the null-law assertion in (iv).

\Needspace{8\baselineskip}
\smallskip\noindent
\textbf{Step 6: the correct score and its independence from wrong scores.}
For the scaling bound we also need the correct pair's joint law
after conditioning on the full text. The reasoning in
Appendix~\ref{app:scaling} continues to apply with the fixed nonce
field; we give the factorisation explicitly.

First condition on $\mathcal{A}$ and $\nu=a$ before generation.
Write $\mathcal{H}_t$ for the single-generation history from Step~1,
with the generation superscript omitted. At each step define
\begin{equation}
M_t=\max_{v:p_t(v)>0}u_t(v)^{1/p_t(v)}.
\end{equation}
The maximiser is the emitted token, because this is an equivalent
form of the Gumbel-max rule. For a token $v$ with $p_t(v)>0$ and
$z\in[0,1]$, independence of the fresh uniforms in Step~1 gives
\begin{equation}
\begin{aligned}
\Pr[Y_t=v,M_t\le z\mid\mathcal{H}_t]
&=\int_0^z p_t(v)s^{p_t(v)-1}
  \prod_{\substack{w\ne v\\p_t(w)>0}}s^{p_t(w)}\,\mathrm{d}s\\
&=p_t(v)\int_0^z s^{\sum_{w:p_t(w)>0}p_t(w)-1}\,\mathrm{d}s\\
&=z\,p_t(v).
\end{aligned}
\end{equation}
The integrand fixes the winning transformed uniform at $s$ and
requires every competitor to be at most $s$. Ties have probability
zero. The last equality uses $\sum_w p_t(w)=1$; for $p_t(v)=0$,
both sides of the identity are zero.

Apply the prefix recursion from Step~1, now also requiring
$M_t\le z_t$ at each step. At step $t$ its multiplier is
$z_t p(y_t\mid x,y_{<t})$, so
\begin{equation}
\Pr[Y_{1:n}=y,\ M_t\le z_t\text{ for all }t
       \mid\mathcal{A},\nu=a]
=\prod_{t=1}^{n}z_t p(y_t\mid x,y_{<t}).
\end{equation}
Dividing by the positive text probability from
\eqref{eq:step1_nonce} gives
\begin{equation}
\Pr[M_t\le z_t\text{ for all }t
       \mid Y_{1:n}=y,\mathcal{A},\nu=a]
=\prod_{t=1}^{n}z_t.
\end{equation}
For fixed $y$, the table determines $\{\mathcal{T}_c\}_{c=1}^C$. Set $z_t=1$
outside $\mathcal{T}_c$ and average over the tables consistent with
$(y,\{\mathcal{T}_c\}_{c=1}^C,a)$. The remaining product is unchanged, proving
that $(M_t)_{t\in \mathcal{T}_c}$ are independent uniforms under $\Pr_*$.

At a scored position, the correct pair reconstructs the encoder's
uniform and satisfies
\begin{equation}
B_t:=u_t^{\theta_*}(y_t)=M_t^{p_t},
\qquad
\Pr_*[B_t\le b]=b^{1/p_t}\quad(0<b<1).
\end{equation}
Consequently the $B_t$ are independent
$\mathrm{Beta}(1/p_t,1)$ variables under $\Pr_*$, and the correct
score is
$S_*=S_c(\theta_*)=\sum_{t\in \mathcal{T}_c}-\log(1-B_t)$.

The complete transcript $\mathcal{H}$ determines $S_*$.
Combining this fact with the conditional uniform factorisation
from Step~5 yields, for any real $s$,
\begin{equation}
\begin{aligned}
\Pr_*[S_*\le s,\ W\le z]
&=\mathbb{E}_*\!\left[
  \mathbf{1}_{\{S_*\le s\}}\Pr[W\le z\mid\mathcal{H}]\right]\\
&=\Pr_*[S_*\le s]
  \prod_{t\in \mathcal{T}_c}\prod_{\theta\in\mathcal{W}_c}z_{t,\theta}.
\end{aligned}
\end{equation}
Here inequalities between vectors are coordinatewise. This
factorisation proves that the entire array of wrong-pair uniforms
is independent of the correct score under $\Pr_*$. In particular,
each wrong score is independent of $S_*$.

\Needspace{8\baselineskip}
\smallskip\noindent
\textbf{Step 7: the certificate after selecting a candidate pair.}
Write $\hat\theta=(\hat\nu_c,\hat m_c)$ and let
$D_\delta^{\mathrm{pair}}$ return $\hat\theta$ if
$\hat\delta_c\le\delta$, and $\bot$ otherwise. We first bound
acceptance of a wrong pair. If $\delta=1$, the claimed bound
is automatic. Suppose $0<\delta<1$.

If $n_c=0$, every score is zero and $Q_0(0)=1$, so
$\hat\delta_c=\min\{1,K-1\}=1$ because $K\ge2$. The decoder
therefore abstains. If $n_c\ge1$, the Gamma density is positive
on $(0,\infty)$, so its upper tail $Q_{n_c}$ is continuous and
strictly decreasing there, from $1$ to $0$. Since
$0<\delta/(K-1)<1$, there is a unique $s_\delta>0$ with
\begin{equation}
Q_{n_c}(s_\delta)=\frac{\delta}{K-1}.
\end{equation}
For $\delta<1$, clipping the certificate at $1$ does not change
its acceptance event. Monotonicity of the tail gives
\begin{equation}
\begin{aligned}
\hat\delta_c\le\delta
&\quad\Longleftrightarrow\quad
Q_{n_c}(S_c(\hat\theta))\le\frac{\delta}{K-1}\\
&\quad\Longleftrightarrow\quad
S_c(\hat\theta)\ge s_\delta.
\end{aligned}
\end{equation}
Thus an accepted wrong pair must be one of the $K-1$ wrong
pairs whose score crosses this fixed threshold:
\begin{equation}
\{D_\delta^{\mathrm{pair}}\notin\{\theta_*,\bot\}\}
\subseteq
\bigcup_{\theta\in\mathcal{W}_c}\{S_c(\theta)\ge s_\delta\}.
\end{equation}
Each fixed wrong pair has the null law from Step~5, so
\begin{equation}
\begin{aligned}
\Pr_*[D_\delta^{\mathrm{pair}}\notin\{\theta_*,\bot\}]
&\le\sum_{\theta\in\mathcal{W}_c}
      \Pr_*[S_c(\theta)\ge s_\delta]\\
&=(K-1)Q_{n_c}(s_\delta)=\delta.
\end{aligned}
\end{equation}
The union bound is applied to the fixed candidates, rather than
assigning the Gamma law to their selected maximum.

Returning a wrong message requires accepting a wrong pair;
a pair with an incorrect nonce but the correct message is
not a message error. Hence the certified message decoder
$D_\delta$ satisfies the same conditional bound. Averaging
over this generation's text, scored sets and nonce gives
\begin{equation}
\Pr[D_\delta(Y_{1:n})\notin\{m_c,\bot\}]
=\mathbb{E}\!\left[
  \Pr[D_\delta(Y_{1:n})\notin\{m_c,\bot\}
       \mid Y_{1:n},\{\mathcal{T}_c\}_{c=1}^C,\nu]\right]
\le\delta.
\end{equation}
This is the guarantee in \eqref{eq:cert}.

\Needspace{8\baselineskip}
\smallskip\noindent
\textbf{Step 8: the payload-scaling bound for pair candidates.}
Assume first that $n_c\ge1$, fix $\theta\in\mathcal{W}_c$,
and choose $0<\lambda<1$. The exponential moment of a wrong
score is
\begin{equation}
\mathbb{E}_*[e^{\lambda S_c(\theta)}]
=\prod_{t\in \mathcal{T}_c}\int_0^1(1-u)^{-\lambda}\,\mathrm{d}u
=(1-\lambda)^{-n_c}.
\end{equation}
For a correct contribution, put $\alpha_t=1/p_t$. Its
negative exponential moment is
\begin{equation}
\begin{aligned}
\mathbb{E}_*[(1-B_t)^\lambda]
&=\alpha_t\int_0^1 b^{\alpha_t-1}(1-b)^\lambda\,\mathrm{d}b\\
&=\frac{\Gamma(\alpha_t+1)\Gamma(\lambda+1)}
        {\Gamma(\alpha_t+\lambda+1)}.
\end{aligned}
\end{equation}
The last equality is the beta integral proved in
Appendix~\ref{app:scaling}. Step~6 gives the independence
needed to multiply these moments. Since
$\{S_c(\theta)\ge S_*\}=\{e^{\lambda(S_c(\theta)-S_*)}\ge1\}$,
Markov's inequality gives
\begin{equation}
\begin{aligned}
\Pr_*[S_c(\theta)\ge S_*]
&\le\mathbb{E}_*[e^{\lambda(S_c(\theta)-S_*)}]\\
&=\mathbb{E}_*[e^{\lambda S_c(\theta)}]\,
  \mathbb{E}_*[e^{-\lambda S_*}]\\
&=(1-\lambda)^{-n_c}
  \prod_{t\in \mathcal{T}_c}
  \frac{\Gamma(1/p_t+1)\Gamma(\lambda+1)}
       {\Gamma(1/p_t+\lambda+1)}\\
&=\exp\left(-\sum_{t\in \mathcal{T}_c}\eta_{p_t}(\lambda)\right).
\end{aligned}
\end{equation}
The last line is precisely the definition of $\eta_p$ in
\eqref{eq:eta}. Taking the infimum over $0<\lambda<1$ yields
\begin{equation}
\Pr_*[S_c(\theta)\ge S_*]
\le
\exp\left(-\sup_{0<\lambda<1}
                 \sum_{t\in \mathcal{T}_c}\eta_{p_t}(\lambda)\right)
=e^{-\Lambda_c}.
\end{equation}
No maximiser is required: values approaching the supremum give
the same bound. Finiteness of $\Lambda_c$ follows from the
endpoint and continuity argument in Appendix~\ref{app:scaling}.

If the maximising pair is wrong, some wrong pair has score at
least $S_*$. This implication also holds for a tie, whatever
the fixed tie-breaking rule. Therefore
\begin{equation}
\begin{aligned}
\Pr_*[\hat m_c\ne m_c]
&\le\Pr_*[\hat\theta\ne\theta_*]\\
&\le\sum_{\theta\in\mathcal{W}_c}
       \Pr_*[S_c(\theta)\ge S_*]\\
&\le(K-1)e^{-\Lambda_c}
<\exp\bigl((k_c+r)\ln2-\Lambda_c\bigr).
\end{aligned}
\end{equation}
The strict inequality uses $K-1<K=2^{k_c+r}$ and
$e^{-\Lambda_c}>0$.

If $n_c=0$, then $\Lambda_c=0$ and every score is zero.
The same bound holds because the error probability is at
most $1\le K-1$. Thus all scored-set sizes are covered.
The right-hand side depends on the text and scored sets but
not on the actual nonce. Averaging over $\nu$ conditional
on $(Y_{1:n}=y,\{\mathcal{T}_c\}_{c=1}^C)$ proves the conditional bound
stated in (iv). Finally, an error in the concatenated message
requires an error in at least one chunk. A union bound under
this same conditioning gives
\begin{equation}
\Pr[\hat m\ne m\mid Y_{1:n}=y,\{\mathcal{T}_c\}_{c=1}^C]
\le\sum_{c=1}^{C}
  \Pr[\hat m_c\ne m_c\mid Y_{1:n}=y,\{\mathcal{T}_c\}_{c=1}^C]
\le\sum_{c=1}^{C}(2^{k_c+r}-1)e^{-\Lambda_c}.
\end{equation}
This completes (iv) and the proof.
\end{proof}

\end{document}